\documentclass[10.5pt,onecolumn]{IEEEtran}

\usepackage{algorithm}
\usepackage{algorithmic}
\usepackage{geometry}

\usepackage{bm,amsmath,amssymb,amsfonts,graphicx,epsfig,amsthm,color}
\usepackage{bbold}

\usepackage[hyphens]{url}
\usepackage{hyperref}
\hypersetup{colorlinks=false}
\usepackage{booktabs}       
\usepackage{amsfonts}       
\usepackage{nicefrac}       
\usepackage{microtype}      

\usepackage{times}
\usepackage{graphicx} 
\usepackage{subfigure} 

\usepackage{wrapfig}

\newcommand{\widesim}[2][1.5]{
  \mathrel{\overset{#2}{\scalebox{#1}[1]{$\sim$}}}}
  
\usepackage{accents}
\makeatletter
\def\wid{\check{{\cc@style\underline{\mskip9.5mu}}}}
\def\Wideubar{\underaccent{{\cc@style\underline{\mskip6mu}}}}
\makeatother

\makeatletter
\def\wideubar{\underaccent{{\cc@style\underline{\mskip9.5mu}}}}
\def\Wideubar{\underaccent{{\cc@style\underline{\mskip6mu}}}}
\makeatother

\makeatletter
\def\widebar{\accentset{{\cc@style\underline{\mskip9.5mu}}}}
\def\Widebar{\accentset{{\cc@style\underline{\mskip6mu}}}}
\makeatother

\newtheorem{proposition}{Proposition}
\newtheorem{lemma}{Lemma}

\newtheorem{theorem}{Theorem}
\newtheorem{definition}{Definition}

\theoremstyle{remark}\newtheorem{remark}{Remark}

\newcommand{\eqdef}{\overset{d}{=}}

\newcommand{\minimize}{{\rm minimize}}
\newcommand{\maximize}{{\rm maximize}}
\def\ccalH{{\ensuremath{\mathcal H}}}
\def\ccalT{{\ensuremath{\mathcal T}}}

\begin{document}
\title{Solving Systems of Random Quadratic Equations via Truncated Amplitude Flow}

\author{
Gang Wang,~\IEEEmembership{Student Member,~IEEE,}
	Georgios B. Giannakis,~\IEEEmembership{Fellow,~IEEE,}\\
	 and
	Yonina C. Eldar,~\IEEEmembership{Fellow,~IEEE}

\thanks{
G. Wang and G. B. Giannakis were supported in part by NSF grants 1500713 and 1514056.  
G. Wang and G. B. Giannakis are with the Digital Technology Center and the Department of Electrical and Computer Engineering, University of Minnesota, Minneapolis, MN 55455, USA. G. Wang is also with the State Key Lab of Intelligent Control and Decision of Complex Systems, Beijing Institute of Technology, Beijing 100081, P. R. China. Y. C. Eldar is with the Department of Electrical Engineering, Technion -- Israel Institute of Technology, Haifa 32000, Israel. Emails: \{gangwang,georgios\}@umn.edu; yonina@ee.technion.ac.il.}


}

\markboth{}{Wang, Giannakis, and Eldar: Solving Systems of Random Quadratic Equations via Truncated Amplitude Flow}

\maketitle

\begin{abstract}
This paper presents a new algorithm, termed \emph{truncated amplitude flow} (TAF), to recover an unknown vector $\bm{x}$ from a system of quadratic equations of the form $y_i=|\langle\bm{a}_i,\bm{x}\rangle|^2$, where $\bm{a}_i$'s are given random measurement vectors. This problem is known to be \emph{NP-hard} in general. We prove that as soon as the number of equations is on the order of the number of unknowns, TAF recovers the solution exactly (up to a global unimodular constant) with high probability and complexity growing linearly with both the number of unknowns and the number of equations. Our TAF approach adopts the \emph{amplitude-based} empirical loss function, and proceeds in two stages. In the first stage, we introduce an \emph{orthogonality-promoting} initialization that can be obtained with a few power iterations. Stage two refines the initial estimate by successive updates of scalable \emph{truncated generalized gradient iterations}, which are able to handle the rather challenging nonconvex and nonsmooth amplitude-based objective function. In particular, when vectors $\bm{x}$ and $\bm{a}_i$'s are real-valued, our gradient truncation rule provably eliminates erroneously estimated signs with high probability to markedly improve upon its untruncated version. Numerical tests using synthetic data and real images demonstrate that our initialization returns more accurate and robust estimates relative to spectral initializations. 
Furthermore, even under the same initialization, the proposed amplitude-based
 refinement outperforms existing Wirtinger flow variants, corroborating the superior performance of TAF over state-of-the-art algorithms.
\end{abstract}


\section{Introduction}
\label{sec:intro}

Consider a system of $m$ quadratic equations
\vspace{-0.em}
\begin{equation}
	y_i=\left|\langle \bm{a}_i,\bm{x}\rangle\right|^2,\quad  1\le i\le m\label{eq:quad}
	\vspace{-.em}
\end{equation}
where the data vector $\bm{y}:=\left[y_1~\cdots~y_m\right]^\ccalT$ and feature vectors $\bm{a}_i\in\mathbb{R}^n$ or $\mathbb{C}^n$ are known, whereas the vector  
$\bm{x}\in\mathbb{R}^n$ or $\mathbb{C}^n$ is the wanted unknown. 
When $\left\{\bm{a}_i\right\}_{i=1}^m$ and/or $\bm{x}$ are complex, the magnitudes of their inner-products $\left\{\langle\bm{a}_i,\bm{x}\rangle\right\}_{i=1}^m$ are given but phase information is lacking; in the real case only the signs of 
$\left\{\langle\bm{a}_i,\bm{x}\rangle\right\}_{i=1}^m$ are unknown. 
Assuming that the system of quadratic equations in \eqref{eq:quad} admits a unique solution $\bm{x}$ (up to a global unimodular constant), our objective is to reconstruct $\bm{x}$ from $m$ phaseless quadratic equations, or equivalently, to recover the missing signs/phases of $\left\{\langle\bm{a}_i,\bm{x}\rangle\right\}_{i=1}^m$ under real-/complex-valued settings. It has been established that $m\ge 2n-1$ or $m\ge 4n-4$ generic measurements $\left\{\left(\bm{a}_i;\,y_i\right)\right\}_{i=1}^m$ as in \eqref{eq:quad} suffice for uniquely determining an $n$-dimensional real-valued or complex-valued vector $\bm{x}$~\cite{2006balan,4m-4}, respectively, while the former with  $m=2n-1$ has also been shown to be necessary~\cite{2006balan,savephase}.


The problem in \eqref{eq:quad} constitutes an instance of nonconvex quadratic programming, that is generally known to be \emph{NP-hard}~\cite{nphard}. 
Specifically for real-valued vectors $\{\bm{a}_i\}$ and $\bm{x}$,  
problem \eqref{eq:quad} can be understood as a combinatorial optimization since one seeks a series of signs $\left\{s_i=\pm 1\right\}_{i=1}^m$, such that the solution to the system of linear equations $\langle \bm{a}_i,\bm{x}\rangle=s_i \psi_i$, where
 $\psi_i:=\sqrt{y_i}$, obeys the given quadratic system. Evidently, there are a total of $2^m$ different combinations of $\left\{s_i\right\}_{i=1}^m$, among which only two lead to $\bm{x}$ up to a global sign. The complex case becomes even more complicated, where instead of a set of signs $\left\{s_i\right\}_{i=1}^m$, one must determine a collection of unimodular complex scalars $\left\{\sigma_i\in\mathbb{C}\right\}_{i=1}^m$. 
Special cases with $\bm{a}_i>\bm{0}$ (entry-wise inequality), $x_i^2=1$, and $y_i=0$, $1\le i\le m$ correspond to the so-called \emph{stone problem}~\cite[Section 3.4.1]{book2001nemirovski},~\cite{twf}. 

In many fields of physical sciences and engineering, the problem of recovering the phase from intensity/magnitude-only measurements is commonly referred to as \emph{phase retrieval}~\cite{1978fienup,siam2015candes,spm2016eldar}. 
Relevant application domains include X-ray crystallography~\cite{nature1999miao}, optics~\cite{optics,oe2015bian}, 
 array and high-power coherent diffractive imaging~\cite{array,altmin2015}, astronomy~\cite{astronomy}, and microscopy~\cite{micro}. In these settings, due to physical limitations, optical sensors and detectors such as charge-coupled device (CCD) cameras, photosensitive films, and human eyes can record only the (squared) modulus of the Fresnel or Fraunhofer diffraction pattern,  while losing the phase of the incident light striking the object. 
It has been shown that reconstructing a discrete, finite-duration signal from its Fourier transform magnitudes is generally \emph{NP-complete}~\cite{nphard1}. 
Even 
checking quadratic feasibility (i.e., whether a solution to a given quadratic system exists or not) 
is itself an \emph{NP-hard} problem~\cite[Theorem 2.6]{2013nphard}. Thus, 
despite its simple form and practical relevance across various fields, 
tackling the quadratic system in~\eqref{eq:quad} is challenging and~\emph{NP-hard} in general.


\subsection{Prior art}
Adopting the least-squares criterion, 
the task of recovering $\bm{x}$ 
from data $y_i$ observed in additive white Gaussian noise (AWGN) can be recast as  
that of minimizing the \emph{intensity-based} 
empirical loss~\cite{wf} 
\begin{equation}\label{eq:ls}
\underset{\bm{z}\in\mathbb{R}^n/\mathbb{C}^n}{\text{minimize}}~~f(\bm{z}):=\frac{1}{2m}\sum_{i=1}^m\left(y_i-|\langle\bm{a}_i,\bm{z}\rangle|^2\right)^2.
\end{equation}
An alternative is to consider an \emph{amplitude-based} loss, in which $\psi_i$ is observed instead of $y_i$ in AWGN \cite{1978fienup,nips2016wg}
\begin{equation}\label{eq:ls1}
\underset{\bm{z}\in\mathbb{R}^n/\mathbb{C}^n}{\text{minimize}}~~h(\bm{z}):=\frac{1}{2m}\sum_{i=1}^m
\big(\psi_i-|\langle\bm{a}_i,\bm{z}\rangle|\big)^2.
\end{equation}  
Unfortunately, 
the presence of quadratic terms in \eqref{eq:ls} or the modulus in \eqref{eq:ls1} renders the corresponding objective function 
 nonconvex. 
Minimizing nonconvex objectives, which may exhibit many stationary points, is in general \emph{NP-hard}~\cite{hardproblems}. 
In fact, even checking whether a given point is a local minimum 
or establishing convergence to a local minimum 
turns out to be \emph{NP-complete}~\cite{hardproblems}. 


In the classical discretized one-dimensional ($1$D) phase retrieval, the amplitude vector $\bm{\psi}$ corresponds to the $n$-point Fourier transform of the $n$-dimensional signal $\bm{x}$~\cite{spm2015eldar}.  
It has been shown based on spectral factorization that in general there is no unique solution to $1$D phase retrieval, even if we disregard trivial ambiguities~\cite{1duniqueness}. To overcome this ill-posedness, several approaches have been suggested. One possibility is to assume additional constraints on the unknown signal such as sparsity~\cite{gespar,2013spr,altmin2015,acha2014yesm}. Other approaches rely on introducing redundancy into the measurements using for example, the short-time Fourier transform, or masks
\cite{2015stft,coded}. Finally, recent works assume
 random measurements (e.g., Gaussian $\{\bm{a}_i\}$ designs)~\cite{altmin,siam2015candes,acha2014yesm,
phaselift,wf,twf,mtwf}. Henceforth, this paper focuses on random measurements $\{\psi_i\}$ obtained from independently and identically distributed (i.i.d.) Gaussian $\{\bm{a}_i\}$ designs. 
 
Existing approaches to solving \eqref{eq:ls} (or related ones using the Poisson likelihood; see, e.g.,~\cite{twf}) or \eqref{eq:ls1} fall under two categories: nonconvex and convex ones. Popular nonconvex solvers include alternating projection~such as Gerchberg-Saxton~\cite{gerchberg} and Fineup~\cite{1978fienup}, 
AltMinPhase~\cite{altmin}, (Truncated) Wirtinger flow (WF/TWF)~\cite{wf,twf,thesis2014}, and Karzmarz variants~\cite{ip2015wei}  
as well as trust-region methods~\cite{sun2016}. 
Inspired by WF, other relevant 
judiciously initialized counterparts have also been developed for faster semidefinite optimization~\cite{procrustes,localcvx}, blind deconvolution~\cite{2016li}, and matrix completion~\cite{focs2015sun}.   
Convex counterparts on the other hand rely on the so-called \emph{matrix-lifting} technique or \emph{Shor}'s semidefinite relaxation to obtain the solvers abbreviated as PhaseLift~\cite{phaselift}, PhaseCut~\cite{phasecut}, and CoRK~\cite{2016huang}. Further approaches dealing with noisy or sparse phase retrieval are discussed in~\cite{gespar,mtwf,tsp2016qian,qianfu,staf,sparta,raf}.

In terms of sample complexity, it has been proven that\footnote{The notation $\phi(n)=\mathcal{O}(g(n))$ means that there is a constant $c>0$ such that $|\phi(n)|\le c|g(n)|$.} $\mathcal{O}(n)$ noise-free random measurements suffice for uniquely determining a general signal~\cite{acha2014yesm}. It is also self-evident that recovering a general $n$-dimensional $\bm{x}$ requires at least $\mathcal{O}(n)$ measurements. 
Convex approaches enable exact recovery from the optimal bound $\mathcal{O}(n)$ of noiseless Gaussian measurements~\cite{fcm2014candes}; they are based on solving a semidefinite program with a matrix variable of size $n\times n$, thus incurring worst-case computational complexity on the order of $\mathcal{O}(n^{4.5})$~\cite{phasecut} 
that does not scale well with the dimension $n$.
 Upon exploiting the underlying problem structure, $\mathcal{O}(n^{4.5})$ can be reduced to $\mathcal{O}(n^3)$~\cite{phasecut}.  
 Solving for vector variables, nonconvex approaches achieve significantly improved computational performance. Using formulation~\eqref{eq:ls1} and adopting a spectral initialization commonly employed in matrix completion~\cite{tit2010spectral}, AltMinPhase establishes exact recovery with sample complexity $\mathcal{O}(n\log^3 n)$ under i.i.d. Gaussian $\{\bm{a}_i\}$ designs with resampling~\cite{altmin}. 
 
Concerning formulation \eqref{eq:ls}, WF iteratively refines the spectral initial estimate by means of a gradient-like update, which can be approximately interpreted as a stochastic gradient descent variant~\cite{wf}, \cite{thesis2014}. 
The follow-up TWF improves upon WF through a truncation procedure to separate  
gradient components of excessively extreme (large or small) sizes. 
 Likewise, due to the heavy tails present in the initialization stage, 
data $\left\{y_i\right\}_{i=1}^m$ are pre-screened to yield improved initial estimates in the so-termed truncated spectral initialization method~\cite{twf}. 
WF allows exact recovery from $\mathcal{O}(n\log n)$ measurements in $\mathcal{O}(mn^2\log(1/\epsilon))$ time/flops to yield an $\epsilon$-accurate solution for any given $\epsilon>0$~\cite{wf}, while TWF advances these to $\mathcal{O}(n)$ measurements and $\mathcal{O}(mn\log(1/\epsilon))$ time~\cite{twf}.
Interestingly, the truncation procedure in the gradient stage turns out to be useful in avoiding spurious stationary points in the context of nonconvex optimization, as will be justified in Section \ref{sec:test} by the numerical comparison between our amplitude flow (AF) algorithms with or without the judiciously designed truncation rule.
It is also worth mentioning that when $m\ge Cn\log^3 n$ for some sufficiently large positive constant $C$, the objective function in \eqref{eq:ls1} is shown to admit benign geometric structure that allows certain iterative algorithms (e.g., trust-region methods) to efficiently find a global minimizer with random initializations~\cite{sun2016}. Hence, the challenge of solving systems of random quadratic equations lies in the case where a near-optimal number of equations are involved, e.g., $m=2n-1$ in the real-valued setting.  

Although achieving a linear (in the number of unknowns $n$) sample and computational complexity, the state-of-the-art TWF approach 
still requires at least $4n\sim 5n$ equations to yield stable empirical success rate (e.g., $\ge 99\%$) under the noiseless real-valued Gaussian model~\cite[Section~3]{twf}, which is more than twice the known information-limit of $m=2n-1$~\cite{2006balan}. Similar though less obvious results hold in the complex-valued scenario. While the truncated spectral initialization in~\cite{twf} improves upon the ``plain-vanilla'' spectral initialization, its performance still suffers when the number of measurements is relatively small and its advantage (over the untruncated one) diminishes as the number of measurements grows; see more details in Fig.~\ref{fig:mdifferent} and Section~\ref{sec:alg}.   
 Furthermore, extensive numerical and experimental validation confirms that the \emph{amplitude-based} cost function performs significantly better than the \emph{intensity-based} one~\cite{experimental2015}; that is, formulation \eqref{eq:ls1} is superior to \eqref{eq:ls}. 
Hence, besides enhancing initialization, 
markedly improved performance in the gradient stage can be expected by re-examining the amplitude-based   
cost function and incorporating judiciously designed gradient regularization rules.

\subsection{This paper}
Along the lines of suitably initialized nonconvex schemes~\cite{wf,twf} and inspired by~\cite{experimental2015}, the present paper develops a linear-time (i.e., the computational time linearly in both dimensions $m$ and $n$) algorithm to minimize the amplitude-based cost function,
referred to as \emph{truncated amplitude flow} (TAF). Our approach provably recovers an $n$-dimensional unknown 
signal $\bm{x}$ exactly from a near-optimal number of noiseless random measurements, while also featuring near-perfect statistical performance in the noisy setting. TAF operates in two stages: In the first stage, we introduce an orthogonality-promoting initialization that is computable using a few power iterations. Stage two refines 
 the initial estimate by successive updates of truncated generalized gradient iterations. 
 
 Our initialization is built upon the hidden orthogonality characteristics of high-dimensional random vectors~\cite{jmlr2013cai}, which is in contrast to spectral alternatives
originating from the strong law of large numbers (SLLN)~\cite{altmin2015,wf,twf}.     
Furthermore, one challenge of phase retrieval lies in reconstructing the signs/phases of $\langle\bm{a}_i,\bm{x}\rangle$ in the real-/complex-valued settings. Our TAF's refinement stage leverages  
 a simple yet effective regularization rule to eliminate the erroneously estimated phases in the generalized gradient components with high probability. 
 Simulated tests corroborate that the proposed initialization returns more accurate and robust initial estimates than its spectral counterparts in the noiseless and noisy settings. In addition,  
 our TAF (with gradient truncation)   
 markedly improves upon its ``plain-vanilla'' version AF. Empirical results demonstrate the advantage of TAF over its competing alternatives.    
 
 Focusing on the same amplitude-based cost function, an independent work develops the so-termed reshaped Wirtinger flow (RWF) algorithm \cite{reshaped}, which coincides with amplitude flow (AF). A slightly modified variant of spectral initialization \cite{wf} is used to obtain an initial guess, followed by a sequence of non-truncated generalized gradient iterations \cite{reshaped}. Numerical comparisons show that the proposed TAF method performs better than RWF especially when the number of equations approaches the information-theoretic limit ($2n-1$ in the real case).

 The remainder of this paper is organized as follows. The amplitude-based cost function, as well as the two algorithmic stages is described and analyzed in Section~\ref{sec:alg}. 
 Section~\ref{sec:main} summarizes the TAF algorithm and establishes its theoretical performance. 
Extensive simulated tests comparing TAF with Wirtinger-based approaches are presented in Section~\ref{sec:test}.  
Finally, main proofs are given in Section~\ref{sec:proof}, while technical details are deferred to the Appendix.

\section{Truncated Amplitude Flow}\label{sec:alg}
In this section, the two stages of our TAF algorithm are detailed. First, the challenge of handling the nonconvex and nonsmooth amplitude-based cost function is analyzed, and addressed by a carefully designed gradient regularization rule. Limitations of (truncated) spectral initializations are then pointed out, followed by
a simple motivating example to inspire our orthogonality-promoting initialization method. For concreteness, the analysis will focus on the real-valued Gaussian model with $\bm{x}\in\mathbb{R}^n$ and i.i.d. design vectors 
	$\bm{a}_i\in\mathbb{R}^n\sim \mathcal{N}(\bm{0},\bm{I}_n)$. Numerical experiments using the complex-valued Gaussian model with $\bm{x}\in\mathbb{C}^n$ and i.i.d. $\bm{a}_i\sim \mathcal{CN}(\bm{0},\bm{I}_n):= \mathcal{N}(\bm{0},\bm{I}_n/2)+j\mathcal{N}(\bm{0},\bm{I}_n/2)$ will be discussed briefly. 
	
	To start, let us define the Euclidean distance of any estimate $\bm{z}$ to the solution set: ${\rm dist}(\bm{z},\,\bm{x}):=
\min\,\{\left\|\bm{z}+\bm{x}\right\|,\,\left\|\bm{z}-\bm{x}\right\|\}$ for real signals, and ${\rm dist}(\bm{z},\,\bm{x}):=\minimize_{\phi \in[0,2\pi)}\|\bm{z}-\bm{x}{\rm e}^{i\phi}\|$ for complex ones~\cite{wf}, where $\|\!\cdot\!\|$ denotes the Euclidean norm.  
Define also the indistinguishable global phase constant in the real-valued setting as
\vspace{-.em}
\begin{equation}\label{eq:adaptation}
	\phi(\bm{z}):=\left\{\begin{array}{lll}
		0,&\|\bm{z}-\bm{x}\|\le \|\bm{z}+\bm{x}\|,\\
		\pi,&{\rm otherwise.}
	\end{array}\right.
	\vspace{-.em}
\end{equation}
Henceforth, fixing $\bm{x}$ to be any solution of the given quadratic system~\eqref{eq:quad}, we always assume that $\phi\left({\bm{z}}\right)=0$; otherwise, ${\bm{z}}$ is replaced by ${\rm e}^{-j\phi\left({\bm{z}}\right)}{\bm{z}}$, but for simplicity of presentation, the constant phase adaptation term ${\rm e}^{-j\phi\left({\bm{z}}\right)}$ will be dropped whenever it is clear from the context.

\subsection{Truncated generalized gradient stage}
For brevity, collect all vectors $\{\bm{a}_i\}_{i=1}^m$ in the $m\times n$ matrix $\bm{A}:=\left[\bm{a}_1~\cdots~\bm{a}_m\right]^\ccalT$, and all amplitudes $\left\{\psi_i\right\}_{i=1}^m$ to form the vector $\bm{\psi}:=\left[\psi_1~\cdots~\psi_m\right]^\ccalT$.
One can rewrite the amplitude-based cost function in matrix-vector representation as
 \vspace{-.em}
 \begin{equation}
 	\label{eq:prob}\vspace{-.em}
\underset{\bm{z}\in\mathbb{R}^n}{\text{minimize}}~~\ell(\bm{z}):=\frac{1}{m}\sum_{i=1}^m\ell_i(\bm{z})
 	=\frac{1}{2m}\big\|\bm{\psi}-\left|\bm{A}\bm{z}\right|
 	\big\|^2
 	\vspace{-.em}
 \end{equation}  
 where $\ell_i(\bm{z}):=\frac{1}{2}(\psi_i-|\bm{a}_i^\ccalT\bm{z}|)^2$ with the superscript $^\ccalT$ ($^\ccalH$) denoting (Hermitian) transpose; 
  and with a slight abuse of notation, $|\bm{A}\bm{z}|:=[|\bm{a}_1^\ccalT\bm{z}|~\cdots~|\bm{a}_m^\ccalT\bm{z}|]^\ccalT$. Apart from being nonconvex, $\ell(\bm{z})$ is also nondiffentiable, hence challenging the algorithmic design and analysis. 
 In the presence of smoothness or convexity, convergence analysis of iterative algorithms relies either on  continuity of the gradient (ordinary gradient methods)~\cite{shor1972class}, or, on the convexity of the objective functional (subgradient methods)~\cite{book1998rockafellar}. 
 Although subgradient methods have found widespread applicability in nonsmooth optimization, they are limited to the class of convex functions~\cite[Page 4]{book1985shor}. 
In nonconvex nonsmooth optimization settings, 
the so-termed \emph{generalized gradient} broadens the scope of the (sub)gradient to the class of \emph{almost everywhere} differentiable functions~\cite{book1990clarke}.

 Consider a continuous but not necessarily differentiable function $h(\bm{z})\in\mathbb{R}$ defined over an open region $\mathcal{S}\subseteq\mathbb{R}^n$. We then have the following definition.
\vspace{-.em}
\begin{definition}\cite[Definition 1.1]{clarke1975gg}
	The generalized gradient of a function $h$ at $\bm{z}$, denoted by $\partial h$, is the convex hull of the set of limits of the form $\lim\nabla h(\bm{z}_k)$, where $\bm{z}_k\to\bm{z}$ as $k\to +\infty$, i.e.,
	\vspace{-0.em}
	\begin{equation*}
				\vspace{-0.em}
		\partial h(\bm{z}):={\rm conv}\Big\{\lim_{k\to+\infty}\nabla h(\bm{z}_k):\bm{z}_k\to\bm{z},\;\bm{z}_k\notin \mathcal{G}_\ell
		\Big\} 
	\end{equation*}
	where the symbol `\emph{conv}' signifies the convex hull of a set, and $\mathcal{G}_\ell$ denotes the set of points in $\mathcal{S}$ at which $h$ fails to be differentiable. 
\end{definition}

 Having introduced the notion of a generalized gradient, and with $t$ denoting the iteration count, our approach to solving \eqref{eq:prob} amounts to iteratively refining the initial guess $\bm{z}_0$ (returned by the orthogonality-promoting initialization method to be detailed shortly) by means of the ensuing \emph{truncated} generalized gradient iterations 
\begin{equation}\label{eq:iter}
	\bm{z}_{t+1}=\bm{z}_t-\mu_t\, \partial \ell_{\rm tr}(\bm{z}_t).
\end{equation}
Here, $\mu_t>0$ is the step size, and the (truncated) generalized gradient $\partial \ell_{\rm tr}(\bm{z}_t)$ is given by
\vspace{-.em}
\begin{equation}\label{eq:agg}
\partial \ell_{\rm tr}(\bm{z}_t):=\frac{1}{m}\sum_{i\in\mathcal{I}_{t+1}}\left(\bm{a}_i^\ccalT\bm{z}_t-\psi_i\frac{\bm{a}_i^\ccalT\bm{z}_t}{|\bm{a}_i^\ccalT\bm{z}_t|}
\right)\bm{a}_i  
\vspace{-.em}
\end{equation}
for some index set $\mathcal{I}_{t+1}\subseteq [m]:=\left\{1,2,\ldots,m\right\}$ to be designed next. The convention $\frac{\bm{a}_i^\ccalT\bm{z}_t}{|\bm{a}_i^\ccalT\bm{z}_t|}:=0$ is adopted, if $\bm{a}_i^\ccalT\bm{z}_t=0$. 
It is easy to verify that the update in~\eqref{eq:iter} with a full generalized gradient in~\eqref{eq:agg}  
 monotonically decreases the objective function value in~\eqref{eq:prob}. 
 

Any stationary point $\bm{z}^\ast$ of $\ell(\bm{z})$ can be characterized by the following fixed-point equation~\cite{2015chen1,2015chen2}
\vspace{-.em}
\begin{equation}\label{eq:fixpoint}
	\bm{A}^\ccalT\left(\bm{A}\bm{z}^\ast-\bm{\psi}\odot\frac{\bm{A}\bm{z}^\ast}{|\bm{A}\bm{z}^\ast|}\right)=\bm{0}
	\vspace{-.em}
\end{equation} 
for entry-wise product $\odot$, 
which may have many solutions. Clearly, if $\bm{z}^\ast$ is a solution, then so is $-\bm{z}^\ast$. 
Furthermore, both solutions/global minimizers $\bm{x}$ and $-\bm{x}$ satisfy~\eqref{eq:fixpoint} due to the fact that $\bm{A}\bm{x}-\bm{\psi}\odot\frac{\bm{A}\bm{x}}{\left|\bm{A}\bm{x}\right|}=\bm{0}$. Considering any stationary point $\bm{z}^\ast\ne \pm\bm{x}$ that has been adapted such that $\phi(\bm{z}^\ast)=0$, one can write
\begin{equation}	\bm{z}^\ast=\bm{x}+(\bm{A}^\ccalT\bm{A})^{-1}\bm{A}^\ccalT\left[\bm{\psi}\odot\left(\tfrac{\bm{A}\bm{z}^\ast}{|\bm{A}\bm{z}^\ast|}-\tfrac{\bm{A}\bm{x}}{|\bm{A}\bm{x}|}\right)\right].\label{eq:neq}
\end{equation}
Thus, a necessary condition for $\bm{z}^\ast\ne\bm{x}$ in~\eqref{eq:neq} is $\frac{\bm{A}\bm{z}^\ast}{\left|\bm{A}\bm{z}^\ast\right|}\ne\frac{\bm{A}\bm{x}}{\left|\bm{A}\bm{x}\right|}$. Expressed differently, there must be sign differences between 
$\bm{A}\bm{z}^\ast$ and $\bm{A}\bm{x}$ whenever one gets stuck with an undesirable stationary point $\bm{z}^\ast$. 
Inspired by this observation, it is reasonable to devise algorithms that can detect and separate out the generalized gradient components corresponding to mistakenly estimated signs $\left\{\frac{\bm{a}_i^\ccalT\bm{z}_t}{|\bm{a}_i^\ccalT\bm{z}_t|}\right\}$ along the iterates $\{\bm{z}_t\}$. 

Precisely, if $\bm{z}_t$ and $\bm{x}$ lie at different sides of the hyperplane $\bm{a}_i^\ccalT\bm{z}=0$, then the sign of $\bm{a}_i^\ccalT\bm{z}_t$ will be different than that of $\bm{a}_i^\ccalT\bm{x}$; that is, $\frac{\bm{a}_i^\ccalT\bm{x}}{|\bm{a}_i^\ccalT\bm{x}|}\ne \frac{\bm{a}_i^\ccalT\bm{z}}{|\bm{a}_i^\ccalT\bm{z}|}$. Specifically, one can re-write the $i$-th generalized gradient component as
\begin{align}
	\partial \ell_i(\bm{z})&=\Big(\bm{a}_i^\ccalT\bm{z}-\psi_i\frac{\bm{a}_i^\ccalT\bm{z}}{|\bm{a}_i^\ccalT\bm{z}|}\Big)\bm{a}_i
	\nonumber\\&
	=\Big(\bm{a}_i^\ccalT\bm{z}-|\bm{a}_i^\ccalT\bm{x}|\frac{\bm{a}_i^\ccalT\bm{x}}{|\bm{a}_i^\ccalT\bm{x}|}\Big)\bm{a}_i+\Big(\frac{\bm{a}_i^\ccalT\bm{x}}{|\bm{a}_i^\ccalT\bm{x}|}-\frac{\bm{a}_i^\ccalT\bm{z}}{|\bm{a}_i^\ccalT\bm{z}|}\Big)\psi_i\bm{a}_i\nonumber\\
	&= \bm{a}_i\bm{a}_i^\ccalT(\bm{z}-\bm{x})+\Big(\frac{\bm{a}_i^\ccalT\bm{x}}{|\bm{a}_i^\ccalT\bm{x}|}-\frac{\bm{a}_i^\ccalT\bm{z}}{|\bm{a}_i^\ccalT\bm{z}|}\Big)\psi_i\bm{a}_i\nonumber\\
	&
	 =\bm{a}_i\bm{a}_i^\ccalT\bm{h}+\underbrace{\Big(\frac{\bm{a}_i^\ccalT\bm{x}}{|\bm{a}_i^\ccalT\bm{x}|}-\frac{\bm{a}_i^\ccalT\bm{z}}{|\bm{a}_i^\ccalT\bm{z}|}\Big)\psi_i\bm{a}_i}_{\buildrel\triangle\over
	 =\,\bm{r}_i},
	\label{eq:split}
\end{align}
where $\bm{h}:=\bm{z}-\bm{x}$. 
Intuitively, the SLLN asserts that averaging the first term $\bm{a}_i\bm{a}_i^\ccalT\bm{h}$ over $m$ instances approaches $\bm{h}$, which qualifies it as a desirable search direction. However, 
certain generalized gradient entries involve erroneously estimated signs of $\bm{a}_i^\ccalT\bm{x}$; hence, nonzero $\bm{r}_i$ terms exert a negative influence on the search direction $\bm{h}$ by dragging the iterate away from $\bm{x}$, and they typically have sizable magnitudes as will be further elaborated in Remark~\ref{rmk:mag} shortly.

\begin{figure}[ht]
\begin{center}
\centerline{\includegraphics[scale = 0.6]{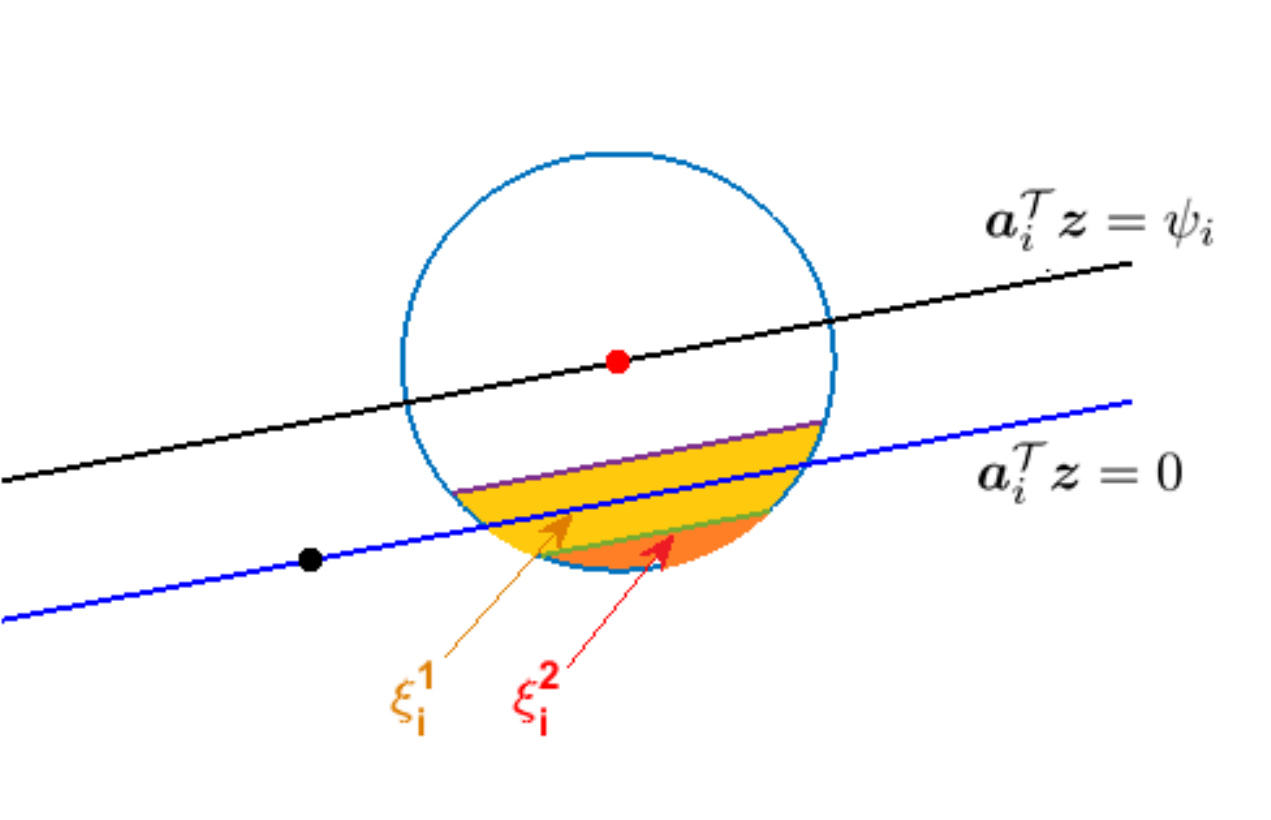}}
\vspace{-0em}
\caption{Geometric description of the proposed truncation rule on the $i$-th gradient component involving $\bm{a}_i^\ccalT\bm{x}=\psi_i$, where the red dot denotes the solution $\bm{x}$ and the black one is the origin. Hyperplanes $\bm{a}_i^\ccalT\bm{z}=\psi_i$ and $\bm{a}_i^\ccalT\bm{z}=0$ (of $\bm{z}\in\mathbb{R}^n$) passing through points $\bm{z}=\bm{x}$ and $\bm{z}=\bm{0}$, respectively, are shown. 
}

\label{fig:truncation}
\end{center}
\end{figure}
Figure~\ref{fig:truncation} demonstrates this from a geometric perspective, where the black dot denotes the origin, and the 
red dot the solution $\bm{x}$; here, $-\bm{x}$ is omitted for ease of exposition. Assume without loss of generality that the $i$-th missing sign is positive, i.e., $\bm{a}_i^\ccalT\bm{x}=\psi_i$. 
 As will be demonstrated in Theorem~\ref{thm:initial}, with high probability, the initial estimate returned by our orthogonality-promoting method obeys $\|\bm{h}\|\le \rho\|\bm{x}\|$ for some sufficiently small constant $\rho>0$. Therefore,
 all points lying on or within the circle (or sphere in high-dimensional spaces) in Fig.~\ref{fig:truncation} 
 satisfy $\|\bm{h}\|\le \rho\|\bm{x}\|$.  
If $\bm{a}_i^\ccalT\bm{z}=0$ does not intersect with the circle, then all points within the circle satisfy $\frac{\bm{a}_i^\ccalT\bm{z}}{|\bm{a}_i^\ccalT\bm{z}|}=\frac{\bm{a}_i^\ccalT\bm{x}}{|\bm{a}_i^\ccalT\bm{x}|}$ qualifying the $i$-th generalized gradient as a desirable search (descent) direction in~\eqref{eq:split}.   
  If, on the other hand, $\bm{a}_i^\ccalT\bm{z}=0$ intersects the circle, then points lying on the same side of $\bm{a}_i^\ccalT\bm{z}=0$ with $\bm{x}$ in Fig.~\ref{fig:truncation} admit correctly estimated signs, while points lying on different sides of $\bm{a}_i^\ccalT\bm{z}=0$ with $\bm{x}$ would have $\frac{\bm{a}_i^\ccalT\bm{z}}{|\bm{a}_i^\ccalT\bm{z}|}\ne \frac{\bm{a}_i^\ccalT\bm{x}}{|\bm{a}_i^\ccalT\bm{x}|}$. This gives rise to a corrupted search direction in~\eqref{eq:split}, implying that the corresponding generalized gradient component should be eliminated.

Nevertheless, it is difficult or even impossible to check whether the sign of $\bm{a}_i^\ccalT\bm{z}_t$ equals  that of $\bm{a}_i^\ccalT\bm{x}$. Fortunately, as demonstrated in Fig.~\ref{fig:truncation}, 
most spurious generalized
 gradient components (those corrupted by nonzero $\bm{r}_i$ terms)  hover around the watershed hyperplane $\bm{a}_i^\ccalT\bm{z}_t=0$. For this reason, TAF includes only those components 
having $\bm{z}_t$ sufficiently away from its watershed, i.e., 
\vspace{-.em}
\begin{equation}\label{eq:large}
	\mathcal{I}_{t+1}:=\left\{1\le i\le m\left|\frac{|\bm{a}_i^\ccalT\bm{z}_t|}{|\bm{a}_i^\ccalT\bm{x}|}\ge \frac{1}{1+\gamma}\right. \right\},\quad t\ge 0 
	\vspace{-.em}
\end{equation}
for an appropriately selected threshold $\gamma>0$. To be more specific, the light yellow color-coded area denoted by $\xi_{i}^1$ in Fig.~\ref{fig:truncation} 
signifies the truncation region of $\bm{z}$:
if $\bm{z}\in\xi_{i}^1$ satisfies the condition in~\eqref{eq:large}, then the corresponding generalized gradient component $\partial\ell_i(\bm{z};\psi_i)$ will be thrown out. However, the truncation rule may mis-reject certain `good' gradients if $\bm{z}_t$ lies in the upper part of $\xi_i^1$; `bad' gradients may be missed as well if $\bm{z}_t$ belongs to the spherical cap $\xi_i^2$.  
Fortunately, as we will show in~Lemmas~\ref{le:1stterm} and \ref{le:2ndterm}, 
the probabilities of misses and mis-rejections are provably very small, hence precluding a noticeable influence on the descent direction. Although not perfect, it turns out that 
such a regularization rule succeeds in detecting and eliminating most corrupted generalized gradient components with high probability, therefore maintaining a well-behaved search direction.

Regarding our gradient regularization rule in~\eqref{eq:large}, two observations are in order.
\begin{remark}\label{rmk:gg}
The truncation rule in~\eqref{eq:large} includes only relatively sizable $\bm{a}_i^\ccalT\bm{z}_t$'s, hence enforcing the smoothness of the (truncated) objective function $\ell_{\rm tr}(\bm{z}_t)$ at $\bm{z}_t$. Therefore, the truncated generalized gradient $\partial\ell_{\rm tr}(\bm{z})$
employed in~\eqref{eq:iter} and \eqref{eq:agg} 
boils down to the ordinary gradient/Wirtinger derivative $\nabla\ell_{\rm tr}(\bm{z}_t)$ in the real-/complex-valued case. 
\end{remark}

\begin{remark}\label{rmk:mag}
As will be elaborated in \eqref{eq:halfnormalmean} and \eqref{eq:logn}, the quantities
$\left(1/m\right)\sum_{i=1}^m \psi_i$ and  $\max_{i\in[m]}\psi_i$ in~\eqref{eq:split} have magnitudes on the order of  $\sqrt{\pi/2}\|\bm{x}\|$ and $\sqrt{m}\|\bm{x}\|$, respectively. In contrast, Proposition \ref{prop:initial} asserts that the first term in~\eqref{eq:split} obeys $\|\bm{a}_i\bm{a}_i^\ccalT\bm{h}\|\approx 
\|\bm{h}\|\le \rho\|\bm{x}\|$ 
 for a sufficiently small $\rho\ll \sqrt{\pi/2}$. 
Thus, spurious generalized gradient components typically have large magnitudes. It turns out that our gradient regularization rule in~\eqref{eq:large} also throws out gradient components of large sizes. To see this, for all $\bm{z}\in\mathbb{R}^n$ such that~$\|\bm{h}\|\le \rho\|\bm{x}\|$ in \eqref{eq:i1x}, one can re-express 
\begin{equation}
	\sum_{i=1}^m\partial\ell_i(\bm{z})=	\sum_{i=1}^m\underbrace{\left(1-\frac{|\bm{a}_i^\ccalT\bm{x}|}{|\bm{a}_i^\ccalT\bm{z}|}\right)}_{\buildrel\triangle\over =\,\beta_i}\bm{a}_i\bm{a}_i^\ccalT\bm{z}
\label{eq:ggsize}
\end{equation}
for some weight $\beta_i\in[-\infty,\,1)$ assigned to the direction $\bm{a}_i\bm{a}_i^\ccalT\bm{z}\approx \bm{z}$ due to $\mathbb{E}[\bm{a}_i\bm{a}_i^\ccalT]=\bm{I}_n$. 
 Then $\partial\ell_i(\bm{z})$ of an excessively large size corresponds to a large $|\bm{a}_i^\ccalT\bm{x}|/|\bm{a}_i^\ccalT\bm{z}|$ in~\eqref{eq:ggsize}, 
 or equivalently a small $|\bm{a}_i^\ccalT\bm{z}|/|\bm{a}_i^\ccalT\bm{x}|$ in~\eqref{eq:large}, thus causing the corresponding $\partial \ell_i(\bm{z})$ to be eliminated according to the truncation rule in~\eqref{eq:large}. 
\end{remark}

Our truncation rule deviates from the intuition behind TWF, which throws away gradient components corresponding to large-size $\{|\bm{a}_i^\ccalT\bm{z}_t|/|\bm{a}_i^\ccalT\bm{x}|\}$ in~\eqref{eq:large}. 
As demonstrated by our analysis in Appendix~\ref{sec:rare}, it rarely happens that a gradient component having large $|\bm{a}_i^\ccalT\bm{z}_t|/|\bm{a}_i^\ccalT\bm{x}|$ yields an incorrect sign of $\bm{a}_i^\ccalT\bm{x}$ under a sufficiently accurate initialization. Moreover, discarding too many samples (those for which $i\notin\mathcal{T}_{t+1}$ in TWF~\cite[Section 2.1]{twf}) introduces large bias into $(1/m)\sum_{i\in\mathcal{T}_{t+1}}^m\bm{a}_i\bm{a}_i^\ccalT\bm{h}$, so that TWF does not work well when $m/n$ is close to the information-limit of $m/n\approx 2$.   
In sharp contrast, the motivation and objective of our truncation rule in~\eqref{eq:large} is to directly sense and eliminate gradient components that involve mistakenly estimated signs with high probability. 

To demonstrate the power of TAF,    
 numerical tests comparing all stages of (T)AF and (T)WF will be presented throughout our analysis.
The basic test settings used in this paper are described next.  
For fairness, all pertinent algorithmic parameters involved in all compared schemes were set to their default values.
Simulated estimates are averaged over $100$ independent Monte Carlo (MC) realizations without mentioning this explicitly each time.  
Performance of different schemes is evaluated in terms of the relative root mean-square error, i.e.,  
\begin{equation}
{\rm Relative~error}:=\frac{{\rm dist}(\bm{z},\,\bm{x})}{\|\bm{x}\|},
\end{equation}
and the success rate among $100$ trials, where a success is claimed for a trial if the returned estimate incurs a relative error less than $10^{-5}$~\cite{twf}.  
Simulated tests under both noiseless and noisy Gaussian models are performed, corresponding to $\psi_i=\big|\bm{a}_i^\ccalH\bm{x}+\eta_i\big|$~\cite{altmin}
with $\eta_i=0$ and $\eta_i\sim\mathcal{N}(0,\sigma^2)$, respectively, with i.i.d. $\bm{a}_i\sim\mathcal{N}(\bm{0},\bm{I}_n)$ or $\bm{a}_i\sim\mathcal{CN}(\bm{0},\bm{I}_n)$.    

Numerical comparison depicted in Fig.~\ref{fig:sameinit} using the noiseless real-valued Gaussian model
suggests that 
even when starting with the \emph{same truncated spectral initialization}, TAF's refinement outperforms those of TWF and WF, demonstrating the merits of our gradient update rule over TWF/WF. Furthermore, comparing TAF (gradient iterations in \eqref{eq:iter}-\eqref{eq:agg} with truncation in~\eqref{eq:large} initialized by the truncated spectral estimate)  
and AF (gradient iterations in \eqref{eq:iter}-\eqref{eq:agg} initialized by the truncated spectral estimate) corroborates the power of the truncation rule in~\eqref{eq:large}.

 \begin{figure}[h]
\centering
    \includegraphics[scale=0.6]{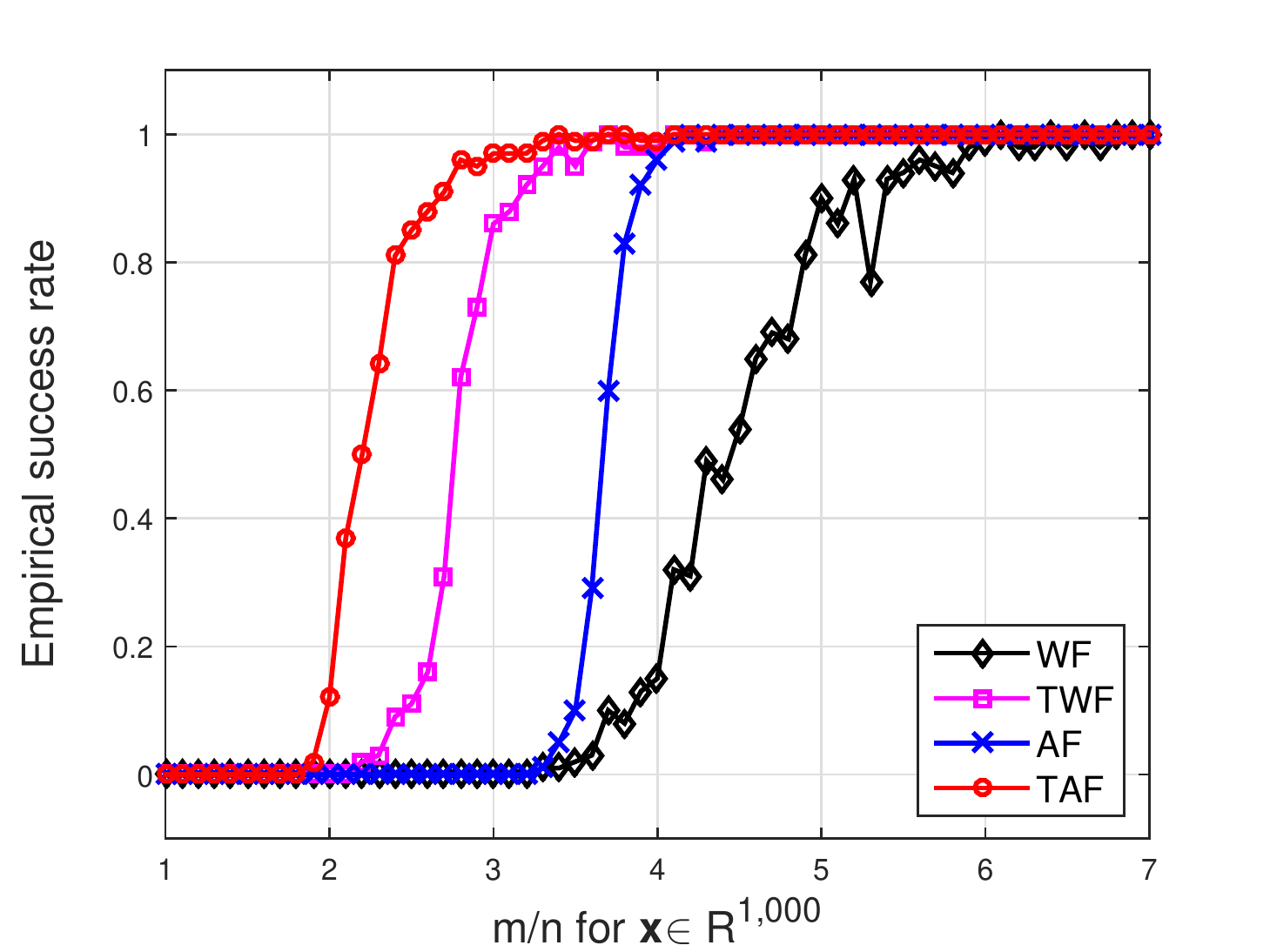}
   \vspace{-0pt}
  \caption{ Empirical success rate for WF, TWF, AF, and TAF with the same truncated spectral initialization under the noiseless real-valued Gaussian model.}
      \label{fig:sameinit}
\end{figure}



\begin{figure}
\vspace{-0pt}
\centering
    \includegraphics[scale=0.6]{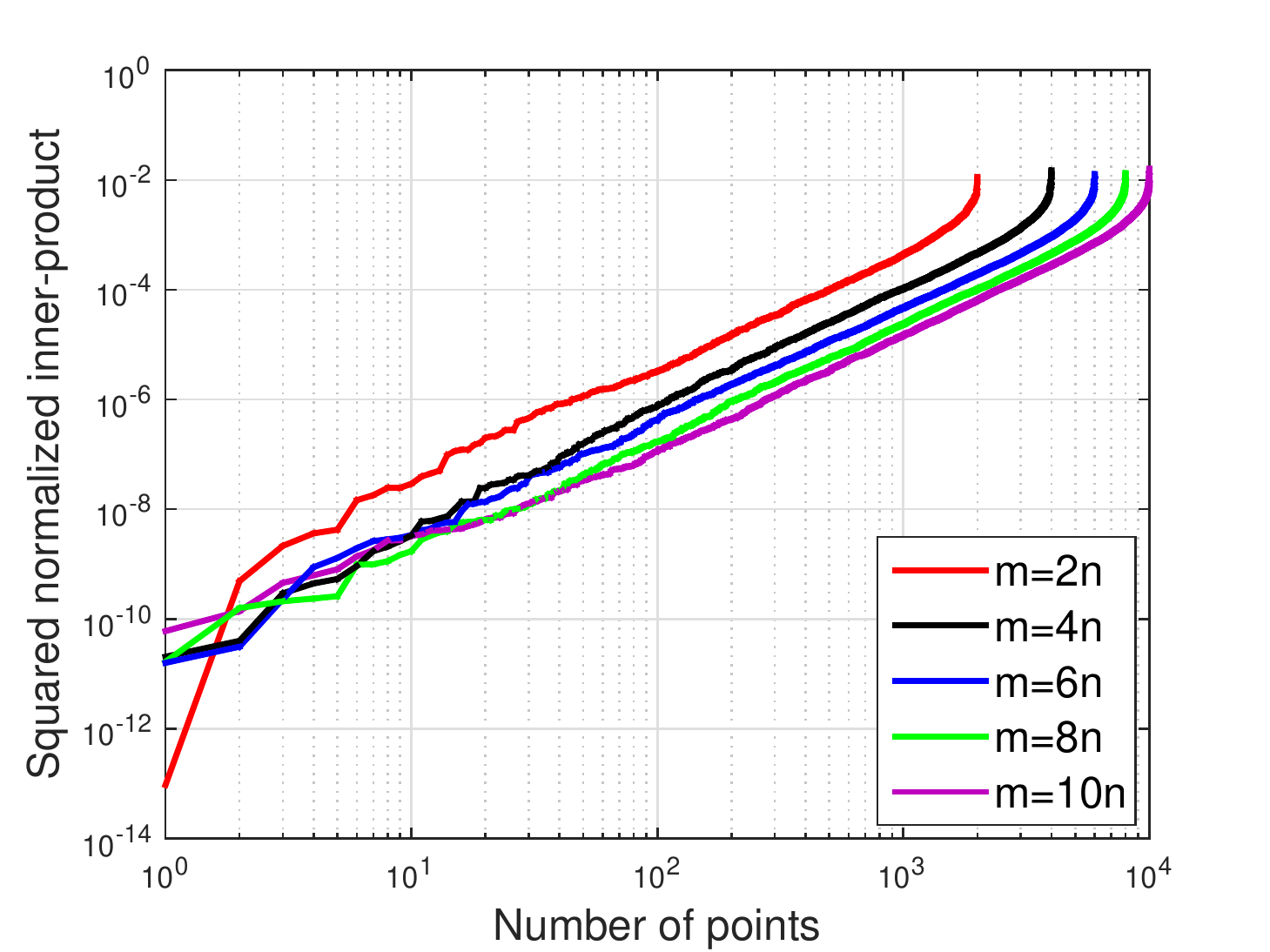}
    \vspace{-0pt}
\caption{Ordered squared normalized inner-product for pairs $\bm{x}$ and $\bm{a}_i$, $\forall i\in[m]$ with $m/n$ varying by $2$ from $2$ to $10$, and $n=1,000$. 
}
\label{fig:inner}
  \vspace{-0pt}
\end{figure}


\vspace{-.em}
\subsection{Orthogonality-promoting initialization stage}
Leveraging the SLLN, spectral initialization methods estimate $\bm{x}$ as the (appropriately scaled) leading eigenvector of $\bm{Y}:=\frac{1}{m}\sum_{i\in\mathcal{T}_0}y_i\bm{a}_i\bm{a}_i^\ccalT$, where $\mathcal{T}_0$ is an index set accounting for possible data truncation. As asserted in~\cite{twf}, each summand $(\bm{a}_i^\ccalT\bm{x})^2\bm{a}_i\bm{a}_i^\ccalT$ follows a heavy-tail probability density function lacking a moment generating function. This causes major performance degradation especially when the number of measurements is small. Instead of spectral initializations, we shall take another route to bypass this hurdle.
To gain intuition into our initialization, 
 a motivating example is presented first that reveals fundamental characteristics of high-dimensional random vectors. 

Fixing any nonzero vector $\bm{x}\in\mathbb{R}^{n}$, generate data $\psi_i=|\langle \bm{a}_i,\bm{x}\rangle|$ using i.i.d. $\bm{a}_i\sim\mathcal{N}(\bm{0},\bm{I}_n)$, $1\le i\le m$. Evaluate the following squared normalized inner-product 
\begin{equation}\label{eq:nip}
	\cos^2\theta_i:= \frac{\left|\langle \bm{a}_i,\bm{x}\rangle\right|^2}{\|\bm{a}_i\|^2\|\bm{x}\|^2}=\frac{\psi_i^2}{\|\bm{a}_i\|^2\|\bm{x}\|^2},\quad1\le i\le m
\end{equation}
where $\theta_i$ is the angle between vectors $\bm{a}_i$ and $\bm{x}$.   
Consider ordering all $\{\cos^2\theta_i\}$ in an ascending fashion, and collectively 
denote them as $\bm{\xi}:=[\cos^2\theta_{[m]}~\cdots~\cos^2\theta_{[1]}]^\ccalT$ with 
$\cos^2\theta_{[1]}
\ge\cdots\ge\cos^2\theta_{[m]}$. Figure~\ref{fig:inner} plots the ordered entries in $\bm{\xi}$ for $m/n$ varying by $2$ from $2$ to $10$ 
with $n=1,000$. Observe that almost all 
$\left\{\bm{a}_i\right\}$ vectors have a squared normalized inner-product with $\bm{x}$ 
smaller than $10^{-2}$, while half of the inner-products are less than $10^{-3}$, which implies that $\bm{x}$ is nearly orthogonal to a large number of 
$\bm{a}_i$'s.  

This example corroborates the folklore that random vectors in high-dimensional spaces are almost always nearly orthogonal to each other~\cite{jmlr2013cai}.
This inspired us to pursue an \emph{orthogonality-promoting initialization method}. Our key idea is to approximate $\bm{x}$ by a vector that is most orthogonal to a subset of vectors $\{\bm{a}_i\}_{i\in\mathcal{I}_0}$, where $\mathcal{I}_0$ is an index set with cardinality $|\mathcal{I}_0|<m$ that includes indices of the smallest 
squared normalized 
inner-products $\left\{\cos^2\theta_i\right\}$.
Since $\left\|\bm{x}\right\|$ appears in all inner-products, its exact value does not influence their ordering.  
Henceforth, we assume with no loss of generality that $\|\bm{x}\|=1$.

Using data $\left\{(\bm{a}_i;\,\psi_i)\right\}$, evaluate $\cos^2\theta_i$ according to~\eqref{eq:nip} for each pair $\bm{x}$ and $\bm{a}_i$.
Instrumental for the ensuing derivations is noticing from the inherent near-orthogonal property of high-dimensional random vectors 
that the summation of $\cos^2\theta_i$ over all indices $i\in\mathcal{I}_0$ should be very small; rigorous justification is deferred to Section~\ref{sec:proof}. Therefore, the sum  $\sum_{i\in\mathcal{I}_0}\cos^2\theta_i$ is also small, or according to \eqref{eq:nip}, equivalently, 
\begin{equation}
\sum_{i\in\mathcal{I}_0}\frac{|\langle\bm{a}_i,\bm{x}\rangle|^2}{\|\bm{a}_i\|^2\|\bm{x}\|^2}=\frac{\bm{x}}{\|\bm{x}\|}\Big(\sum_{i\in\mathcal{I}_0}\frac{\bm{a}_i\bm{a}_i^\ccalT}{\|\bm{a}_i\|^2}\Big)\frac{\bm{x}}{\|\bm{x}\|}
\end{equation} 
is small. Therefore, 
a meaningful approximation of $\bm{x}$ 
can be obtained by minimizing the former with $\bm{x}$ replaced by the optimization variable $\bm{z}$, namely  
\vspace{-.em}
\begin{equation}\label{eq:mineig}
	\underset{\|\bm{z}\|=1}{\text{minimize}}
~~
	\bm{z}^\ccalT\left(\frac{1}{|\mathcal{I}_0|}\sum_{i\in\mathcal{I}_0}\frac{\bm{a}_i\bm{a}_i^\ccalT}{\|\bm{a}_i\|^2}\right)\bm{z}.
\end{equation}	
This amounts to finding the smallest eigenvalue and the associated eigenvector of 
$\bm{Y}_0:=\frac{1}{|\mathcal{I}_0|}\sum_{i\in\mathcal{I}_0}\frac{\bm{a}_i\bm{a}_i^\ccalT}{\|\bm{a}_i\|^2}\succeq\bm{0}$ (the symbol $\succeq$ means positive semidefinite).  
Finding the smallest eigenvalue calls for eigen-decomposition or matrix inversion, each typically requiring computational complexity on the order of $\mathcal{O}(n^3)$. Such a computational burden may be intractable when $n$ grows large. Applying a standard 
concentration result, we show how the computation can be significantly reduced.

Since $\bm{a}_i/\|\bm{a}_i\|$ has unit norm and is uniformly distributed on the unit sphere, it is uniformly spherically distributed.\footnote{A random vector $\bm{z}\in\mathbb{R}^n$ is said to be spherical (or spherically symmetric) if its distribution does not change under rotations of the coordinate system; that is, the distribution of $\bm{P}\bm{z}$ coincides with that of $\bm{z}$ for any given orthogonal $n\times n$ matrix $\bm{P}$.}  
Spherical symmetry implies that ${\bm{a}_i}/{\|\bm{a}_i\|}$ has zero mean and covariance matrix $\bm{I}_n/n$~\cite{chap2010vershynin}.   
Appealing again to the SLLN, the sample covariance matrix $\frac{1}{m}\sum_{i=1}^m\frac{\bm{a}_i\bm{a}_i^\ccalT}{\|\bm{a}_i\|^2}
$ approaches $\bm{I}_n/n$ as $m$ grows. Simple derivations lead to 
\begin{equation}
	\sum_{i\in\mathcal{I}_0}\frac{\bm{a}_i\bm{a}_i^\ccalT}{\|\bm{a}_i\|^2}
=\sum_{i=1}^m\frac{\bm{a}_i\bm{a}_i^\ccalT}{\|\bm{a}_i\|^2}-\sum_{i\in\widebar{\mathcal{I}}_0}\frac{\bm{a}_i\bm{a}_i^\ccalT}{\|\bm{a}_i\|^2}
\approxeq \frac{m}{n}\bm{I}_n-\sum_{i\in\widebar{\mathcal{I}}_0}\frac{\bm{a}_i\bm{a}_i^\ccalT}{\|\bm{a}_i\|^2}
\end{equation} 
where $\widebar{\mathcal{I}}_0$ is the complement of $\mathcal{I}_0$ in the set $[m]$. 
Define $\bm{S}:=\left[\bm{a}_1/\|\bm{a}_1\|~\cdots~\bm{a}_m/\|\bm{a}_m\|\right]^\ccalT\in\mathbb{R}^{m\times n}$, and 
form $\widebar{\bm{S}}_0$ by removing the rows of $\bm{S}$ whose indices belong to ${\mathcal{I}}_0$.
Seeking the smallest eigenvalue of $\bm{Y}_0=\frac{1}{|\mathcal{I}_0|}\bm{S}_0^\ccalT\bm{S}_0$ then reduces to computing the largest eigenvalue of the matrix
\begin{equation}
\widebar{\bm{Y}}_0:=\frac{1}{|\widebar{\mathcal{I}}_0|}\widebar{\bm{S}}^\ccalT_0\widebar{\bm{S}}_0,
\label{eq:y0}	
\end{equation}
namely, 
\begin{equation}
\label{eq:maxeig}
	\tilde{\bm{z}}_0:=\arg\max_{\|\bm{z}\|=1}~~\bm{z}^\ccalT\widebar{\bm{Y}}_0\bm{z}
\vspace{-.em}
\end{equation}	
which can be efficiently solved via simple power iterations. 

When $\|\bm{x}\|\ne 1$, the estimate $\tilde{\bm{z}}_0$ from~\eqref{eq:maxeig} is scaled so that its norm matches approximately that of $\bm{x}$, 
 which is estimated as $\sqrt{\frac{1}{m}\sum_{i=1}^m y_i}$, or more accurately $\sqrt{\frac{n\sum_{i=1}^my_i}{\sum_{i=1}^m\|\bm{a}_i\|^2}}$. To motivate these estimates, using the rotational invariance property of normal distributions, it suffices to consider the case where $\bm{x}=\|\bm{x}\|\bm{e}_1$, with $\bm{e}_1$ denoting the first canonical vector of $\mathbb{R}^n$. Indeed,
\begin{align}
	\Big|\Big\langle\bm{a}_i, \frac{\bm{x}}{\|\bm{x}\|}\Big\rangle\Big|^2&=\left|\left\langle\bm{a}_i,\bm{U}\bm{e}_1\right\rangle\right|^2	=\left|\left\langle\bm{U}^\ccalT\bm{a}_i,\bm{e}_1\right\rangle\right|^2 \eqdef\left|\left\langle \bm{a}_i,\bm{e}_1\right\rangle\right|^2\label{eq:rip}
\end{align}
where $\bm{U}\in\mathbb{R}^{n\times n}$ is some unitary matrix, and 
$\eqdef$ means that terms on both sides of the equality have the same distribution. 
It is then easily verified that 
 \begin{equation}\label{eq:1stestimate}
 	\frac{1}{m}\sum_{i=1}^my_i=\frac{1}{m}\sum_{i=1}^ma_{i,1}^2\|\bm{x}\|^2\approx \|\bm{x}\|^2,
 	 \end{equation}
 where the last approximation arises from the following concentration result 
$(1/m)\sum_{i=1}^m a_{i,1}^2\approx \mathbb{E}[a_{i,1}^2]=1$ using again the SLLN. Regarding the second estimate, one can rewrite its square as
\begin{equation}\label{eq:2ndestimate}
\frac{n\sum_{i=1}^my_i}{\sum_{i=1}^m\|\bm{a}_i\|^2}=\frac{1}{m}\sum_{i=1}^my_i\cdot\frac{n}{(1/m)\cdot\sum_{i=1}^m\|\bm{a}_i\|^2}.
\end{equation}
It is clear from~\eqref{eq:1stestimate} that the first term on the right hand side of~\eqref{eq:2ndestimate} 
approximates $\|\bm{x}\|^2$. The second term approaches $1$ because the denominator $(1/m)\cdot\sum_{i=1}^m\|\bm{a}_i\|^2\approx n$ appealing to the SLLN again and the fact that $\bm{a}_i\sim\mathcal{N}(\bm{0},\bm{I}_n)$.
For simplicity, we choose to work with the first norm estimate
\begin{equation}\label{eq:oie}
\bm{z}_0=\sqrt{\frac{\sum_{i=1}^my_i}{m}}\tilde{\bm{z}}_0.	
\end{equation}

 It is worth highlighting that, compared to the matrix $\bm{Y}:=\frac{1}{m}\sum_{i\in\mathcal{T}_0}y_i\bm{a}_i\bm{a}_i^\ccalT$ used in spectral methods, our constructed matrix $\widebar{\bm{Y}}_0$ in~\eqref{eq:y0} 
does not depend on the observed data $\{y_i\}$ explicitly; the dependence is only through the choice of the index set $\mathcal{I}_0$. The novel orthogonality-promoting initialization thus enjoys two advantages over its spectral alternatives: a1) it does not suffer from heavy-tails of the fourth-order moments of Gaussian $\{\bm{a}_i\}$ vectors common in spectral initialization schemes; and, a2) it is less sensitive to noisy data. 


\begin{figure}[ht]
	\centering
	\begin{subfigure}
	\centering
	\includegraphics[width=.5\textwidth]{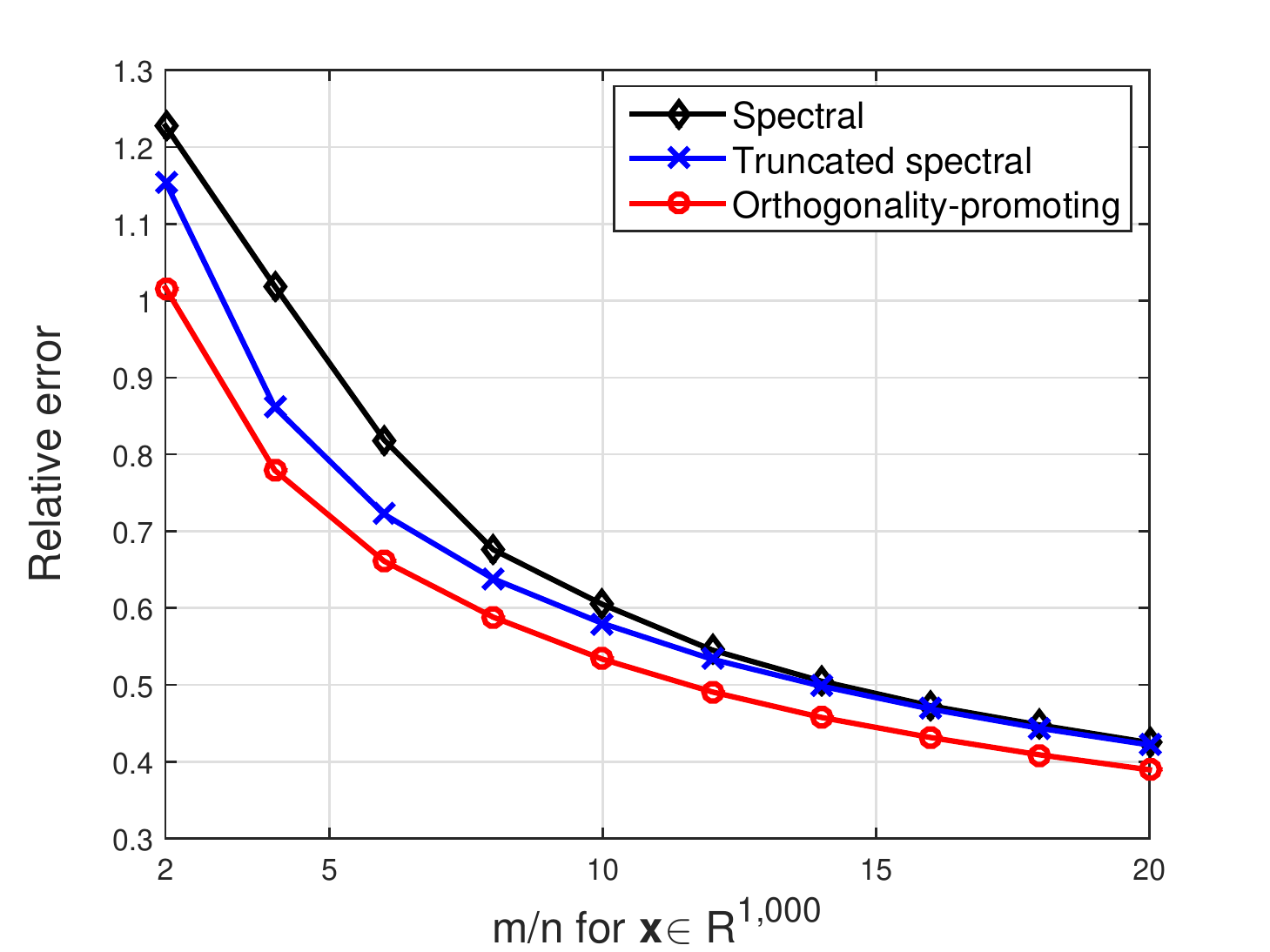}
	\end{subfigure}
	\hspace{-12pt}
	\begin{subfigure}
	\centering
	\includegraphics[width=.5\textwidth]{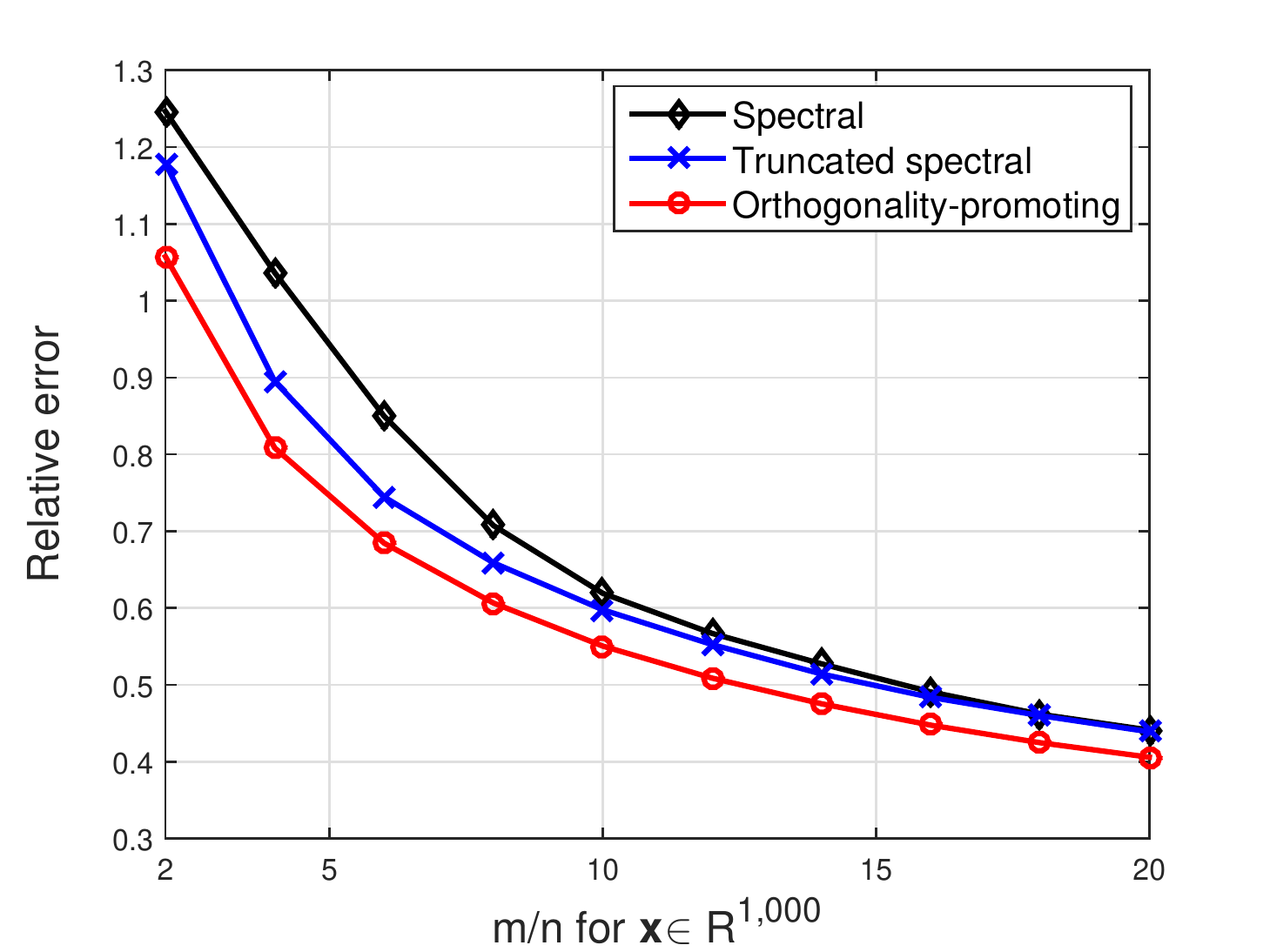} 
	\end{subfigure}
\caption{Relative error of initial estimates versus $m/n$ for: i) the spectral method~\cite{wf}; ii) the truncated spectral method~\cite{twf}; and iii) our orthogonality-promoting method with $n=1,000$, and $m/n$ varying by $2$ from $2$ to $20$. Top: Noiseless real-valued Gaussian model with $\bm{x}\sim\mathcal{N}(\bm{0},\bm{I}_n)$, $\bm{a}_i\sim\mathcal{N}(\bm{0},\bm{I}_n)$, and $\eta_i=0$. Bottom: Noisy real-valued Gaussian model with $\bm{x}\sim\mathcal{N}(\bm{0},\bm{I}_n)$, $\bm{a}_i\sim\mathcal{N}(\bm{0},\bm{I}_n)$, and $\sigma^2=0.2^2\|\bm{x}\|^2$.
}
\label{fig:mdifferent}
	\hspace{-0pt}
\end{figure}


Figure~\ref{fig:mdifferent} compares three different initialization schemes including spectral initialization~\cite{altmin,wf}, truncated spectral initialization~\cite{twf}, and the proposed orthogonality-promoting initialization. The relative error of their returned initial estimates versus the measurement/unknown ratio $m/n$ is depicted under the noiseless and noisy real-valued Gaussian models, where $\bm{x}\in\mathbb{R}^{1,000}$ was randomly generated and $m/n$ increases by $2$ from $2$ to $20$. Clearly,
all schemes enjoy improved performance as $m/n$ increases in both noiseless and noisy settings. The orthogonality-promoting initialization achieves consistently superior performance over its competing spectral alternatives under both noiseless and noisy Gaussian data. Interestingly, the spectral and truncated spectral schemes exhibit similar performance when $m/n$ becomes sufficiently large (e.g., $m/n\ge 14$ in the noiseless setup or $m/n\ge 16$ in the noisy one). This confirms that the truncation helps only if $m/n$ is relatively small. Indeed, the truncation discards measurements of excessively large or small sizes emerging from the heavy tails of the data distribution. Hence, its advantage over the non-truncated spectral initialization diminishes as the number of measurements increases, which gradually straightens out the heavy tails.

\begin{algorithm}[h!]
  \caption{Truncated amplitude flow (TAF)}
  \label{alg:TAF}
  \begin{algorithmic}[1]
\STATE {\bfseries Input:}
Amplitude data $\left\{\psi_i:=\left|\langle\bm{a}_i,\bm{x}\rangle\right|\right\}_{i=1}^m$ and design vectors $\{\bm{a}_i\}_{i=1}^m$; the maximum number of iterations $T$; by default, take constant step sizes $\mu=0.6/1$ for the real-/complex-valued models, truncation thresholds $|\widebar{\mathcal{I}}_0|=\lceil\frac{1}{6}m\rceil$, 
and $\gamma=0.7$. 
\STATE {\bfseries Set} $\widebar{\mathcal{I}}_0$ as the set of indices corresponding to the $|\widebar{\mathcal{I}}_0|$ largest values of $\left\{\psi_i/\|\bm{a}_i\|\right\}$.
\STATE {\bfseries Initialize} $\bm{z}_0$ to $\sqrt{\frac{\sum_{i=1}^m\psi_i^2}{m}}\tilde{\bm{z}}_0$, where $\tilde{\bm{z}}_0$ is 
 the normalized leading eigenvector of $\widebar{\bm{Y}}_0:=\frac{1}{|\widebar{\mathcal{I}}_0|}\sum_{i\in\widebar{\mathcal{I}}_0}\frac{\bm{a}_i\bm{a}_i^\ccalT}{\|\bm{a}_i\|^2}$. 
  \STATE {\bfseries Loop: for} {$t=0$ {\bfseries to} $T-1$}
  \vspace{-.em}
{ $$    	\bm{z}_{t+1}=\bm{z}_t-\frac{\mu}{m}\sum_{i\in\mathcal{I}_{t+1}}\left(\bm{a}_i^\ccalT\bm{z}_t-\psi_i\frac{\bm{a}_i^\ccalT\bm{z}_t}{|\bm{a}_i^\ccalT\bm{z}_t|}\right)\bm{a}_i
\vspace{-.em}
$$
   where $\mathcal{I}_{t+1}:=\left\{1\le i\le m\left|{\left|\bm{a}_i^\ccalT\bm{z}_t\right|}\ge \frac{1}{1+\gamma}{\psi_i}\right.\right\}$. }
     \STATE {\bfseries Output:}
$\bm{z}_{T}$.
  \end{algorithmic}
   \vspace{-.em} 
\end{algorithm} 

\section{Main Results}\label{sec:main}
The TAF algorithm is summarized in Algorithm~\ref{alg:TAF}. Default values are set for pertinent algorithmic parameters.
Assuming independent data samples $\{(\bm{a}_i;\psi_i)\}$ drawn from the noiseless real-valued Gaussian model, the following result establishes the theoretical performance of TAF.


\begin{theorem}[Exact recovery]
	\label{thm:initial}
	Let $\bm{x}\in\mathbb{R}^n$ be an arbitrary signal vector, and consider (noise-free) measurements $\psi_i=|\bm{a}_i^\ccalT\bm{x}|$, in which $\bm{a}_i\widesim{i.i.d.} \mathcal{N}(\bm{0},\bm{I}_n)$, $1\le  i\le m$. Then with probability at least $1-(m+5){\rm e}^{-n/2}-{\rm e}^{-c_0m}-1/n^2$ for some universal constant $c_0>0$, the initialization $\bm{z}_0$ returned by the orthogonality-promoting method in Algorithm~\ref{alg:TAF} 
	satisfies
	\vspace{-.em}
	\begin{equation}\label{eq:i1}
		{\rm dist}(\bm{z}_0,\bm{x})\le \rho\left\|\bm{x}\right\|
		\vspace{-.em}
	\end{equation}
		with $\rho={1}/{10}$ (or any sufficiently small positive constant), provided that $m\ge c_1|\widebar{\mathcal{I}}_0|\ge c_2n$ for some numerical constants $c_1,\,c_2>0$, and sufficiently large $n$. Furthermore, choosing a constant step size $\mu\le \mu_0$ along with a truncation level $ \gamma\ge 1/2$, and starting from any initial guess $\bm{z}_0$ satisfying \eqref{eq:i1}, successive estimates of the TAF solver (tabulated in Algorithm \ref{alg:TAF}) obey
			\vspace{-.em}
	\begin{equation}
		{\rm dist}\left( \bm{z}_t,\bm{x}\right)\le \rho\left(1-\nu\right)^t\left\|\bm{x}\right\|,\quad  t=0,\,1,\,2,\,\ldots
			\vspace{-.em}
	\end{equation} 
	for some $0<\nu<1$, which holds
 with probability exceeding $1-(m+5){\rm e}^{-n/2}-8{\rm e}^{-c_0m}-1/n^2$. 
\end{theorem}

	\footnotetext{The symbol $\lceil\cdot\rceil$ is the ceiling operation returning the smallest integer greater than or equal to the given number.} 
 	
 Typical parameter values for TAF in Algorithm \ref{alg:TAF} are $\mu=0.6$, and $\gamma=0.7$. The proof of Theorem~\ref{thm:initial} is relegated to Section~\ref{sec:proof}. 
Theorem~\ref{thm:initial} asserts that: i) TAF reconstructs the solution $\bm{x}$ exactly as soon as the number of equations is about the number of unknowns, which is theoretically order optimal. Our numerical tests demonstrate that for the real-valued Gaussian model, TAF achieves a success rate of $100\%$ when $m/n$ is as small as $3$, which is slightly larger than the information limit of $m/n=2$ (Recall that $m\ge 2n-1$ is necessary for the uniqueness.)
This is a significant reduction in the sample complexity ratio, which 
is $5$ for TWF and $7$ for WF. 
Surprisingly, TAF also enjoys a success rate of over $50\%$ when $m/n$ is the information limit $2$, which has not yet been presented for any existing algorithms. See further discussion in Section~\ref{sec:test};  
and, ii) TAF converges exponentially fast with convergence rate independent of the dimension $n$. Specifically,
TAF requires at most $\mathcal{O}(\log (1/\epsilon))$ iterations to achieve any given solution accuracy $\epsilon>0$ (a.k.a., ${\rm dist}(\bm{z}_t,\bm{x})\le \epsilon\left\|\bm{x}\right\|$), 
with iteration cost $\mathcal{O}(mn)$. Since the truncation 
takes time on the order of $\mathcal{O}(m)$, the
 computational burden of TAF per iteration 
is dominated by the evaluation of the gradient components. The latter involves two matrix-vector multiplications that are computable in $\mathcal{O}(mn)$ flops, namely, $\bm{A}\bm{z}_t$ yields $\bm{u}_t$, and $\bm{A}^\ccalT\bm{v}_t$ the  gradient, where $\bm{v}_t:=\bm{u}_t-\bm{\psi}\odot\tfrac{\bm{u}_t}{|\bm{u}_t|}$. Hence, the total running time of TAF is $\mathcal{O}(mn\log(1/\epsilon))$, which is proportional to the time taken to read the data $\mathcal{O}(mn)$.

In the noisy setting, TAF is stable under additive noise. To be more specific, consider the amplitude-based data model 
$
\psi_i=|\bm{a}_i^\ccalT\bm{x}|+\eta_i
$.
 It can be shown that the truncated amplitude flow estimates in Algorithm \ref{alg:TAF} satisfy
\begin{equation}
		{\rm dist}\left( \bm{z}_t,\bm{x}\right)\lesssim \left(1-\nu\right)^t\left\|\bm{x}\right\|+\frac{1}{\sqrt{m}}\left\|\bm{\eta}\right\|,\quad  t=0,\,1,\,\ldots
\end{equation}
with high probability for all $\bm{x}\in\mathbb{R}^n$, provided that $m\ge c_1|\widebar{\mathcal{I}}_0|\ge c_2n$ for sufficiently large $n$ and the noise is bounded $\left\|\bm{\eta}\right\|_{\infty}\le c_3\left\|\bm{x}\right\|$ with $\bm{\eta}:=[\eta_1~\cdots~\eta_n]^\ccalT$, where $0<\nu<1$, and $c_1,\,c_2,\,c_3>0$ are some universal constants. 
	The proof can be directly adapted from those of Theorem~\ref{thm:initial} above and Theorem 2 in \cite{twf}.

\section{Simulated Tests}\label{sec:test}
In this section, we provide additional numerical tests evaluating performance of the proposed scheme relative to (T)WF \footnote{Matlab codes directly downloaded from the authors' websites: \url{http://statweb.stanford.edu/~candes/TWF/algorithm.html}; \url{http://www-bcf.usc.edu/~soltanol/WFcode.html}.}  and AF.
The initial estimate was found based on $50$ power iterations, and was subsequently refined by $T=1,000$ gradient-type iterations in each scheme. The Matlab implementations of TAF are available at \url{https://gangumn.github.io/TAF/} for reproducibility.

\begin{figure}[t!]
	\centering
	\begin{subfigure}
		\centering
		\includegraphics[width=.5\textwidth]{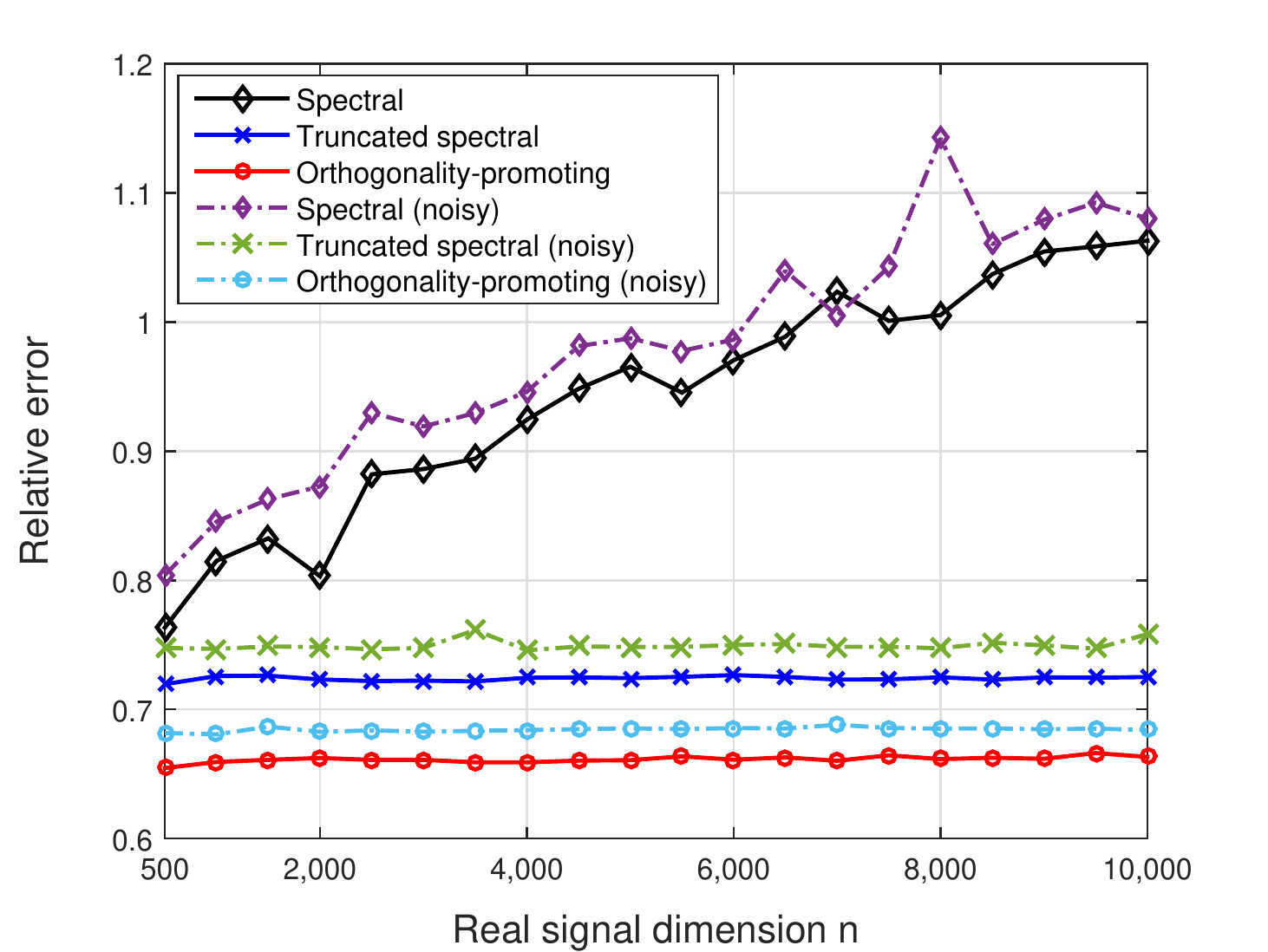}
	\end{subfigure}
	\hspace{-12pt}
	\begin{subfigure}
		\centering
		\includegraphics[width=.5\textwidth]{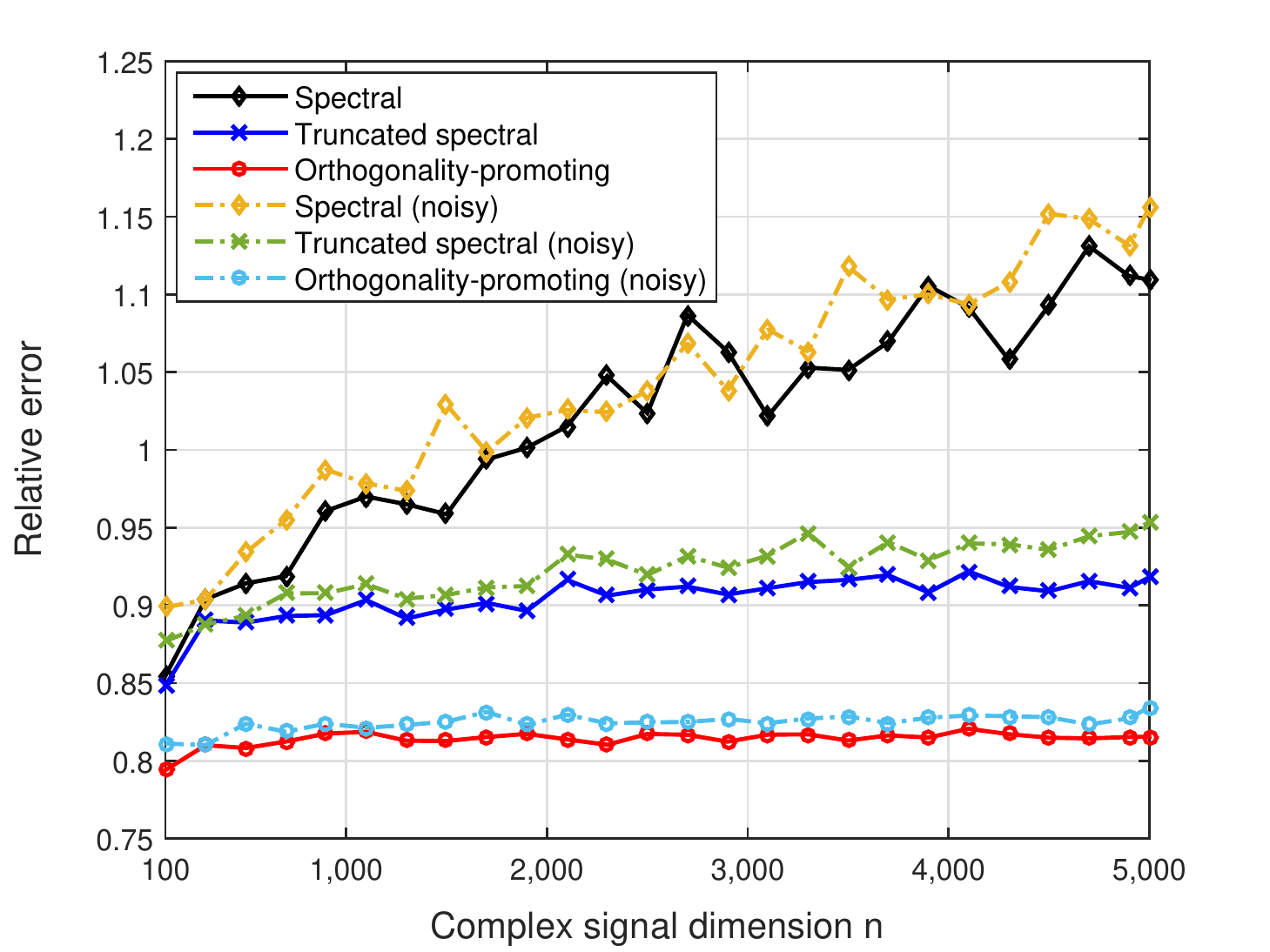} 
	\end{subfigure}
	\caption{
		Average relative error of estimates obtained from $100$ MC trials 
		using: i) the spectral method~\cite{altmin,wf}; ii) the truncated spectral method~\cite{twf}; and iii) the proposed orthogonality-promoting method on noise-free (solid lines) and noisy (dotted lines) instances with $m/n=6$, and $n$ varying from $500/100$ to $10,000/5,000$ for real-/complex-valued vectors. Left: Real-valued Gaussian model with $\bm{x}\sim\mathcal{N}(\bm{0},\bm{I}_n)$, $\bm{a}_i\sim\mathcal{N}(\bm{0},\bm{I}_n)$, and $\sigma^2=0.2^2\left\|\bm{x}\right\|^2$. Right: Complex-valued Gaussian model with $\bm{x}\sim \mathcal{CN}(\bm{0},\bm{I}_n)$, $\bm{a}_i\sim\mathcal{CN}(\bm{0},\bm{I}_n)$, and $\sigma^2=0.2^2\left\|\bm{x}\right\|^2$. 
	}
	\label{fig:ireal}
	\hspace{-0pt}
\end{figure}

\begin{figure}[ht]
\centering
   \includegraphics[scale=0.58]{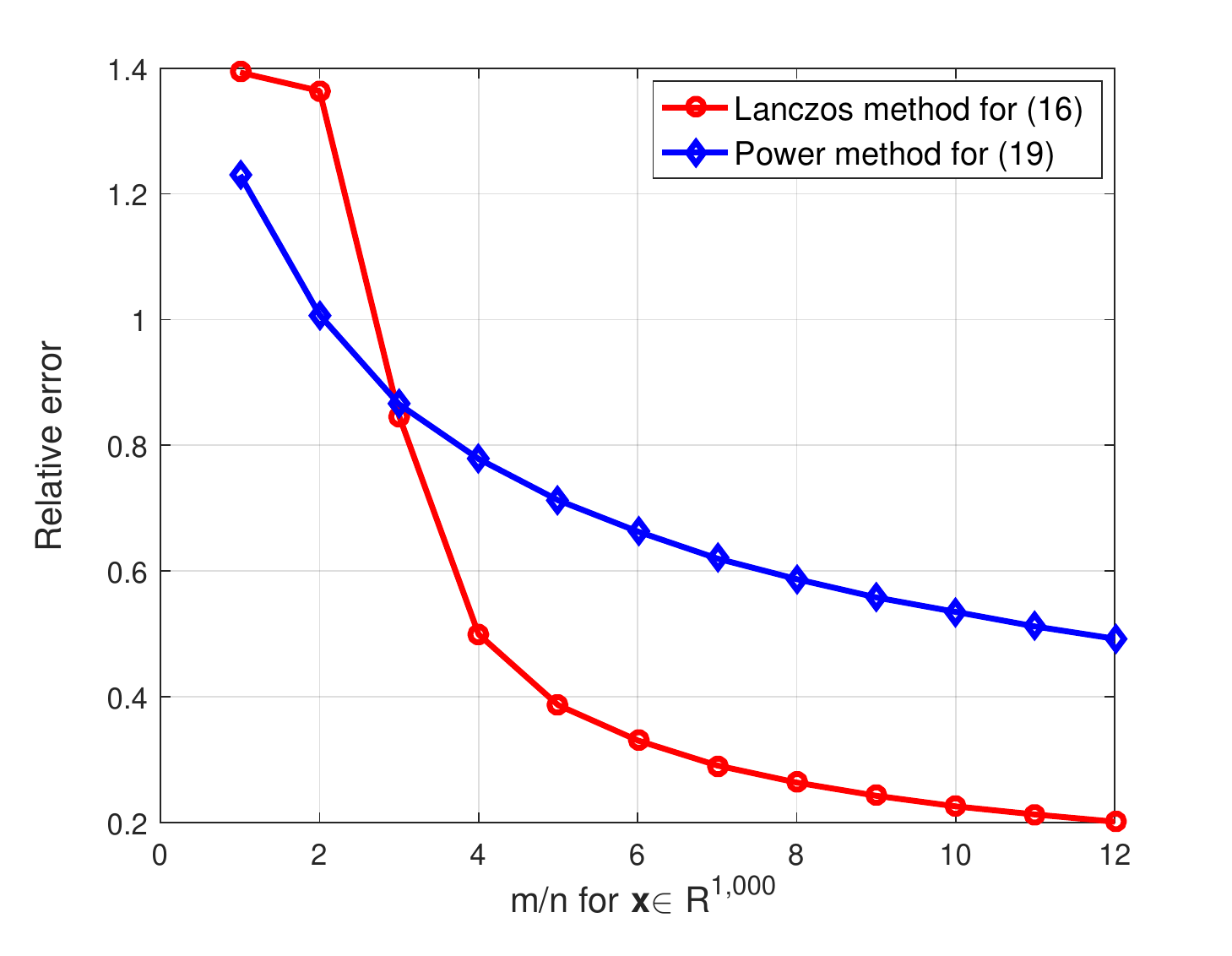}
    \vspace{-0pt}
\caption{Relative initialization error of the initialization solving the minimum eigenvalue problem in \eqref{eq:mineig} via the Lanczos method and by solving the maximum eigenvalue problem in \eqref{eq:maxeig}.}
\label{fig:minmax}
\end{figure}

Left panel in Fig.~\ref{fig:ireal} presents the average relative error of three initialization methods on a series of noiseless/noisy real-valued Gaussian problems with $m/n=6$ fixed, and $n$ varying from $500$ to $10^4$, while those for the corresponding complex-valued Gaussian instances are shown in the right panel. Clearly, the proposed initialization method returns more accurate and robust estimates than the spectral ones. Under the same condition for the real-valued Gaussian model, Fig. \ref{fig:minmax}
compares the initialization implemented in Algorithm \ref{alg:TAF} obtained by solving the maximum eigenvalue problem in \eqref{eq:maxeig} with
the one obtained by tackling the minimum eigenvalue problem in \eqref{eq:mineig} via the Lanczos method \cite{saad1}. When the number of equations is relatively small (less than about $3n$), the former performs better than the latter. Interestingly though, the latter works remarkably well and almost halves the error incurred by the implemented initialization of Algorithm \ref{alg:TAF} as soon as the number of equations becomes larger than $4$. 

\begin{figure}[ht]
	\centering
	\includegraphics[scale=0.58]{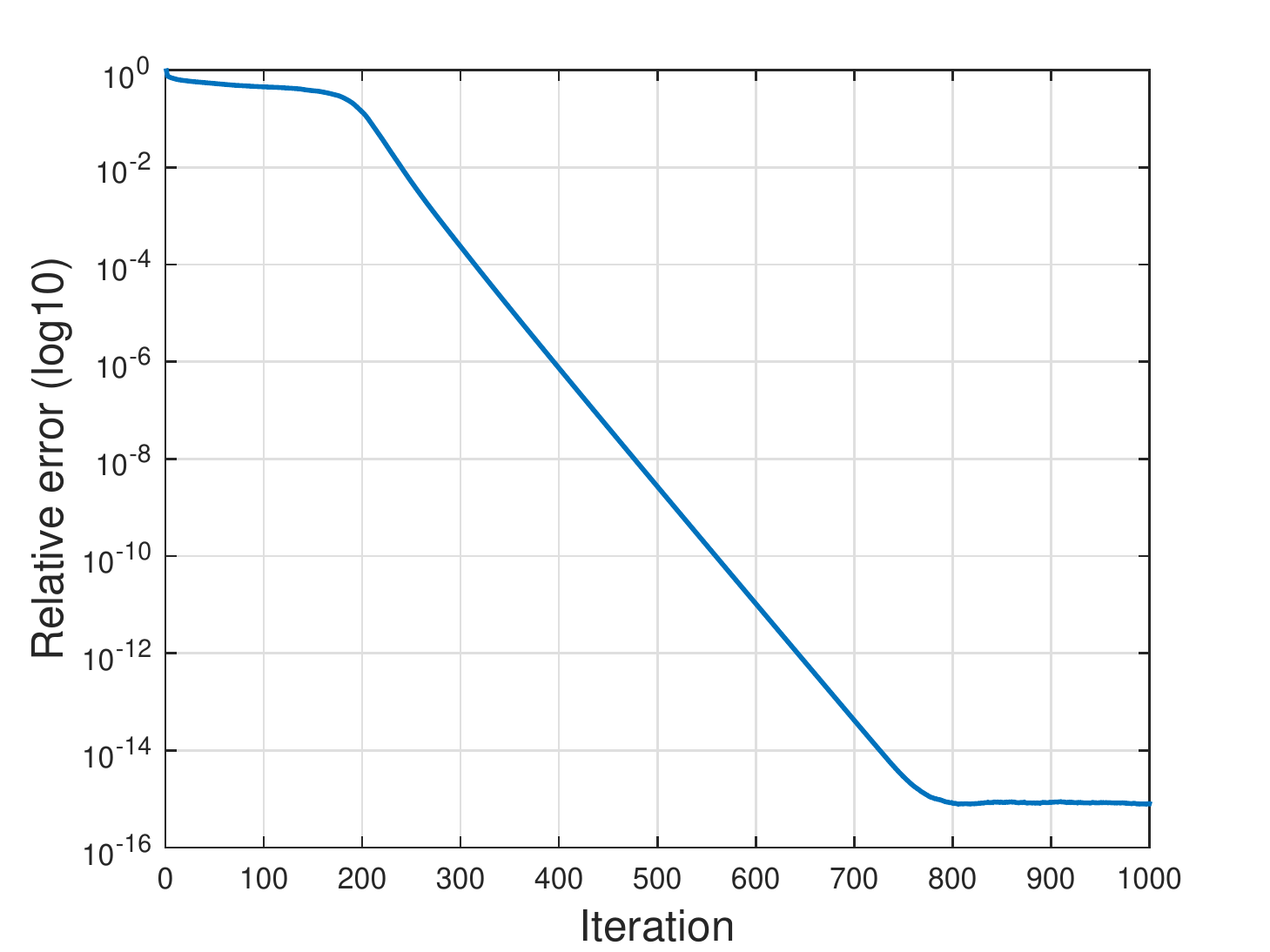}
	\vspace{-0pt}
	\caption{Relative error versus iteration for TAF for a noiseless real-valued Gaussian model under the information-limit of $m=2n-1$.}
	\label{fig:instance}
\end{figure}

\begin{figure}[ht]
	\centering
	\begin{subfigure}
	\centering
	\includegraphics[width=.5\textwidth]{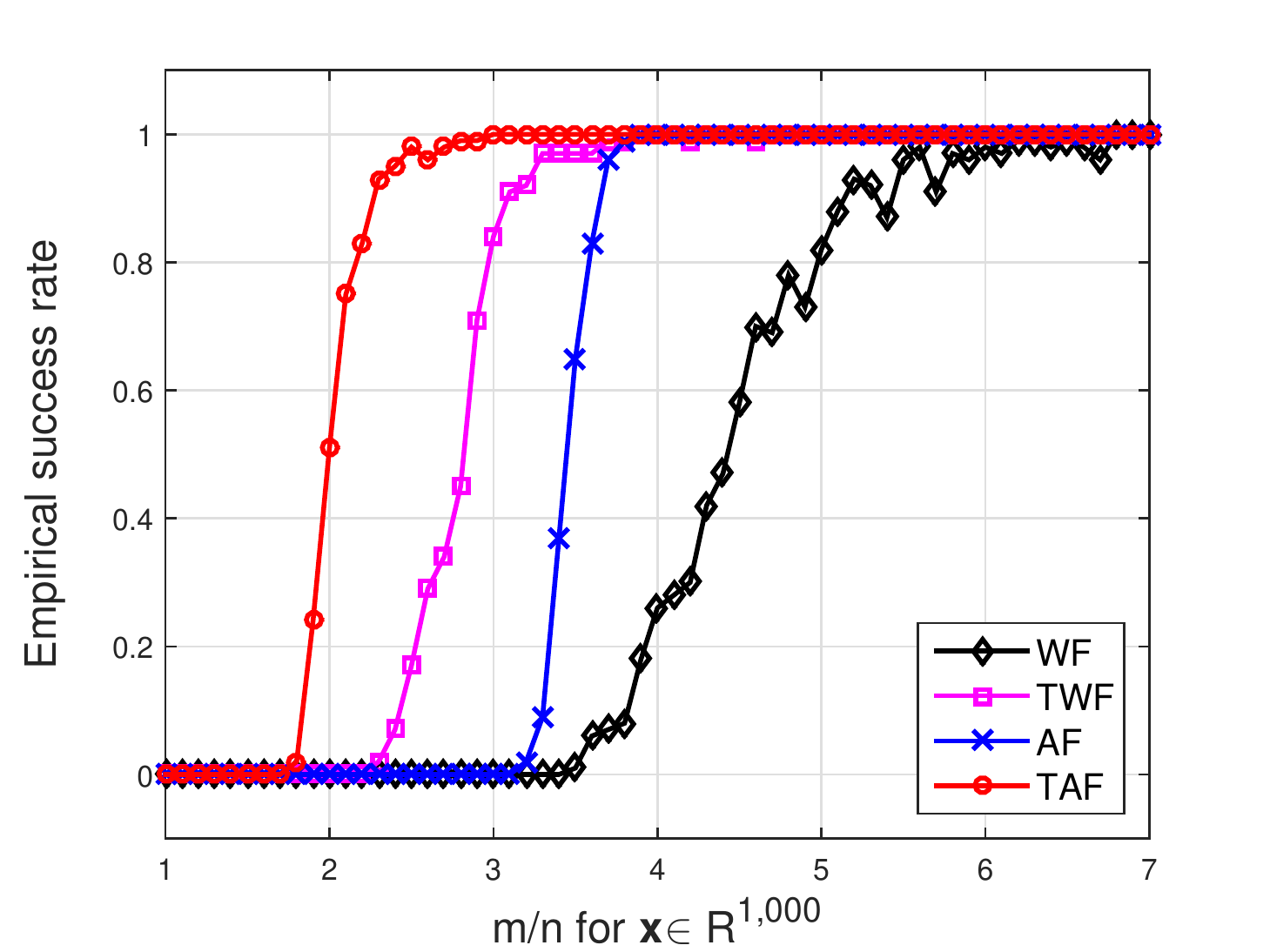}
	\end{subfigure}
	\hspace{-12pt}	
	\begin{subfigure}
	\centering
	\includegraphics[width=.5\textwidth]{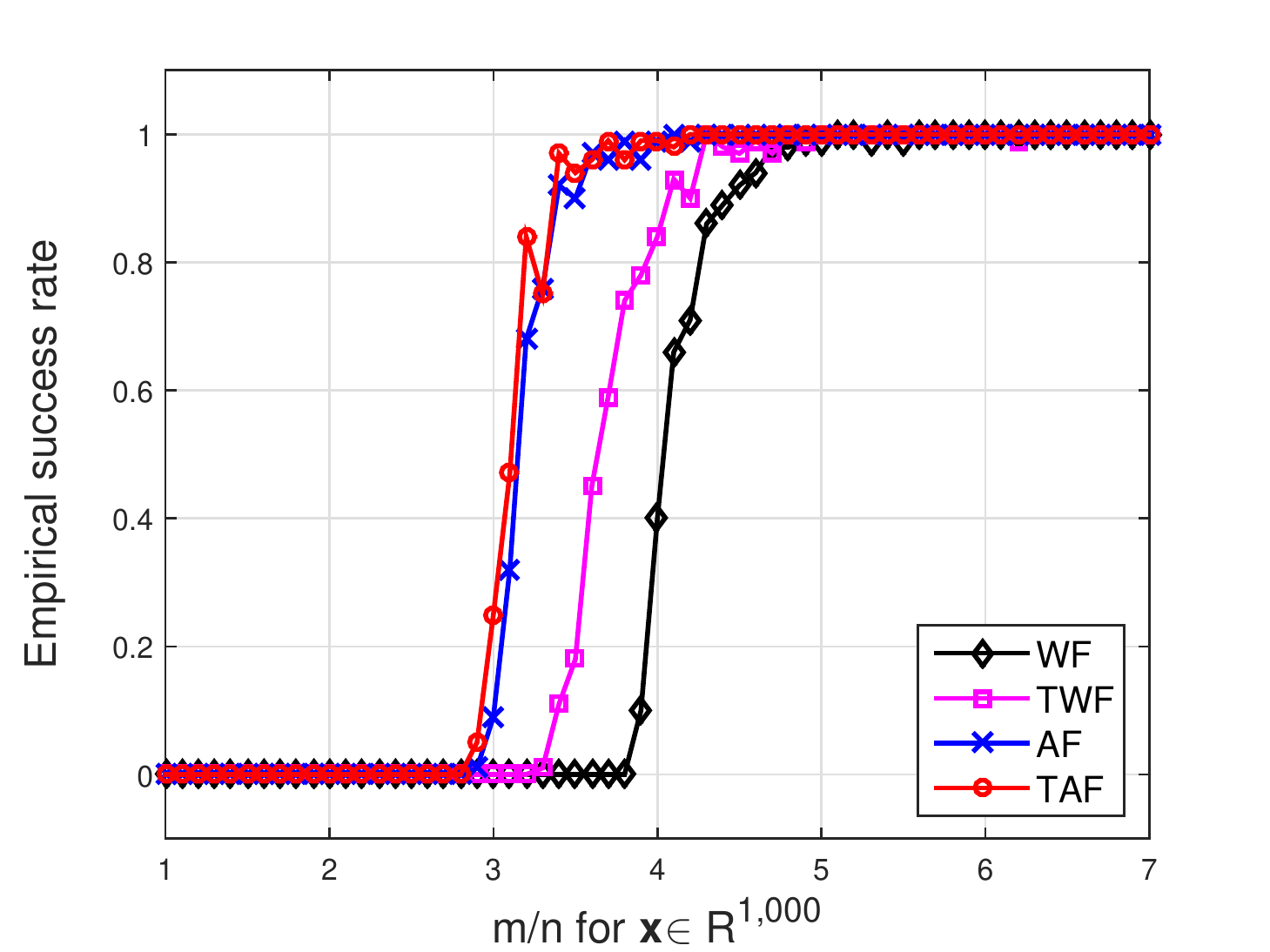} 
	\end{subfigure}
\caption{Empirical success rate for WF, TWF, AF, and TAF with $n=1,000$ and $m/n$ varying by $0.1$ from $1$ to $7$. Left: Noiseless real-valued Gaussian model with 
$\bm{x}\sim\mathcal{N}(\bm{0},\bm{I}_n)$ and $\bm{a}_i\sim\mathcal{N}(\bm{0},\bm{I}_n)$; Right: Noiseless complex-valued Gaussian model with 
 $\bm{x}\sim\mathcal{CN}(\bm{0},\bm{I}_n)$ and $\bm{a}_i\sim\mathcal{CN}(\bm{0},\bm{I}_n)$.}
\label{fig:srate}
\vspace{-0pt}
\end{figure}

To demonstrate the power of TAF, Fig.~\ref{fig:instance} plots the relative error of recovering a real-valued signal in logarithmic scale versus the iteration count under the information-limit of $m=2n-1$ noiseless i.i.d. Gaussian measurements~\cite{2006balan}. In this case, since the returned initial estimate is relatively far from the optimal solution (see Fig.~\ref{fig:mdifferent}), 
TAF converges slowly for the first $200$ iterations or so due to elimination of a significant amount of `bad' generalized gradient components (corrupted by mistakenly estimated signs). As the iterate gets more accurate and lands within a small-size neighborhood of $\bm{x}$, TAF converges exponentially fast to the globally optimal solution. It is worth emphasizing that no existing method succeeds in this case. 
Figure~\ref{fig:srate} compares the empirical success rate of three schemes under both real-valued and complex-valued Gaussian models
with $n=10^3$ and 
 $m/n$ varying by $0.1$ from $1$ to $7$, where a success is claimed if the estimate has a relative error less than $10^{-5}$. 
For real-valued vectors, 
  TAF achieves a success rate of over $50\%$ when $m/n=2$, and guarantees perfect recovery from about $3n$ measurements; while for complex-valued ones, TAF enjoys a success rate of $95\%$ when $m/n=3.4$, and ensures perfect recovery from about $4.5n$ measurements.

  To demonstrate the stability of TAF, the relative mean-squared error (MSE) $${\rm Relative~MSE}:=\frac{{\rm dist}^2(\bm{z}_T,\bm{x})}{\|\bm{x}\|^2}$$
  as a function of the signal-to-noise ratio (SNR) is plotted for different $m/n$ values. We consider the noisy model $\psi_i=|\langle\bm{a}_i,\bm{x}\rangle|+\eta_i$ with $\bm{x}\sim\mathcal{N}(\bm{0},\bm{I}_{1,000})$ and real-valued independent Gaussian sensing vectors $\bm{a}_i\sim\mathcal{N}(\bm{0},\bm{I}_{1,000})$, in which $m/n$ takes values $\{6,\,8,\,10\}$, and the SNR in dB, given by $${\rm SNR}:=10\log_{10} \frac{\sum_{i=1}^m|\langle\bm{a}_i,\bm{x}\rangle|^2}{\sum_{i=1}^m\eta_i^2}
  $$
  is varied from $10$ dB to $50$ dB. Averaging over $100$ independent trials, Fig.~\ref{fig:db} demonstrates that the relative MSE for all $m/n$ values scales inversely proportional to SNR, hence justifying the stability of TAF under bounded additive noise. 
  
  \begin{figure}[ht]
  	\centering
  	\includegraphics[scale=0.6]{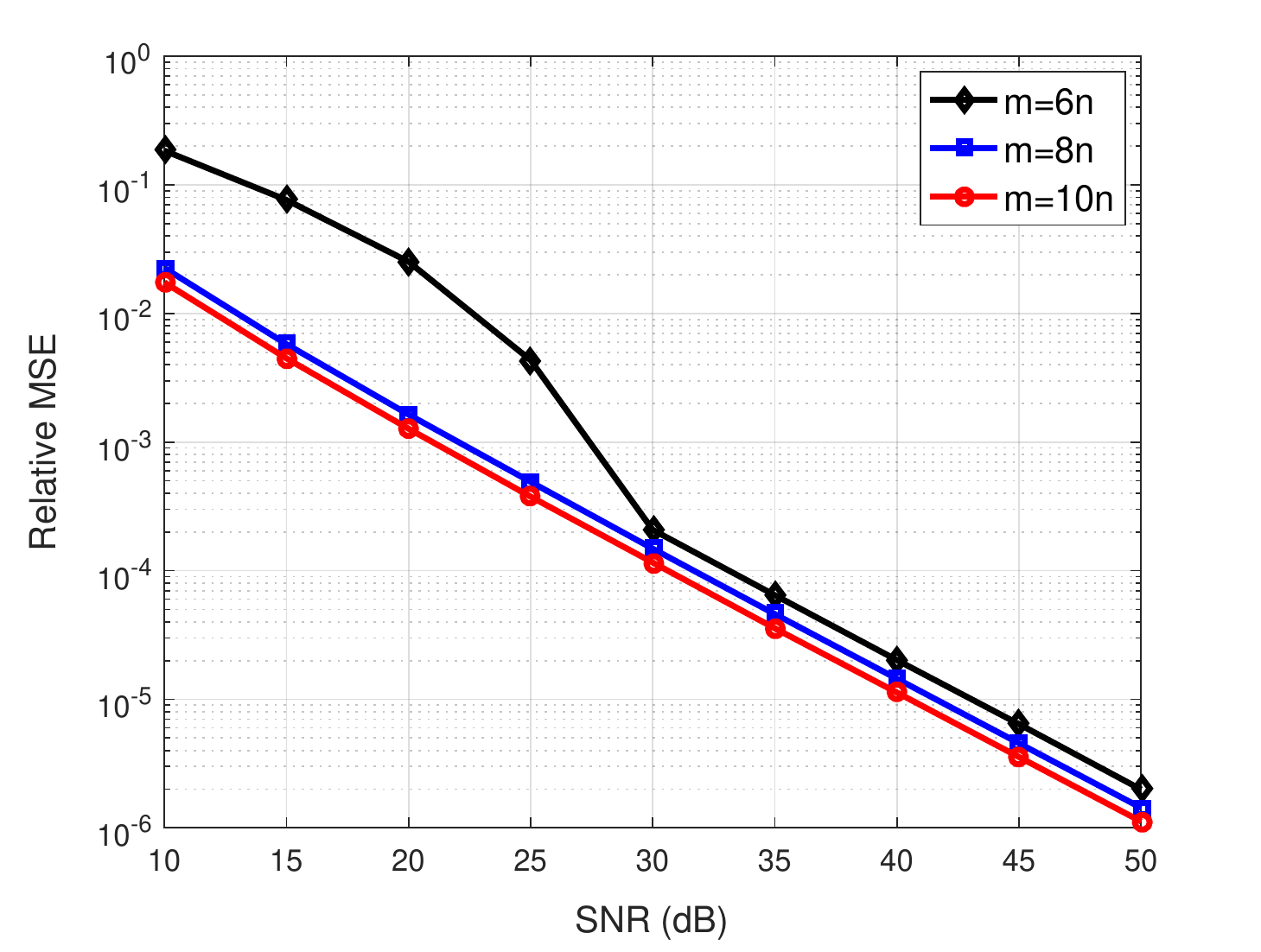}
  	\caption{Relative MSE versus SNR for TAF when $\psi_i$'s follow the amplitude-based noisy data model.}
  	\label{fig:db}
  \end{figure}

The next experiment evaluates the efficacy of the proposed initialization method, simulating all schemes initialized by the truncated spectral initial estimate~\cite{twf} and the orthogonality-promoting initial estimate. Apparently, all algorithms except WF admit a significant performance improvement when initialized by the proposed orthogonality-promoting initialization relative to the truncated spectral initialization. Nevertheless, TAF with our developed orthogonality-promoting initialization enjoys superior performance over all simulated approaches.

\begin{figure}[!ht]
\centering
   \includegraphics[scale = 0.58]
   {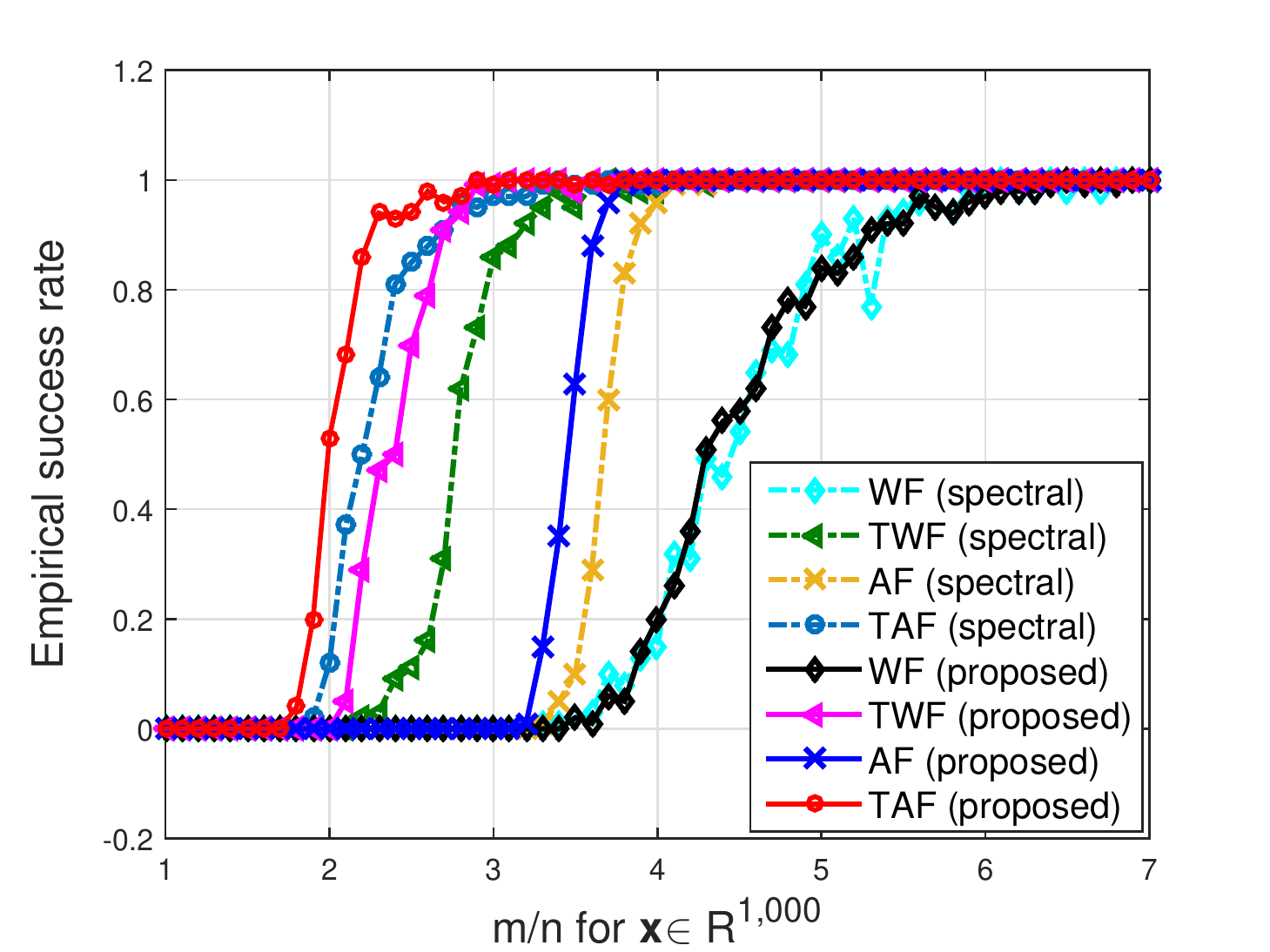}
    \vspace{-0pt}
\caption{Empirical success rate for WF, TWF, AF, and TAF initialized by the truncated spectral and the orthogonality-promoting initializations 
with $n=1,000$ and $m/n$ varying by $0.1$ from $1$ to $7$.}
\label{fig:optsi}
\end{figure}

	\begin{figure*}[ht]
	\centering
\includegraphics[height=.2\textheight, width=1\textwidth]{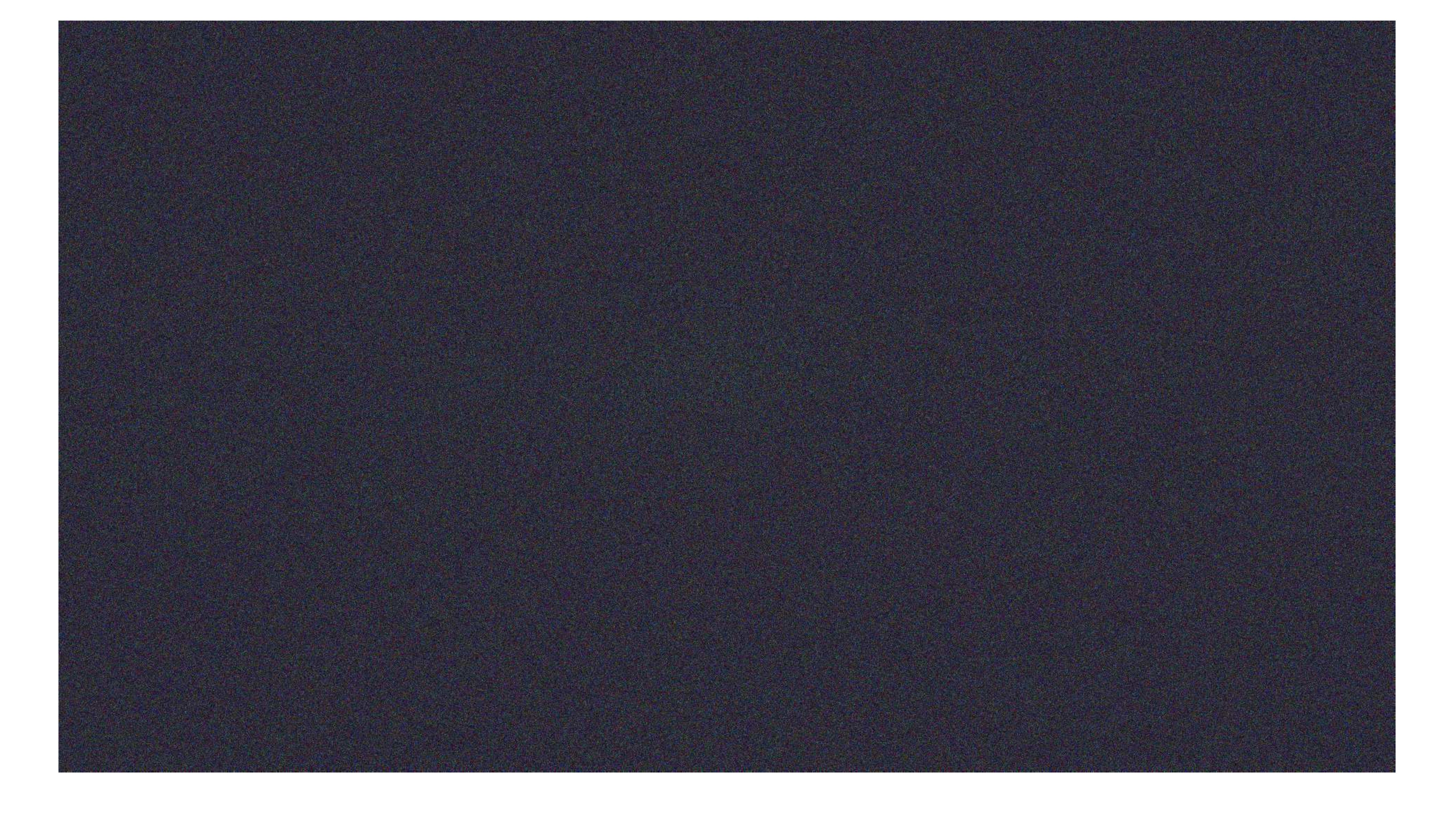}
	\includegraphics[height=.2\textheight, width=1\textwidth]{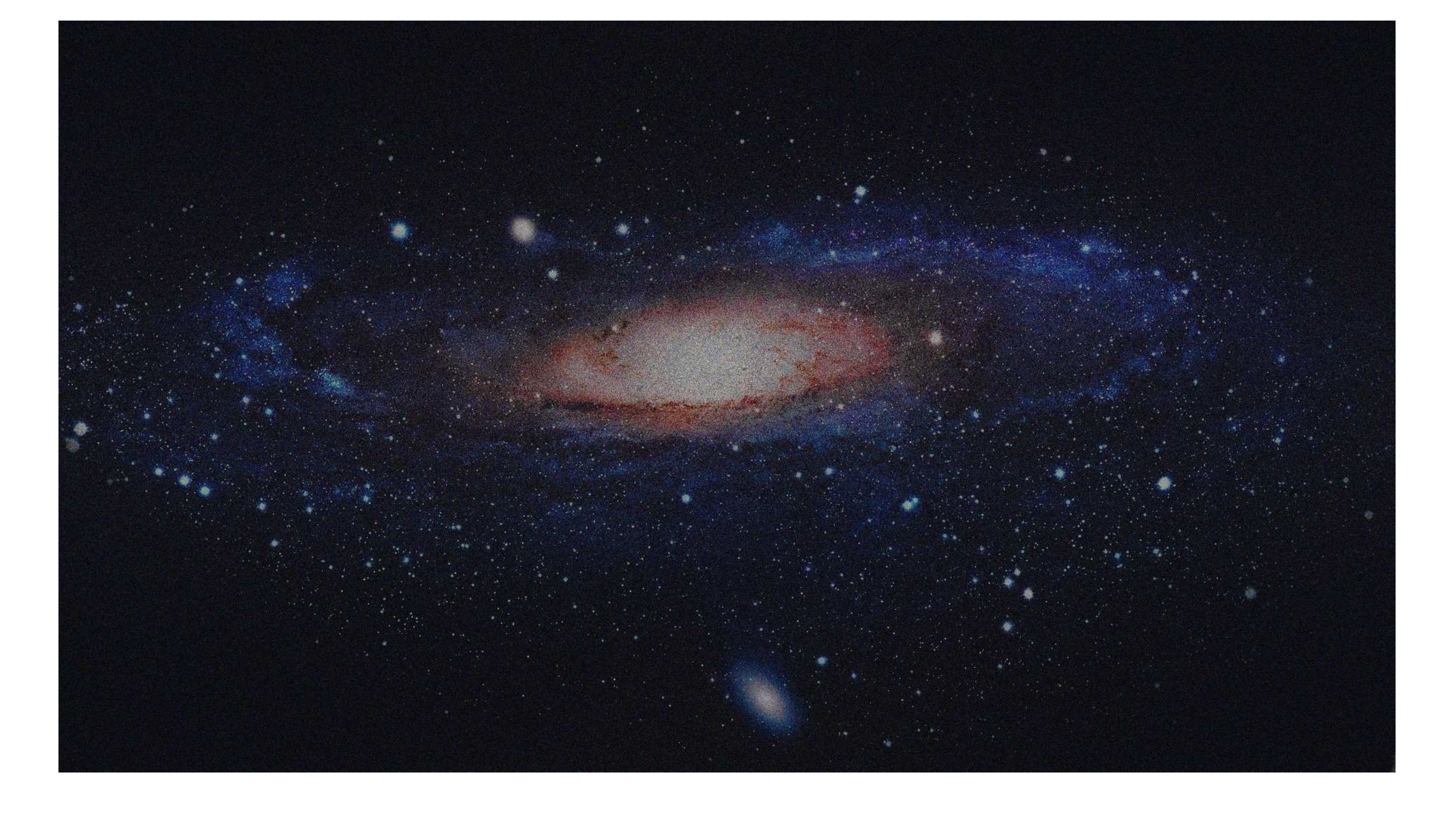}
	\includegraphics[height=.2\textheight, width=1\textwidth]{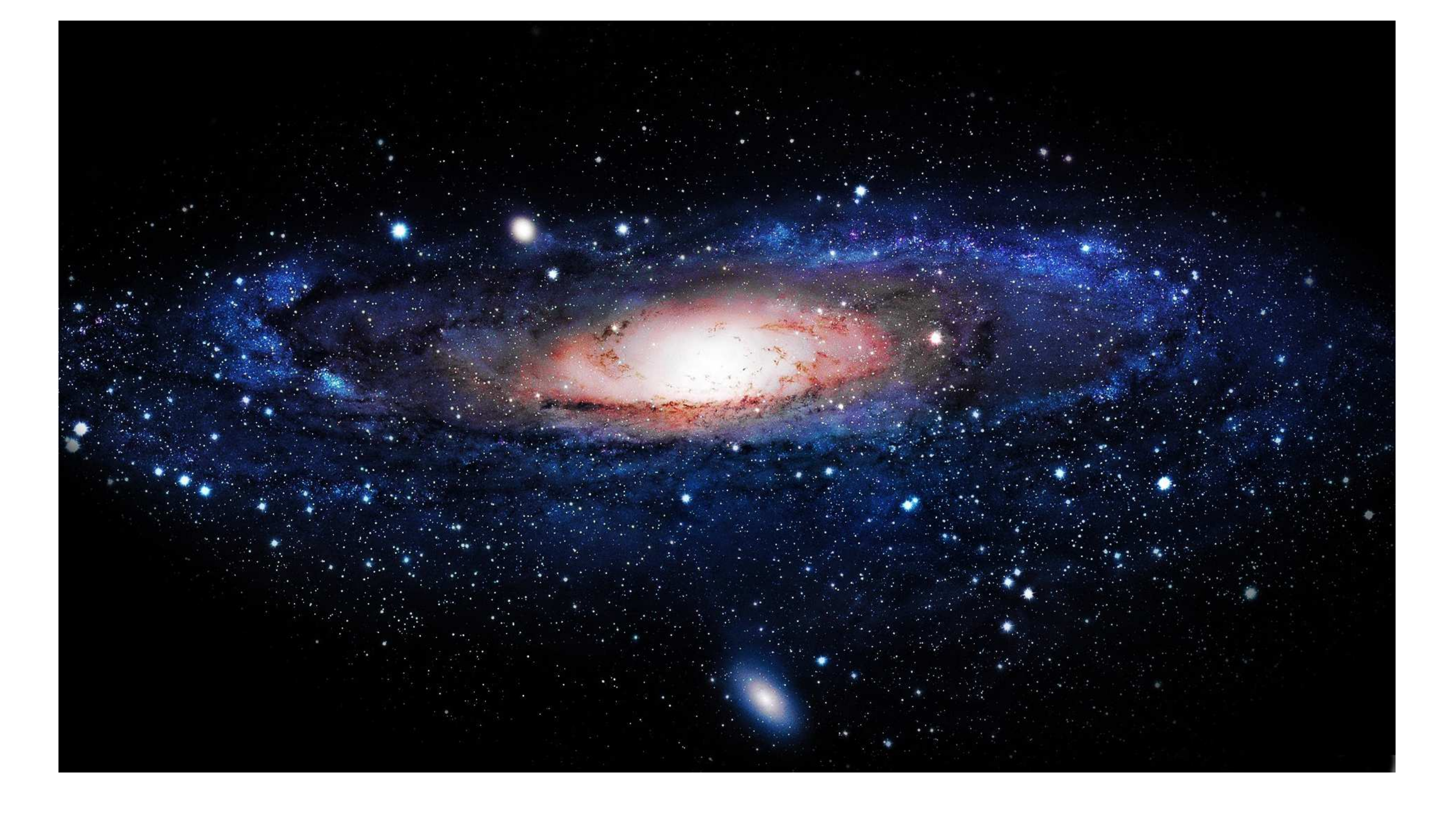}
\caption{
The recovered Milky Way Galaxy images after i) truncated spectral initialization (top); ii) orthogonality-promoting initialization (middle); and iii) $100$ TAF gradient iterations refining the orthogonality-promoting initialization (bottom).
}
\label{fig:milky}
	\vspace{-0pt}
\end{figure*}

Finally, to examine the effectiveness and scalability of TAF in real-world conditions, we simulate recovery of the Milky Way Galaxy image \footnotemark  \footnotetext{Downloaded from \url{http://pics-about-space.com/milky-way-galaxy}.} $\bm{X}\in\mathbb{R}^{1080\times 1920\times 3}$ shown in Fig.~\ref{fig:milky}. The first two indices encode the pixel locations, and the third the RGB (red, green, blue) color bands. 
Consider the coded diffraction pattern (CDP) measurements with random masks~\cite{coded,wf,twf}.
Letting $\bm{x}\in\mathbb{R}^n$ be a vectorization of a certain band of $\bm{X}$ and postulating a number $K$ of random masks, one can further write
\begin{equation}
	\bm{\psi}^{(k)}=\big|\bm{F}\bm{D}^{(k)}\bm{x}\big|,\quad 1\le k\le K,
\end{equation}
where $\bm{F}$ denotes the $n\times n$ discrete Fourier transform matrix, and $\bm{D}^{(k)}$ is a diagonal matrix holding entries sampled uniformly at random from $\{1,\,-1,\,j,\,-j\}$ (phase delays) on its diagonal, with $j$ 
denoting the imaginary unit. Each $\bm{D}^{(k)}$ represents a random mask placed after the object~\cite{coded}. 
With $K=6$ masks implemented in our experiment, the total number of quadratic measurements is $m=6n$. Every algorithm was run independently on each of the three bands. A number $100$ of power iterations were used to obtain an initialization, which was refined by $100$ gradient-type iterations.   
The relative errors after our orthogonality-promoting initialization and after $100$ TAF iterations are $0.6807$ and $9.8631\times 10^{-5}$, respectively, and the recovered images are displayed in Fig.~\ref{fig:milky}. In sharp contrast, TWF returns images of corresponding relative errors $1.3801$ and $1.3409$, which are far away from the ground truth. 

 Regarding running times in all performed experiments, TAF converges slightly faster than TWF, while both are markedly faster than WF. 
All experiments were implemented using MATLAB on an Intel CPU @ $3.4$ GHz ($32$ GB RAM) computer. 

%

\section{Proofs}\label{sec:proof}


This section presents the main ideas behind the proof of Theorem~\ref{thm:initial}, and establishes a few necessary lemmas. Technical details are deferred to the Appendix. 
Relative to WF and TWF, our objective function involves nonsmoothness and nonconvexity, rendering the proof of exact recovery of TAF nontrivial. In addition, our initialization method starts from a rather different perspective than spectral alternatives, so that the tools involved in proving performance of our initialization deviate from those of spectral methods~\cite{altmin,wf,twf}. 
Part of our proof is adapted from \cite{wf,twf} and \cite{2015chen1}.

The proof of Theorem~\ref{thm:initial} consists of two parts: Section~\ref{subsec:initial} justifies the performance of the proposed orthogonality-promoting initialization, which essentially achieves any given constant relative error as soon as the number of equations is on the order of the number of unknowns, namely, $m \asymp n$.\footnote{The notations $\phi(n)=\mathcal{O}(g(n))$ or $\phi(n) \gtrsim g(n)$ (respectively, $\phi(n)\lesssim g(n)$) means there exists a numerical constant $c>0$ such that $\phi(n)\le c g(n)$, while $\phi(n)\asymp g(n)$ means $\phi(n)$ and $g(n)$ are orderwise equivalent.} Section~\ref{subsec:er} demonstrates theoretical convergence of TAF to the solution of the quadratic system in~\eqref{eq:quad} at a geometric rate provided that the initial estimate has a sufficiently small constant relative error as in~\eqref{eq:i1}. The two stages of TAF can be performed independently, meaning that better initialization methods, if available, could be adopted to initialize our truncated generalized gradient iterations; likewise, our initialization may be applied to other iterative optimization algorithms.


\subsection{Constant relative error by orthogonality-promoting initialization}\label{subsec:initial}

This section concentrates on proving guaranteed performance of the proposed orthogonality-promoting initialization method, as asserted in the following proposition.
An alternative approach may be found in \cite{duchi2017}. 

\begin{proposition}\label{prop:initial}
	Fix $\bm{x}\in\mathbb{R}^n$ arbitrarily, and consider the noiseless case $\psi_i=|\bm{a}_i^\ccalT\bm{x}|$, where $\bm{a}_i\widesim{i.i.d.} \mathcal{N}(\bm{0},\,\bm{I}_n)$, $1\le  i\le m$. Then with probability at least $1-(m+5){\rm e}^{-n/2}-{\rm e}^{-c_0 m}-1/n^2$ for some universal constant $c_0>0$, the initialization $\bm{z}_0$ returned by the orthogonality-promoting method satisfies
	\vspace{-.em}
	\begin{equation}\label{eq:i1x}
		{\rm dist}(\bm{z}_0,\,\bm{x})\le \rho\left\|\bm{x}\right\|
		\vspace{-.em}
	\end{equation}
		for $\rho=1/10$ or any positive constant,  with the proviso that $m\ge c_1|\widebar{\mathcal{I}}_0|\ge c_2n$ for some numerical constants $c_1,\,c_2>0$ and sufficiently large $n$.
\end{proposition}

Due to homogeneity in~\eqref{eq:i1x}, it suffices to consider the case $\|\bm{x}\|=1$. 
Assume for the moment that $\left\|\bm{x}\right\|=1$ is known and 
$\bm{z}_0$ has been scaled such that $\left\|\bm{z}_0\right\|=1$ in~\eqref{eq:oie}. 
The error between the employed $\bm{x}$'s norm estimate $\sqrt{\frac{1}{m}\sum_{i=1}^m y_i}$  
and the unknown norm $\left\|\bm{x}\right\|=1$ will be accounted for at the end of this section. 
Instrumental in proving Proposition~\ref{prop:initial} is the following result, whose proof is provided in Appendix~\ref{sec:lembeta}. 

\begin{lemma}\label{lem:beta}
Consider the noiseless data $\psi_i=|\bm{a}_i^\ccalT\bm{x}|$, where $\bm{a}_i\widesim{i.i.d.} \mathcal{N}(\bm{0},\,\bm{I}_n)$, $1\le  i\le m$. For any unit vector $\bm{x}\in\mathbb{R}^n$, there exists a vector $\bm{u}\in\mathbb{R}^{n}$ with $\bm{u}^\ccalT\bm{x}=0$ and $\|\bm{u}\|=1$ such that
	\begin{equation}\label{eq:mse}
	\frac{1}{2}\left\|\bm{x}\bm{x}^\ccalT-\bm{z}_0\bm{z}_0^\ccalT\right\|^2_F\le \frac{\big\|\widebar{\bm{S}}_0\bm{u}\big\|^2}{\big\|\widebar{\bm{S}}_0\bm{x}\big\|^2}
	\end{equation}
	for $\bm z_0=\tilde{\bm{z}}_0$, where the unit vector $\tilde{\bm{z}}_0$ is given in~\eqref{eq:maxeig}, and 
	$\widebar{\bm{S}}_0$ is formed by removing the rows of $\bm{S}:=\big[\bm{a}_1/\left\|\bm{a}_1\right\|~\cdots~
	\bm{a}_m/\left\|\bm{a}_m\right\|\big]^\ccalT\in\mathbb{R}^{m\times n}$ if their indices do not belong to the set $\widebar{\mathcal{I}}_0$ specified in Algorithm~\ref{alg:TAF}.
\end{lemma}

We now turn to prove Proposition~\ref{prop:initial}. The first step consists in upper-bounding the term on the right-hand-side of~\eqref{eq:mse}. Specifically, its numerator is upper bounded, and the denominator lower bounded, as summarized in Lemma~\ref{lem:up} and Lemma~\ref{lem:low} next; their proofs are provided in Appendix~\ref{sec:proofup} and Appendix \ref{sec:prooflow}, respectively.

\begin{lemma}\label{lem:up}
In the setup of Lemma~\ref{lem:beta}, if $|\widebar{\mathcal{I}}_0|\ge c_1' n$, then
\begin{equation}\label{eq:up0}
	\big\|\widebar{\bm{S}}_0\bm{u}\big\|^2\le{1.01|\widebar{\mathcal{I}}_0|}/{n}
\end{equation}
holds with probability at least $1-2{\rm e}^{-c_K n}$, where $c_2'$ and $c_K$ are some universal constants.
\end{lemma}

\begin{lemma}\label{lem:low}
In the setup of Lemma~\ref{lem:beta}, the following holds with probability at least $1-(m+1){\rm e}^{-n/2}-{\rm e}^{-c_0m}-1/n^2$, 
\begin{equation}\label{eq:low0}
\big\|\widebar{\bm{S}}_0\bm{x}\big\|^2\ge \frac{0.99|\widebar{\mathcal{I}}_0|}{2.3n}\left[1+\log \big(m\big/|\widebar{\mathcal{I}}_0|\big)\right]
\end{equation}
provided that $|\widebar{\mathcal{I}}_0|\ge c_1' n$, $m\ge c_2' |\widebar{\mathcal{I}}_0|$, and $m\ge c_3' n$ for some absolute constants $c_1',\,c_2',\,c_3'>0$, and sufficiently large $n$. 
\end{lemma}


Leveraging the upper
and lower bounds in~\eqref{eq:up0} and~\eqref{eq:low0}, one arrives at
\begin{align}\label{eq:fbound}
\frac{\big\|\widebar{\bm{S}}_0\bm{u}\big\|^2}{\big\|\widebar{\bm{S}}_0\bm{x}\big\|^2}\le \frac{2.4}{1+\log\big(m/|\widebar{\mathcal{I}}_0|\big)} \buildrel\triangle\over = 
\kappa
\end{align}
which holds with probability at least $1-(m+3){\rm e}^{-n/2}-{\rm e}^{-c_0m}-1/n^2$, assuming that $m\ge c_1' |\widebar{\mathcal{I}}_0|$, and $m\ge c_2' n$, $|\widebar{\mathcal{I}}_0|\ge c_3' n$ for some absolute constants $c_1',\,c_2',\,c_3'>0$, and sufficiently large $n$. 

The bound $\kappa$ in \eqref{eq:fbound} is meaningful only when the ratio $\log(m/|\widebar{\mathcal{I}}_0|)>1.4$, i.e., $m/|\widebar{\mathcal{I}}_0|> 4$, because the left hand side is expressible in terms of $\sin^2\theta$, and therefore, enjoys a trivial upper bound of $1$. Henceforth, we will assume $m/|\widebar{\mathcal{I}}_0|> 4$. Empirically, $\lfloor m/|\widebar{\mathcal{I}}_0|\rfloor=6$ or equivalently $|\widebar{\mathcal{I}}_0|=\lceil\frac{1}{6}m\rceil$ in Algorithm~\ref{alg:TAF} works well when $m/n$ is relatively small.  
 Note further that the bound $\kappa$ can be made arbitrarily small  
by letting $m/|\widebar{\mathcal{I}}_0|$ be large enough. Without any loss of generality, let us take $\kappa:=0.001$. An additional step leads to the wanted bound on the distance between $\tilde{\bm{z}}_0$ and $\bm{x}$; similar arguments are found in \cite[Section 7.8]{wf}. Recall that
\begin{equation}
	|\bm{x}^\ccalT\tilde{\bm{z}}_0|^2=\cos^2\theta=1-\sin^2\theta\ge 1-\kappa.
\end{equation}
Therefore,
\begin{align}
		{\rm dist}^2(\tilde{\bm{z}}_0,\,\bm{x})&\le \|\tilde{\bm{z}}_0\|^2+\|\bm{x}\|^2-2|\bm{x}^\ccalT\tilde{\bm{z}}_0|\nonumber\\
		&\le\left( 2-2\sqrt{1-\kappa}\right)\left\|\bm{x}\right\|^2\nonumber\\
		&\approx \kappa\left\|\bm{x}\right\|^2.
\end{align}

Coming back to the case in which $\|\bm{x}\|$ is unknown stated prior to Lemma~\ref{lem:beta},  
the unit eigenvector $\tilde{\bm{z}}_0$ is scaled by an estimate of $\|\bm{x}\|$ to yield the initial guess $\bm{z}_0=\sqrt{\frac{1}{m}\sum_{i=1}^m y_i}\tilde{\bm{z}}_0$. Using the results in Lemma 7.8 in \cite{wf}, the following holds with high probability 
\begin{equation}
	\left\|\bm{z}_0-\tilde{\bm{z}}_0\right\|=\left|\left\|\bm{z}_0\right\|-1\right|\le (1/20)\left\|\bm{x}\right\|.
\end{equation}
Summarizing the two inequalities, we conclude that
\begin{align}
	{\rm dist}(\bm{z}_0,\,\bm{x})&\le \left\|\bm{z}_0-\tilde{\bm{z}}_0\right\|+{\rm dist}(\tilde{\bm{z}}_0,\,\bm{x})\le (1/10)\left\|\bm{x}\right\|.
\end{align}
The initialization thus obeys ${\rm dist}(\bm{z}_0,\,\bm{x})/\|\bm{x}\|\le 1/10$ for any $\bm{x}\in\mathbb{R}^n$ with high probability provided that $m\ge c_1|\widebar{\mathcal{I}}_0|\ge c_2n$ holds for some universal constants $c_1,\,c_2>0$ and sufficiently large $n$.

\subsection{Exact recovery from noiseless data}\label{subsec:er}

We now prove that with accurate enough initial estimates, TAF converges at a geometric rate to $\bm{x}$ with high probability~(i.e., the second part of Theorem~\ref{thm:initial}). To be specific, with initialization obeying~\eqref{eq:i1x} in Proposition~\ref{prop:initial}, TAF reconstructs the solution $\bm{x}$ exactly in linear time.  
To start, it suffices to demonstrate that the TAF's update rule (i.e., Step~4 in Algorithm~\ref{alg:TAF}) is locally contractive within a sufficiently small neighborhood of $\bm{x}$, as asserted in the following proposition. 
  
 \begin{proposition}[Local error contraction] \label{prop:lec}
 	Consider the noise-free measurements $\psi_i=\left|\bm{a}_i^\ccalT\bm{x}\right|$ with i.i.d. Gaussian design vectors $\bm{a}_i\sim\mathcal{N}(\bm{0},\,\bm{I}_n)$,~$1\le i\le m$, and fix any $ \gamma\ge 1/2$. There exist universal constants $c_0,\,c_1>0$ and $0<\nu<1$ such that with probability at least $1-7{\rm e}^{-c_0m}$, the following holds
\begin{equation}\label{eq:contract}
	{\rm dist}^2\left(\bm{z}-\frac{\mu}{m}\nabla \ell_{\rm tr}(\bm{z}),\,\bm{x}\right)\le (1-\nu){\rm dist}^2\left(\bm{z},\,\bm{x}\right)
\end{equation}
for all $\bm{x},\,\bm{z}\in\mathbb{R}^n$ obeying \eqref{eq:i1x} with the proviso that $m\ge c_1 n$ and 
that the constant step size $\mu$ satisfies $0<\mu\le \mu_0$ for some $\mu_0>0$. 
 \end{proposition}

 Proposition~\ref{prop:lec} demonstrates that the distance of TAF's successive iterates to $\bm{x}$ is monotonically decreasing once the algorithm enters a small-size neighborhood around $\bm{x}$. This neighborhood is commonly referred to as the \emph{basin of attraction}; see further discussions in~\cite{wf,thesis2014,twf,localcvx, focs2015sun}. In other words, as soon as one lands within the basin of attraction, TAF's iterates remain in this region and will be attracted to $\bm{x}$ exponentially fast. To substantiate Proposition~\ref{prop:lec}, recall the \emph{local regularity condition}, which was first developed in~\cite{wf} and plays a fundamental role in establishing linear convergence to global optimum of nonconvex optimization approaches such as WF/TWF~\cite{wf,thesis2014,twf,mtwf}.
 
  Consider the update rule of TAF
\begin{equation}
\bm{z}_{t+1}=\bm{z}_t-\frac{\mu}{m}\nabla\ell_{\rm tr}(\bm{z}_t),\quad t=0,\,1,\,2,\,\ldots
\end{equation}
where the truncated gradient $\nabla\ell_{\rm tr}(\bm{z}_t)$ (as elaborated in Remark~\ref{rmk:gg}) evaluated at some point $\bm{z}_t\in\mathbb{R}^n$ is given by
 $$\frac{1}{m}\nabla\ell_{\rm tr}(\bm{z}_t)\buildrel\triangle\over =\frac{1}{m}\sum_{i\in\mathcal{I}}\left(\bm{a}_i^\ccalT\bm{z}_t-\psi_i\frac{\bm{a}_i^\ccalT\bm{z}_t}{|\bm{a}_i^\ccalT\bm{z}_t|}
\right)\bm{a}_i.  $$
The truncated gradient $\nabla\ell_{\rm tr}(\bm{z})$ 
 is said to satisfy the local regularity condition, or ${\rm LRC}(\mu,\lambda,\epsilon)$ for some constant $\lambda>0$, provided that 
 \begin{align}
	\label{eq:lrc}
\left\langle\frac{1}{m}\nabla\ell_{\rm tr}(\bm{z}),\,\bm{h}\right\rangle \ge \frac{\mu}{2}\left\|\frac{1}{m}\nabla\ell_{\rm tr}(\bm{z})\right\|^2+\frac{\lambda}{2}\left\|\bm{h}\right\|^2
 \end{align}
 holds for all $\bm{z}\in\mathbb{R}^n$ such that $\left\|\bm{h}\right\|=\left\|\bm{z}-\bm{x}\right\|
 \le \epsilon\left\|\bm{x}\right\|$ for some constant $0<\epsilon<1$, where the ball $\left\|\bm{z}-\bm{x}\right\|\le \epsilon\left\|\bm{x}\right\|$ is the so-called \emph{basin of attraction}. Simple linear algebra along with the regularity condition in \eqref{eq:lrc} leads to
 \begin{align}
 	{\rm dist}^2&\left(\bm{z}-\frac{\mu}{m}\nabla\ell_{\rm tr}(\bm{z}),\bm{x}\right)=\left\|\bm{z}-\frac{\mu}{m}\nabla\ell_{\rm tr}(\bm{z})-\bm{x}
 	\right\|^2\nonumber\\
 	&=\left\|\bm{h}\right\|^2-2\mu\left\langle\bm{h},\frac{1}{m}\nabla
 	\ell_{\rm tr}(\bm{z})\right\rangle+\left\|\frac{\mu}{m}\nabla\ell_{\rm tr}(\bm{z})\right\|^2\label{eq:lastterm}\\
 	&\le \left\|\bm{h}\right\|^2-2\mu\left( \frac{\mu}{2}\left\|\frac{1}{m}\nabla\ell_{\rm tr}(\bm{z})\right\|^2+\frac{\lambda}{2}\left\|\bm{h}\right\|^2\right)
 +\left\|\frac{\mu}{m}\nabla\ell_{\rm tr}(\bm{z})\right\|^2\nonumber\\
 	&=\left(1-\lambda\mu\right)\left\|\bm{h}\right\|^2=\left(1-\lambda\mu\right){\rm dist}^2(\bm{z},\bm{x})\label{eq:lrc1}
 \end{align}
 for all $\bm{z}$ obeying $\left\|\bm{h}\right\|\le \epsilon\left\|\bm{x}\right\|$. 
Evidently, if the ${\rm LRC}(\mu,\lambda,\epsilon)$ is proved for TAF, then \eqref{eq:contract} follows upon letting $\nu:=\lambda\mu$.

 \subsubsection{Proof of the local regularity condition in~\eqref{eq:lrc}}
 
 By definition, justifying the local regularity condition in~\eqref{eq:lrc} entails controlling the norm of the truncated gradient $\frac{1}{m}\nabla\ell_{\rm tr}(\bm{z})$, i.e., bounding the last term in~\eqref{eq:lastterm}. Recall that 
 \begin{equation}
 	 	\frac{1}{m}\nabla\ell_{\rm tr}(\bm{z})=\frac{1}{m}\sum_{i\in\mathcal{I}}\left(\bm{a}_i^\ccalT\bm{z}-\psi_i\frac{\bm{a}_i^\ccalT\bm{z}}{\left|\bm{a}_i^\ccalT\bm{z}\right|}\right)\bm{a}_i\buildrel\triangle\over =\frac{1}{m}\bm{A}\bm{v}\label{eq:eqav}
 \end{equation}
 where $\mathcal{I}:=\{1\le i\le m |{|\bm{a}_i^\ccalT\bm{z}|}\ge |\bm{a}_i^\ccalT\bm{x}|/{(1+\gamma)}\}$, and 
 $\bm{v}:=[v_1~\cdots~v_m]^\ccalT\in\mathbb{R}^m$ with $v_i:=\frac{\bm{a}_i^\ccalT\bm{z}}{\left|\bm{a}_i^\ccalT\bm{z}\right|}\left(|\bm{a}_i^\ccalT\bm{z}|-\psi_i \right)\mathbb{1}_{\left\{{|\bm{a}_i^\ccalT\bm{z}|}\ge |\bm{a}_i^\ccalT\bm{x}|/{(1+\gamma)}\right\}}$.   
Now, consider
 \begin{align}
 	 	 	|v_i|^2=\left|
 	 	\left(\left|\bm{a}_i^\ccalT\bm{z}\right|-\left|\bm{a}_i^\ccalT\bm{x}\right| \right)\mathbb{1}_{\left\{{|\bm{a}_i^\ccalT\bm{z}|}\ge |\bm{a}_i^\ccalT\bm{x}|/{(1+\gamma)}\right\}}\right|^2
 	 	\le \left|\left|\bm{a}_i^\ccalT\bm{z}\right|-\left|\bm{a}_i^\ccalT\bm{x}\right| \right|^2\le  \left|\bm{a}_i^\ccalT\bm{h}\right|^2
 	 	 	 	\label{eq:ineqvi}
 \end{align}
 where $\bm{h}=\bm{z}-\bm{x}$. 
Appealing to \cite[Lemma 3.1]{phaselift}, fixing any $\delta'>0$, 
the following holds for any $\bm{h}\in\mathbb{R}^n$ with probability at least $1-{\rm e}^{-m\delta'^2/2}$: 
 \begin{equation}\label{eq:vnorm}
 	\|\bm{v}\|^2=\sum_{i=1}^mv_i^2\le \sum_{i=1}^m\left|\bm{a}_i^\ccalT\bm{h}\right|^2\le (1+\delta')m\|\bm{h}\|^2.
 \end{equation}
 
  On the other hand, standard matrix concentration results confirm that the largest singular value of $\bm{A}=\left[\bm{a}_1~\cdots~\bm{a}_m\right]^\ccalT$ with i.i.d. Gaussian $\{\bm{a}_i\}$ 
   satisfies $\sigma_1:=\|\bm{A}\|\le(1+\delta'') \sqrt{m}$~for some $\delta''>0$ with probability exceeding $1-2{\rm e}^{-c_0m}$ as soon as $m\ge c_1n$ for sufficiently large $c_1>0$, where $c_1>0$ is a universal constant depending on $\delta''$~\cite[Remark 5.25]{chap2010vershynin}.  
Combining \eqref{eq:eqav}, \eqref{eq:ineqvi}, and \eqref{eq:vnorm} yields 
 \begin{align}
 \left\|\frac{1}{m}\nabla\ell_{\rm tr}(\bm{z})\right\|\le \frac{1}{m}\left\|\bm{A}\right\|\cdot\|\bm{v}\|
 \le(1+\delta')(1+\delta'')\|\bm{h}\|
 \le  (1+\delta)^2\left\|\bm{h}\right\|,\quad \delta:=\max\{\delta',\delta''\}
 \label{eq:gdnorm}
 \end{align}
 which holds with high probability. 
 This condition essentially asserts that the truncated  gradient of the objective function $\ell(\bm{z})$
 or the search direction is well behaved (the function value does
not vary too much).

 	We have related $\|\nabla\ell_{\rm tr}(\bm{z})\|^2$ to $\|\bm{h}\|^2$ through \eqref{eq:gdnorm}.  
Therefore, a more conservative lower bound for $	\langle \frac{1}{m}\nabla
\ell_{\rm tr}(\bm{z}),\,\bm{h}\rangle$ in LRC can be given in terms of $\|\bm{h}\|^2$. It is equivalent to show that the truncated gradient $\frac{1}{m}\nabla\ell_{\rm tr}(\bm{z})$ ensures sufficient descent, 
  i.e., it obeys a uniform lower bound along the search direction $\bm{h}$ taking the form
 \begin{equation}
 	\left\langle \frac{1}{m}\nabla
 	\ell_{\rm tr}(\bm{z}),\,\bm{h}\right\rangle \gtrsim \|\bm{h}\|^2
 \end{equation} 
 which occupies the remaining of this section. Formally, this can be stated as follows.
 
 \begin{proposition}\label{prop:inner}
 	Consider the noiseless measurements $\psi_i=|\bm{a}_i^\ccalT\bm{x}|$, and fix any sufficiently small constant $\epsilon>0$. There exist universal constants $c_0,\,c_1>0$ such that if $m> c_1n$,
 	 then the following holds with probability exceeding $1-4{\rm e}^{-c_0m}$:
\begin{equation}\label{eq:innerproduct}
	\left\langle\frac{1}{m}\nabla	\ell_{\rm tr}(\bm{z}),\bm{h}\right\rangle \ge 2\left(1-\zeta_1-\zeta_2-2\epsilon
	\right)\left\|\bm{h}\right\|^2
\end{equation} 
for all $\bm{x},\,\bm{z}\in\mathbb{R}^n$ such that $\left\|\bm{h}\right\|/\left\|\bm{x}\right\|\le \rho$ for $0<\rho\le 1/10$ and any fixed $ \gamma\ge 1/2$. 
 \end{proposition}

Before justifying Proposition \ref{prop:inner}, we introduce the following events.
\begin{lemma}\label{le:events}
	Fix any $\gamma>0$. For each $i\in[m]$, define 
			\begin{align}
		\mathcal{E}_i\,&:=\left\{\frac{|\bm{a}_i^\ccalT\bm{z}|}{|\bm{a}_i^\ccalT\bm{x}|}\ge \frac{1}{1+\gamma}
		\right\}\label{eq:ee},\\
		\mathcal{D}_i&:=\left\{\frac{\left|\bm{a}_i^\ccalT\bm{h}\right|}{
	\left|\bm{a}_i^\ccalT\bm{x}\right|}\ge \frac{2+\gamma}{1+\gamma}\right\}\label{eq:ed},\\
	{\rm and}\quad 	
	\mathcal{K}_i&:=\left\{\frac{\bm{a}_i^\ccalT\bm{z}}{|\bm{a}_i^\ccalT\bm{z}|}\ne \frac{\bm{a}_i^\ccalT\bm{x}}{|\bm{a}_i^\ccalT\bm{x}|}\right\}\label{eq:ek}
	\end{align}
	where $\bm{h}=\bm{z}-\bm{x}$. 
	Under the condition $\left\|\bm{h}\right\|/\left\|\bm{x}\right\|\le \rho$, the following inclusion holds for all nonzero $\bm{z},\;\bm{h}\in\mathbb{R}^n$
	\begin{equation}
		\label{eq:inclusion}
	\mathcal{E}_i\cap	\mathcal{K}_i
		\subseteq \mathcal{D}_i\cap	\mathcal{K}_i. 
	\end{equation}
\end{lemma}

\begin{proof}
	From Fig.~\ref{fig:truncation}, it is clear that if $\bm{z}\in\xi_i^2$, then the sign of $\bm{a}_i^\ccalT\bm{z}$ will be different than that of $\bm{a}_i^\ccalT\bm{x}$. The region $\xi_i^2$ can be readily specified by the conditions that 
	\begin{equation*}
		\frac{\bm{a}_i^\ccalT\bm{z}}{\left|\bm{a}_i^\ccalT\bm{z}\right|}\ne \frac{\bm{a}_i^\ccalT\bm{x}}{\left|\bm{a}_i^\ccalT\bm{x}\right|}
	\end{equation*}
 and $$\frac{\left|\bm{a}_i^\ccalT\bm{h}\right|}{\left|\bm{a}_i^\ccalT\bm{x}\right|}\ge 1+\frac{1}{1+\gamma}=\frac{2+\gamma}{1+\gamma}.$$  
	Under our initialization condition $\left\|\bm{h}\right\|/\left\|\bm{x}\right\|\le \rho$, it is self-evident that $\mathcal{D}_i$ describes two symmetric spherical caps over $\bm{a}_i^\ccalT\bm{x}=\psi_i$ with one being $\xi_i^2$. Hence, it holds that $\mathcal{E}_i\cap\mathcal{K}_i=\xi_i^2\subseteq \mathcal{D}_i\cap	\mathcal{K}_i$.   
\end{proof}

To prove \eqref{eq:innerproduct}, consider rewriting the truncated gradient in terms of the events defined in Lemma \ref{le:events}: 
\begin{align}
	\frac{1}{m}\nabla\ell_{\rm tr}(\bm{z})
	&=\frac{1}{m}\sum_{i=1}^m\left(\bm{a}_i^\ccalT\bm{z}-\left|\bm{a}_i^\ccalT\bm{x}\right|
	\frac{\bm{a}_i^\ccalT\bm{z}}{|\bm{a}_i^\ccalT\bm{z}|}\right)\bm{a}_i\mathbb{1}_{\mathcal{E}_i}\nonumber\\
	&=\frac{1}{m}\sum_{i=1}^m \bm{a}_i\bm{a}_i^\ccalT\bm{h}\mathbb{1}_{\mathcal{E}_i}-\!
\frac{1}{m}\sum_{i=1}^m \!	\left(\!\frac{\bm{a}_i^\ccalT\bm{z}}{|\bm{a}_i^\ccalT\bm{z}|}-\!\frac{\bm{a}_i^\ccalT\bm{x}}{|\bm{a}_i^\ccalT\bm{x}|}\right)\left|\bm{a}_i^\ccalT\bm{x}\right|\bm{a}_i\mathbb{1}_{\mathcal{E}_i}.
\end{align}
Using the definitions and properties in Lemma~\ref{le:events}, one further arrives at 
\begin{align}\label{eq:target}
	\left\langle\frac{1}{m}\nabla\ell_{\rm tr}(\bm{z}),\,\bm{h}\right\rangle 
	&\ge\frac{1}{m}\sum_{i=1}^m\left(\bm{a}_i^\ccalT\bm{h}\right)^2\mathbb{1}_{\mathcal{E}_i} -\frac{1}{m}\sum_{i=1}^m
	\left|\bm{a}_i^\ccalT\bm{x}\right|
	 \left|\bm{a}_i^\ccalT\bm{h}\right|\mathbb{1}_{\mathcal{E}_i\cap\mathcal{K}_i}\nonumber\\
	&\ge \frac{1}{m}\sum_{i=1}^m\left(\bm{a}_i^\ccalT\bm{h}\right)^2\mathbb{1}_{\mathcal{E}_i} -\frac{2}{m}\sum_{i=1}^m\left|\bm{a}_i^\ccalT\bm{x}\right|
	\left|\bm{a}_i^\ccalT\bm{h}\right|\mathbb{1}_{\mathcal{D}_i\cap	\mathcal{K}_i}\nonumber\\
	&\ge \frac{1}{m}\sum_{i=1}^m\left(\bm{a}_i^\ccalT\bm{h}\right)^2\mathbb{1}_{\mathcal{E}_i} -\frac{1+\gamma}{2+\gamma}\cdot 
	\frac{2}{m}\sum_{i=1}^m
	\left(\bm{a}_i^\ccalT\bm{h}\right)^2\mathbb{1}_{\mathcal{D}_i\cap	\mathcal{K}_i}
\end{align}
where the last inequality arises from the property ${\left|\bm{a}_i^\ccalT\bm{x}\right|}\le \frac{1+\gamma}{2+\gamma} {\left|\bm{a}_i^\ccalT\bm{h}\right|}$ by the definition of $\mathcal{D}_i$. 

Establishing the regularity condition or Proposition \ref{prop:inner}, boils down to lower bounding the right-hand side of \eqref{eq:target}, namely, to lower bounding the first term and to upper bounding the second one. 
By the SLLN, the first term in \eqref{eq:target}
approximately gives $\left\|\bm{h}\right\|^2$ as long as our truncation procedure does not eliminate too many generalized gradient components (i.e., summands in the first term). 
Regarding the second, one would expect its contribution to be small under our initialization condition in \eqref{eq:i1x}  and as the relative error $\left\|\bm{h}\right\|/\left\|\bm{x}\right\|$ decreases. 
Specifically, 
under our initialization, $\mathcal{D}_i$ is provably a rare event, thus eliminating the possibility of the second term exerting a noticeable influence on the first term. Rigorous analyses concerning the two terms are elaborated in Lemma~\ref{le:1stterm} and Lemma~\ref{le:2ndterm}, whose proofs are provided in Appendix~\ref{proof:1stterm} and Appendix~\ref{sec:rare}, respectively.  
\begin{lemma}
	\label{le:1stterm}
	Fix $\gamma\ge 1/2$ and $\rho\le 1/10$, and let $\mathcal{E}_i$ be defined in~\eqref{eq:ee}. For independent random variables $W\sim\mathcal{N}(0,\,1)$ and $Z\sim\mathcal{N}(0,\,1)$, set
	\begin{align}
		\label{eq:zeta12}
		\zeta_1:=1-\min\bigg\{
		\mathbb{E}\left[\mathbb{1}_{\left\{\left|\frac{1-\rho}{\rho}+\frac{W}{Z}
		\right|\ge \frac{\sqrt{1.01}}{\rho\left(1+\gamma\right)}
		\right\}
		}
		\right], \mathbb{E}\left[Z^2\mathbb{1}_{\left
		\{\left|\frac{1-\rho}{\rho}+\frac{W}{Z}
		\right|\ge \frac{\sqrt{1.01}}{\rho\left(1+\gamma\right)}\right\}
		}
		\right]
		\bigg\}.
	\end{align}
	Then for any $\epsilon>0$ and any vector $\bm{h}$ obeying $\left\|\bm{h}\right\|/\left\|\bm{x}\right\|\le \rho$, the following holds 
	with probability exceeding $1-2{\rm e}^{-c_5\epsilon^2 m}$:
	\begin{equation}\label{eq:1stterm}
	\frac{1}{m}\sum_{i=1}^m\left(\bm{a}_i^\ccalT\bm{h}\right)^2\mathbb{1}_{\mathcal{E}_i}\ge \left(
	1-\zeta_1-\epsilon\right)\left\|\bm{h}\right\|^2
	\end{equation}
provided that $m>(c_6 \cdot\epsilon^{-2}\log\epsilon^{-1})n$ for some universal constants $c_5,\,c_6>0$.  
\end{lemma}

To have a sense of how large the quantities involved in Lemma~\ref{le:1stterm} are, 
		when $\gamma=0.7$ and $\rho=1/10$, it holds that
		$$\mathbb{E}\Big[\mathbb{1}_{\left\{\left|\frac{1-\rho}{\rho}+\frac{W}{Z}
		\right|\ge \frac{\sqrt{1.01}}{\rho\left(1+\gamma\right)}
		\right\}
		}
		\Big]\approx 0.92$$ and $$\mathbb{E}\Big[Z^2\mathbb{1}_{\left
		\{\left|\frac{1-\rho}{\rho}+\frac{W}{Z}
		\right|\ge \frac{\sqrt{1.01}}{\rho\left(1+\gamma\right)}\right\}
		}
		\Big]\approx 0.99$$
		 hence leading to $\zeta_1\approx 0.08$.

Having derived a lower bound for the first term in the right-hand side of~\eqref{eq:target}, it remains to deal with the second one. 

\begin{lemma}
	\label{le:2ndterm}
	Fix $\gamma> 0$ and $\rho\le 1/10$, and let $\mathcal{D}_i$, $\mathcal{K}_i$ be defined in~\eqref{eq:ed}, \eqref{eq:ek}, respectively. For any constant $\epsilon>0$, there exists some universal constants $c_5,\,c_6>0$ such that
	\begin{equation}
		\label{eq:2ndterm}
\frac{1}{m}\sum_{i=1}^m\left(\bm{a}_i^\ccalT\bm{h}\right)^2\mathbb{1}_{\mathcal{D}_i\cap \mathcal{K}_i}\le \left(\zeta_2'+\epsilon\right)\left\|\bm{h}\right\|^2
\end{equation}
holds with probability at least $1-2{\rm e}^{-c_5\epsilon^2 m}$ provided that $m/n>c_6 \cdot\epsilon^{-2}\log\epsilon^{-1}$ for some universal constants $c_5,\,c_6>0$, where $\zeta_2'=0.9748\sqrt{\rho\tau/(0.99\tau^2-\rho^2)}$ with $\tau=(2+\gamma)/(1+\gamma)$.  
\end{lemma}

With our TAF default parameters $\rho=1/10$ and $\gamma=0.7$, we have $\zeta_2'\approx 0.2463$.
Using \eqref{eq:target},~\eqref{eq:1stterm}, and~\eqref{eq:2ndterm}, choosing $m/n$ exceeding some sufficiently large constant such that $c_0\le c_5\epsilon^2$, and denoting $\zeta_2:=2\zeta_2'(1+\gamma)/(2+\gamma)$, the following holds with probability exceeding $1-4{\rm e}^{-c_0m}$
\begin{equation}
	\left\langle\bm{h},\frac{1}{m}\nabla\ell_{\rm tr}(\bm{z})\right\rangle \ge \left(1-\zeta_1-\zeta_2-2\epsilon
	\right)\left\|\bm{h}\right\|^2
\end{equation} 
for all $\bm{x}$ and $\bm{z}$ such that $\left\|\bm{h}\right\|/\left\|\bm{x}\right\|\le \rho$ for $0<\rho\le 1/10$ and any fixed $ \gamma\ge 1/2$. This combined with \eqref{eq:lrc} and \eqref{eq:lrc1} proves Proposition~\ref{prop:lec} for appropriately chosen $\mu>0$ and $\lambda>0$.

To conclude this section, an estimate for the working step size is provided next. Plugging the results of~\eqref{eq:gdnorm} and~\eqref{eq:innerproduct} into~\eqref{eq:lastterm} suggests that 
\begin{align} 
 	{\rm dist}^2\left(\bm{z}-\frac{\mu}{m}\nabla\ell_{\rm tr}(\bm{z}),\bm{x}\right)
 	&=\left\|\bm{h}\right\|^2-2\mu\left\langle\bm{h},\frac{1}{m}\nabla
 	\ell_{\rm tr}(\bm{z})\right\rangle+\left\|\frac{\mu}{m}\nabla\ell_{\rm tr}(\bm{z})\right\|^2\\
 	&\le\left\{1-\mu\left[2\left(1-\zeta_1-\zeta_2-2\epsilon\right)-\mu(1+\delta)^4\right]\right\}
 \left\|\bm{h}\right\|^2
\nonumber\\
 	&\buildrel\triangle\over =\left(1-\nu\right)\left\|\bm{h}\right\|^2,
\end{align}
and also that $\lambda=2\left(1-\zeta_1-\zeta_2-2\epsilon\right)-\mu(1+\delta)^4\buildrel\triangle\over =\lambda_0$ in the local regularity condition in \eqref{eq:lrc}. Clearly, it holds that $0<\lambda<2(1-\zeta_1-\zeta_2)$. 
Taking $\epsilon$ and $\delta$ to be sufficiently small, one obtains the feasible range of the step size for TAF
\begin{equation}
\mu\le \frac{2\left(0.99-\zeta_1-\zeta_2\right)}{1.05^4}\buildrel\triangle\over =\mu_0.
\end{equation}
In particular, 
under default parameters in Algorithm \ref{alg:TAF}, $\mu_0=0.8388$ and $\lambda_0=1.22$,
thus concluding the proof of Theorem~\ref{thm:initial}. 

\section{Conclusion}
This paper developed a linear-time algorithm termed TAF for solving generally unstructured systems of random quadratic equations. Our TAF algorithm builds on three key ingredients:
an orthogonality-promoting initialization, along with a simple yet effective gradient truncation rule, as well as scalable gradient-like iterations. Numerical tests using synthetic data and real images
corroborate the superior performance of TAF over state-of-the-art solvers of the same type.

A few timely and pertinent future research directions are worth pointing out. First, in parallel with spectral initialization methods, 
the proposed orthogonality-promoting initialization can be applied for semidefinite optimization~\cite{localcvx}, matrix completion~\cite{tit2010spectral,focs2015sun}, as well as blind deconvolution~\cite{2016li}. It is also interesting to investigate suitable gradient regularization rules in more general nonconvex optimization settings. 
Extending the theory to the more challenging case where $\bm{a}_i$'s are generated from the coded diffraction pattern model~\cite{coded} constitutes another meaningful direction.


\appendix

\section{Proofs for Section \ref{sec:proof}}

\subsection{Proof of Lemma \ref{lem:beta}}
\label{sec:lembeta}
By homogeneity of \eqref{eq:i1x}, it suffices to work with the case where $\|\bm{x}\|=1$.
It is easy to check that 
\begin{align}
	\frac{1}{2}\left\|\bm{x}\bm{x}^\ccalT-\tilde{\bm{z}}_0\tilde{\bm{z}}_0^\ccalT\right\|_F
	^2&=\frac{1}{2}\|\bm{x}\|^4+\frac{1}{2}\|\tilde{\bm{z}}_0\|^4-|\bm{x}^\ccalT\tilde{\bm{z}}_0|^2\nonumber\\
	&=1-|\bm{x}^\ccalT\tilde{\bm{z}}_0|^2\nonumber\\
	&=1-\cos^2\theta\label{eq:ineq1}
	\end{align}
where $0\le \theta\le \pi/2$ is the angle between the spaces spanned by $\bm{x}$ and $\tilde{\bm{z}}_0$. Then one can write 
\begin{equation}
	\label{eq:orth1}
	\bm{x}=\cos\theta\,\tilde{\bm{z}}_0+\sin\theta\,\tilde{\bm{z}}_0^{\perp},
\end{equation}
where $\tilde{\bm{z}}_0^\perp\in \mathbb{R}^n$ is a unit vector that is orthogonal to $\tilde{\bm{z}}_0$ and has a nonnegative inner product with $\bm{x}$. 
Likewise,
\begin{equation}
	\label{eq:orth2}
	\bm{x}^{\perp}:=-\sin\theta\,\tilde{\bm{z}}_0+\cos\theta\,\tilde{\bm{z}}_0^{\perp},
\end{equation}
in which $\bm{x}^{\perp}\in\mathbb{R}^n$ is a unit vector orthogonal to $\bm{x}$. 

Since $\tilde{\bm{z}}_0$ is the solution to the maximum eigenvalue problem
\vspace{-.em}
\begin{align}\label{eq:maxeig1}
	\tilde{\bm{z}}_0:=\arg\max_{\left\|\bm{z}\right\|=1}~&~\bm{z}^\ccalT\widebar{\bm{Y}}_0\bm{z}
\end{align}		
for $\widebar{\bm{Y}}_0:=\frac{1}{|\widebar{\mathcal{I}}_0|}\widebar{\bm{S}}^\ccalT_0\widebar{\bm{S}}_0$,
it is the leading eigenvector of $\widebar{\bm{Y}}_0$, i.e., $\widebar{\bm{Y}}_0\tilde{\bm{z}}_0=\lambda_{1}\tilde{\bm{z}}_0$, where $\lambda_1>0$ is the largest eigenvalue of $\widebar{\bm{Y}}_0$. 
Premultiplying \eqref{eq:orth1} and \eqref{eq:orth2} by $\widebar{\bm{S}}_0$ yields  
\begin{subequations}\label{eq:prem}
\begin{align}
	\widebar{\bm{S}}_0	\bm{x}&=\cos\theta\,	\widebar{\bm{S}}_0\tilde{\bm{z}}_0+\sin\theta\,	\widebar{\bm{S}}_0\tilde{\bm{z}}_0^{\perp}\label{eq:prem1},\\
		\widebar{\bm{S}}_0	\bm{x}^{\perp}&=-\sin\theta\,	\widebar{\bm{S}}_0\tilde{\bm{z}}_0+\cos\theta\,	\widebar{\bm{S}}_0\tilde{\bm{z}}_0^{\perp}\label{eq:prem2}.
\end{align}	
\end{subequations}
Pythagoras' relationship now gives 
\begin{subequations}\label{eq:prem}
\begin{align}
	\big\|\widebar{\bm{S}}_0\bm{x}\big\|^2&=\cos^2\theta\big\|\widebar{\bm{S}}_0\tilde{\bm{z}}_0\big\|^2+\sin^2\theta\big\|\widebar{\bm{S}}_0\tilde{\bm{z}}_0^{\perp}\big\|^2\label{eq:prem11},\\
		\big\|\widebar{\bm{S}}_0\bm{x}^\perp\big\|^2&=\sin^2\theta\big\|\widebar{\bm{S}}_0\tilde{\bm{z}}_0\big\|^2+\cos^2\theta\big\|\widebar{\bm{S}}_0\tilde{\bm{z}}_0^{\perp}\big\|^2\label{eq:prem21},
\end{align}	
\end{subequations}
where the cross-terms vanish because $\tilde{\bm{z}}_0^\ccalT\widebar{\bm{S}}_0^\ccalT\widebar{\bm{S}}_0\tilde{\bm{z}}_0^{\perp}=|\widebar{\mathcal{I}}_0|\tilde{\bm{z}}_0^\ccalT\widebar{\bm{Y}}_0\tilde{\bm{z}}_0^\perp=\lambda_1|\widebar{\mathcal{I}}_0|\tilde{\bm{z}}_0^\ccalT\tilde{\bm{z}}_0^{\perp}=0$ following from the definition of $\tilde{\bm{z}}^\perp_0$. 

We next construct the following expression:
\begin{align}
	&\sin^2\theta\big\|\widebar{\bm{S}}_0\bm{x}\big\|^2-\big\|\widebar{\bm{S}}_0\bm{x}^\perp\big\|^2\nonumber\\
	&=\sin^2\theta\Big(\cos^2\theta\big\|\widebar{\bm{S}}_0\tilde{\bm{z}}_0\big\|^2+\sin^2\theta\big\|\widebar{\bm{S}}_0\tilde{\bm{z}}_0^\perp\big\|^2\Big) -\Big(\sin^2\theta\big\|\widebar{\bm{S}}_0\tilde{\bm{z}}_0\big\|^2+\cos^2\theta\big\|\widebar{\bm{S}}_0\tilde{\bm{z}}_0^\perp \big\|^2\Big)\nonumber\\
		&=\sin^2\theta\Big(
		\cos^2\theta\big\|\widebar{\bm{S}}_0\tilde{\bm{z}}_0\big\|^2-\big\|\widebar{\bm{S}}_0\tilde{\bm{z}}_0\big\|^2+\sin^2\theta\big\|\widebar{\bm{S}}_0\tilde{\bm{z}}_0^\perp\big\|^2 \Big)-
	\cos^2\theta\big\|\widebar{\bm{S}}_0\tilde{\bm{z}}_0^\perp \big\|^2\nonumber\\
	&=\sin^4\theta\Big(\big\|\widebar{\bm{S}}_0\tilde{\bm{z}}_0^\perp\big\|^2-\big\|\widebar{\bm{S}}_0\tilde{\bm{z}}_0\big\|^2\Big)-\cos^2\theta\big\|\widebar{\bm{S}}_0\tilde{\bm{z}}_0^\perp\big\|^2\label{eq:last11}\\
	&\le 0.\nonumber
\end{align}
Regarding the last inequality, since $\tilde{\bm{z}}_0$ maximizes the term $\tilde{\bm z}_0^\ccalT\widebar{\bm Y}_0\tilde{\bm z}_0=\frac{1}{|\widebar{\mathcal{I}}_0|}\tilde{\bm z}_0^\ccalT\widebar{\bm S}_0^\ccalT\widebar{\bm S}_0\tilde{\bm z}_0$ according to \eqref{eq:maxeig1}, then in \eqref{eq:last11} the first term $\|\widebar{\bm{S}}_0\tilde{\bm{z}}_0^\perp\|^2-\|\widebar{\bm{S}}_0\tilde{\bm{z}}_0\|^2\le 0$ holds for any unit vector $\tilde{\bm{z}}_0^\perp\in\mathbb{R}^n$. In addition, the second term $-\cos^2\theta\|\widebar{\bm{S}}_0\tilde{\bm{z}}_0^\perp\|^2\le 0$, thus yielding $\sin^2\theta\|\widebar{\bm{S}}_0\bm{x}\|^2-\|\widebar{\bm{S}}_0\bm{x}^\perp\|^2\le 0$. 
For any nonzero $\bm{x}\in\mathbb{R}^n$, 
it holds that
\begin{equation}
	\sin^2\theta=1-\cos^2\theta \le \frac{\big\|\widebar{\bm{S}}_0\bm{x}^\perp\big\|^2}{\big\|\widebar{\bm{S}}_0\bm{x}\big\|^2}.
\end{equation}
Upon letting $\bm{u}=\bm{x}^\perp$, the last inequality taken together with \eqref{eq:ineq1} concludes the proof of \eqref{eq:mse}.

\subsection{Proof of Lemma \ref{lem:up}}\label{sec:proofup}

Assume $\|\bm{x}\|=1$.
Let $\bm{s}\in\mathbb{R}^n$ be sampled uniformly at random on the unit sphere, which has zero mean and covariance matrix $\bm{I}_n/n$. 
Let also $\bm{U}\in\mathbb{R}^{n\times n}$ be a unitary matrix such that $\bm{U}\bm{x}=\bm{e}_1$, where $\bm{e}_1$ is the first canonical vector in $\mathbb{R}^n$. It is then easy to verify that the following holds for any fixed threshold $0<\tau<1$ \cite{duchi2017}:
\allowdisplaybreaks
\begin{align}
\label{eqq:exp}
&\mathbb{E}[\bm{s}\bm{s}^\ccalT|(\bm{s}^\ccalT\bm{x})^2>\tau]\nonumber\\
&=\bm{U}\mathbb{E}[\bm{U}^\ccalT\bm{s}\bm{s}^\ccalT\bm{U}|(\bm{s}^\ccalT\bm{U}\bm{U}^\ccalT\bm{x})^2>\tau]\bm{U}^\ccalT\nonumber\\
&\buildrel(i)\over = \bm{U}\mathbb{E}[\tilde{\bm{s}}\tilde{\bm{s}}^\ccalT|(\tilde{\bm{s}}^\ccalT\bm{e}_1)^2>\tau]\bm{U}^\ccalT\nonumber\\
& = \bm{U}\mathbb{E}[\tilde{\bm{s}}\tilde{\bm{s}}^\ccalT|\tilde{s}_1^2>\tau]\bm{U}^\ccalT\nonumber\\
& = \bm{U}\left[\begin{array}{ll} \mathbb{E}[\tilde{s}_1^2|\tilde{s}_1^2>\tau] &\mathbb{E}[\tilde{s}_1\tilde{\bm{s}}_{\backslash1}^\ccalT|\tilde{s}_1^2>\tau]\\
\mathbb{E}[\tilde{s}_1\tilde{\bm{s}}_{\backslash1}|\tilde{s}_1^2>\tau]&\mathbb{E}[\tilde{\bm{s}}_{\backslash1}\tilde{\bm{s}}_{\backslash1}^\ccalT|\tilde{s}_1^2>\tau]
\end{array}\right]\bm{U}^\ccalT\nonumber\\
&\buildrel(ii)\over = \bm{U}\left[\begin{array}{ll} \mathbb{E}[\tilde{s}_1^2|\tilde{s}_1^2>\tau] &\bm{0}^\ccalT\\
\bm{0}&\mathbb{E}[\tilde{\bm{s}}_{\backslash1}\tilde{\bm{s}}_{\backslash1}^\ccalT|\tilde{s}_1^2>\tau]
\end{array}\right]\bm{U}^\ccalT\nonumber\\
&\buildrel(iii)\over = \mathbb{E}[\tilde{s}_2^2|\tilde{s}_1^2>\tau]
\bm{I}_n+\big( \mathbb{E}[\tilde{s}_1^2|\tilde{s}_1^2>\tau]- \mathbb{E}[\tilde{s}_2^2|\tilde{s}_1^2>\tau]\big)\bm{x}\bm{x}^\ccalT\nonumber\\
&\buildrel\triangle\over =C_1\bm{I}_n+C_2\bm{x}\bm{x}^\ccalT
\end{align}
with the constants $C_1:=\mathbb{E}[\tilde{s}_2^2|\tilde{s}_1^2>\tau]<\frac{1-\tau}{n-1}$, $C_2:= \mathbb{E}[\tilde{s}_1^2|\tilde{s}_1^2>\tau]- C_1>0$, and  $\bm{s}_{\backslash 1}\in\mathbb{R}^{n-1}$  denoting the subvector of $\bm{s}\in\mathbb{R}^n$ after removing the first entry from $\bm{s}$. Here, the result $(i)$ follows upon defining $\tilde{\bm{s}}:=\bm{U}^\ccalT\bm{s}$, which obeys the uniformly spherical distribution too using the rotational invariance. The equality $(ii)$ is due to the zero-mean and symmetrical properties of the uniformly spherical distribution. Finally, to derive $(iii)$, we have used the fact $\bm{x}=\bm{U}\bm{e}_1=\bm{u}_1$, the first column of $\bm{U}$, which arises from $\bm{U}^\ccalT\bm{x}=\bm{e}_1$ and $\bm{U}\bm{U}^\ccalT=\bm{I}_n$.

By the argument above, assume without loss of generality that $\bm{x}=\bm{e}_1$. Consider now the truncated vector $\bm{s}_{\backslash 1}|(\bm{s}^\ccalT\bm{x})^2>\tau$, or equivalently, $\bm{s}_{\backslash 1}|s_1^2>\tau$. 
It is then clear that $\bm{s}_{\backslash 1}|s_1^2>\tau$ is bounded, and thus subgaussian; furthermore, the next hold
\begin{subequations}
\begin{align}
\mathbb{E}[\bm{s}_{\backslash 1}|s_1^2>\tau]&=\bm{0}\\
\mathbb{E}\big[\big(\bm{s}_{\backslash 1}|s_1^2>\tau\big) \big(\bm{s}_{\backslash 1}|s_1^2>\tau\big)^\ccalT\big]&=C_1\bm{I}_{n-1}\label{eqq:cov}
\end{align}
\end{subequations}
where \eqref{eqq:cov} is obtained as a submatrix of the first term in \eqref{eqq:exp} since the second term $C_2\bm{e}_1\bm{e}_1^\ccalT$ is removed.

Considering a unit vector $\bm{x}^\perp$ such that 
$\bm{x}^\ccalT\bm{x}^\perp=\bm{e}_1^\ccalT\bm{x}^\perp=0$, there exists a unit vector $\bm{d}\in\mathbb{R}^{n-1}$ such that $\bm{x}^\perp=\left[0~\bm{d}^\ccalT\right]^\ccalT$. Thus, it holds that  
\begin{equation}\label{eq:su}
\big\|\widebar{\bm{S}}_0\bm{x}^\perp\big\|^2
=\Big\|\widebar{\bm{S}}_0\big[0~\bm{d}^\ccalT\big]^\ccalT\Big\|^2=\big\|\widebar{\bm{S}}_{0, \backslash 1}\bm{d}\big\|^2
\end{equation}
where $\widebar{\bm{S}}_{0, \backslash 1}\in\mathbb{R}^{|\widebar{\mathcal{I}}_0|\times (n-1)}$ is obtained through deleting the first column in $\widebar{\bm{S}}_0$, which is denoted by $\widebar{\bm{S}}_{0,1}$; that is, $\widebar{\bm{S}}_0=\big[\widebar{\bm{S}}_{0,1}~\widebar{\bm{S}}_{0, \backslash 1}\big]$.

The rows of $\widebar{\bm{S}}_{0, \backslash 1}$ may therefore be viewed as independent realizations of the conditional random vector $\bm{s}_{\backslash 1}^\ccalT|s_1^2>\tau$, with the threshold $\tau$ being the  $|\widebar{\mathcal{I}}_0|$-largest value in $\{y_i/\|\bm{a}_i\|^2\}_{i=1}^m$. 
Standard concentration inequalities on the sum of random positive semi-definite matrices composed of independent non-isotropic subgaussian rows~\cite[Remark 5.40]{chap2010vershynin} confirm that 
\begin{equation}\label{eq:iso}
\left\|\tfrac{1}{|\widebar{\mathcal{I}}_0|}\widebar{\bm{S}}_{0, \backslash 1}^\ccalT\widebar{\bm{S}}_{0, \backslash 1}-C_1\bm{I}_{n-1}\right\|\le \sigma C_1 \le  \frac{(1-\tau)\sigma}{n-1}
\end{equation}
holds with probability at least $1-2{\rm e}^{-c_K n}$ as long as $|\widebar{\mathcal{I}}_0|/n$ is sufficiently large, where
$\sigma$ is a numerical constant that can take arbitrarily small values, and $c_K>0$ is a universal constant. 
Without loss of generality, let us work with $\sigma:=0.005$ in~\eqref{eq:iso}.
Then for any unit vector $\bm{d}\in\mathbb{R}^{n-1}$, the following inequality holds with probability at least $1-2{\rm e}^{-c_K n}$:
\begin{equation}
\left|\tfrac{1}{|\widebar{\mathcal{I}}_0|}\bm{d}^\ccalT\widebar{\bm{S}}_{0,\backslash 1}^\ccalT\widebar{\bm{S}}_{0,\backslash 1}\bm{d}-C_1\right|\le  \frac{0.01}{n}
\end{equation}
for $n\ge 3$.
Therefore, one readily concludes that
\begin{equation}\label{eq:up}
\big\|\widebar{\bm{S}}_0\bm{x}^\perp\big\|^2=\left|(\bm{x}^\perp)^\ccalT\bm{S}^\ccalT\bm{S}\bm{x}^\perp\right|\le {1.01|\widebar{\mathcal{I}}_0|}\big/{n}
\end{equation}
holds with probability at least $1-2{\rm e}^{-c_K n}$, provided that $|\widebar{\mathcal{I}}_0|\big/n$ exceeds some constant. Note that $c_K$ depends on the maximum subgaussian norm of rows of $\bm{S}$, and we assume without loss of generality $c_K\ge 1/2$. Hence, $\|\widebar{\bm{S}}_0\bm{u}\|^2$ in~\eqref{eq:mse}
is upper bounded simply by letting $\bm{u}=\bm{x}^\perp$ in~\eqref{eq:up}.

\subsection{Proof of Lemma \ref{lem:low}}\label{sec:prooflow}

We next pursue a meaningful lower bound for $\|\widebar{\bm{S}}_0\bm{x}\|^2$ in~\eqref{eq:low0}. When $\bm{x}=\bm{e}_1$, one has $\|\widebar{\bm{S}}_0\bm{x}\|^2=\|\widebar{\bm{S}}_0\bm{e}_1\|^2=\sum_{i=1}^{|\widebar{\mathcal{I}}_0|}\bar{s}_{i,1}^2$, where $\{\bar{s}_{i,1}\}_{i=1}^{|\widebar{\mathcal{I}}_0|}$ are entries of the first column of $\widebar{\bm{S}}_0$.
It is further worth mentioning that all squared entries of any spherical random vector obey the \emph{Beta} distribution with parameters $\alpha=\frac{1}{2}$, and $\beta=\frac{n-1}{2}$, i.e., $\bar{s}^2_{i,j}\sim {\rm Beta}\!\left(\frac{1}{2},\,\frac{n-1}{2}\right)$ for all $i,\,j$,~\cite[Lemma 2]{1981ecd}. Although they have closed-form probability density functions (pdfs) that may facilitate deriving a lower bound, 
we take another route detailed as follows. A simple yet useful inequality is established first.

\begin{lemma}
	\label{lem:max}
	Given $m$ fractions obeying $1>\frac{p_1}{q_1}\ge \frac{p_2}{q_2}\ge \cdots\ge\frac{p_m}{q_m}>0$, in which $p_i,\,q_i>0$, $\forall i\in[m]$, the following holds for all $1\le k\le m$
	\begin{equation}\label{eq:ineq}
	\sum_{i=1}^{k}\frac{p_{i}}{q_{i}}\ge \sum_{i=1}^k\frac{p_{[i]}}{q_{[1]}}
	\end{equation} 
	where $p_{[i]}$ denotes the $i$-th largest one among $\{p_i\}_{i=1}^m$, and hence, $q_{[1]}$ is the maximum in $\{q_i\}_{i=1}^m$.  
\end{lemma} 

\begin{proof}
	%
	For any $k\in[m]$, according to the definition of $q_{[i]}$, it holds that $p_{[1]}\ge p_{[2]}\ge \cdots\ge p_{[k]}$, so $\frac{p_{[1]}}{q_{[1]}}\ge \frac{p_{[2]}}{q_{[1]}}\ge \cdots\ge \frac{p_{[k]}}{q_{[1]}}$. Considering $q_{[1]}\ge q_{i}$, $\forall i\in [m]$, and letting $j_i\in[m]$ be the index such that $p_{j_{i}}=p_{[i]}$,  
	then $\frac{p_{j_i}}{q_{j_i}}=\frac{p_{[i]}}{q_{j_i}}\ge \frac{p_{[i]}}{q_{[1]}}$ holds for any $i\in[k]$. Therefore, 
	$\sum_{i=1}^k\frac{p_{j_i}}{q_{j_i}}=
	\sum_{i=1}^k\frac{p_{[i]}}{q_{j_i}}\ge \sum_{i=1}^k\frac{p_{[i]}}{q_{[1]}}$. Note that $\left\{\frac{p_{[i]}}{q_{j_i}}\right\}_{i=1}^k$ comprise a subset of terms in $\left\{\frac{p_i}{q_i}\right\}_{i=1}^m$. On the other hand, according to our assumption, $\sum_{i=1}^k\frac{p_i}{q_i}$ is the largest among all sums of $k$ summands; hence, $\sum_{i=1}^k\frac{p_i}{q_i}\ge \sum_{i=1}^k\frac{p_{[i]}}{q_{j_i}}$ yields
	$\sum_{i=1}^k\frac{p_i}{q_i}\ge \sum_{i=1}^k\frac{p_{[i]}}{q_{[1]}}$ concluding the proof.   
\end{proof}

Without loss of generality and for simplicity of exposition, let us assume that indices of $\bm{a}_i$'s have been re-ordered such that
\begin{equation}\label{eq:allai}
	\frac{a_{1,1}^2}{\left\|\bm{a}_{1}\right\|^2}\ge \frac{a_{2,1}^2}{\left\|\bm{a}_{2}\right\|^2}\ge \cdots \ge \frac{a_{m,1}^2}{\left\|\bm{a}_{m}\right\|^2},
\end{equation}
where $a_{i,1}$ denotes the first element of $\bm{a}_i$.
Therefore, writing
$	\|\widebar{\bm{S}}_0\bm{e}_1\|^2=\sum_{i=1}^{|\widebar{\mathcal{I}}_0|}a_{i,1}^2/\|\bm{a}_{i}\|^2
$, the next task  
amounts to finding the sum of the $|\widebar{\mathcal{I}}_0|$ largest out of all $m$ entities in \eqref{eq:allai}.
Applying the result \eqref{eq:ineq} in Lemma \ref{lem:max} gives
\begin{equation}\label{eq:mid}
	\sum_{i=1}^{|\widebar{\mathcal{I}}_0|}\frac{a_{i,1}^2}{\left\|\bm{a}_{i}\right\|^2}\ge \sum_{i=1}^{|\widebar{\mathcal{I}}_0|}\frac{a_{[i],1}^2}{\max_{i\in[m]}\left\|\bm{a}_i\right\|^2},
\end{equation}  
in which $a_{[i],1}^2$ stands for the $i$-th largest entity in $\left\{a^2_{i,1}\right\}_{i=1}^m$. 

Observe that for i.i.d. random vectors $\bm{a}_i\sim\mathcal{N}\big(\bm{0},\bm{I}_n\big)$, the property $\mathbb{P}(\left\|\bm{a}_i\right\|^2\ge 2.3n)\le {\rm e}^{-n/2}$ holds for large enough $n$ (e.g., $n\ge 20$),
which can be understood upon substituting $\xi:=n/2$ into the following standard result~\cite[Lemma 1]{chisquaretail}
\begin{equation}
\mathbb{P}\left(\left\|\bm{a}_i\right\|^2-n\ge 2\sqrt{\xi}+2\xi \right)\le {\rm e}^{-\xi}.
\end{equation}  
In addition, one readily concludes that  
$\mathbb{P}\left(\max_{i\in[m]}\left\|\bm{a}_i\right\|\le \sqrt{2.3n}\right)\ge 1-m{\rm e}^{-n/2}$. We will henceforth build our subsequent proofs on this event without stating this explicitly each time encountering it. Therefore, \eqref{eq:mid} can be lower bounded by
\begin{align}
			\big\|\widebar{\bm{S}}\bm{x}\big\|^2=
			\sum_{i=1}^{|\widebar{\mathcal{I}}_0|}\frac{a_{i,1}^2}{\left\|\bm{a}_{i}\right\|^2}
	&\ge \sum_{i=1}^{|\widebar{\mathcal{I}}_0|}\frac{a_{[i],1}^2}{\max_{i\in[m]}\left\|\bm{a}_i\right\|^2} \frac{1}{2.3n}\sum_{i=1}^{|\widebar{\mathcal{I}}_0|}{\left|a_{[i],1}\right|^2}\label{eq:down}
\end{align}
which holds with probability at least $1-m{\rm e}^{-n/2}$. 
The task left for bounding $\|\widebar{\bm{S}}\bm{x}\|^2$
is to derive a meaningful lower bound for $\sum_{i=1}^{|\widebar{\mathcal{I}}_0|}{a_{[i],1}^2}$. 
Roughly speaking, because the ratio ${|\widebar{\mathcal{I}}_0|}/{m}$ is small, e.g., ${|\widebar{\mathcal{I}}_0|}/{m}\le 1/5$, a trivial result consists of bounding $(1/|\widebar{\mathcal{I}}_0|) \sum_{i=1}^{|\widebar{\mathcal{I}}_0|}{a_{[i],1}^2}$ by its sample average 
$(1/m) \sum_{i=1}^{m}{a_{[i],1}^2}$.
The latter can be bounded using its ensemble mean, i.e., $\mathbb{E}[a_{i,1}^2]=1$, $\forall i\in[\widebar{\mathcal{I}}_0]$, to yield $(1/m) \sum_{i=1}^{m}{a_{[i],1}^2}\ge (1-\epsilon)\mathbb{E}[a_{i,1}^2]=1-\epsilon$, which holds with high probability for some numerical constant $\epsilon>0$~\cite[Lemma 3.1]{phaselift}. Therefore, one has a candidate lower bound $\sum_{i=1}^{|\widebar{\mathcal{I}}_0|}{a_{[i],1}^2}\ge (1-\epsilon) |\widebar{\mathcal{I}}_0|$. 
Nonetheless, this lower bound is in general too loose, and it contributes to a relatively large upper bound on the wanted term in \eqref{eq:mse}.

To obtain an alternative bound, let us examine first the typical size of the maximum in $\left\{a_{i,1}^2\right\}_{i=1}^m$. Observe obviously that the modulus $\left|a_{i,1}\right|$ follows the half-normal distribution having the pdf $p(r)=\sqrt{{2}/{\pi}}\cdot {\rm e}^{-{r^2}/{2}}$, $r> 0$, and it is easy to verify that 
\begin{equation}\label{eq:halfnormalmean}
\mathbb{E}[|a_{i,1}|]=\sqrt{2/\pi}.
\end{equation}
Then integrating the pdf from $0$ to $+\infty$ yields the corresponding accumulative distribution function~(cdf) expressible in terms of the error function 
$	\mathbb{P}\left(\left|a_{i,1}\right|>\xi\right)=1-{\rm erf}\left({\xi}/{2}\right)$, i.e., ${\rm erf}\left(\xi\right):={2}/{\sqrt{\pi}}\cdot \int_{0}^\xi{\rm e}^{-r^2}{\rm d}r$. Appealing to a lower bound on the complimentary error function ${\rm erfc}\left(\xi\right):=1-{\rm erf}\left(\xi\right)$ from~\cite[Theorem 2]{erf2011}, one establishes that 
 $\mathbb{P}\left(\left|a_{i,1}\right|>\xi\right)=1-{\rm erf}\left({\xi}/{2}\right)\ge (3/5){\rm e}^{-{\xi^2}/{2}}
$. 
 Additionally, direct application of probability theory and Taylor expansion confirms that
\begin{align}\label{eq:highp}
\mathbb{P}\big(\max_{i\in[m]}\left|a_{i,1}\right|\ge \xi\big)&=1-\left[\mathbb{P}\left(\left|a_{i,1}\right|\le \xi\right)\right]^m\nonumber\\
&\ge 1-\left(1-0.6{\rm e}^{-{\xi^2}/{2}}\right)^m\nonumber\\
&\ge 1-{\rm e}^{-{0.6m}{\rm e}^{-\xi^2/2}}.
\end{align}
Choosing now $\xi:=\sqrt{2\log n}$ leads to 
\begin{equation}
	\mathbb{P}\big(\max_{i\in[m]}\left|a_{i,1}\right|\ge \sqrt{2\log n}\big)\ge 1-{\rm e}^{-0.6m/n}\ge 1-o(1)
	\label{eq:logn}	
\end{equation}
which holds with the proviso that $m/n$ is large enough, and the symbol $o(1)$ represents a small constant probability. 
Thus, provided that $m/n$ exceeds some large constant, the event $\max_{i\in [m]} a_{i,1}^2\ge 2\log n$ occurs with high probability. Hence, one may expect a tighter lower bound than $(1-\epsilon_0)|\widebar{\mathcal{I}}_0|$, which is on the same order of $m$ under the assumption that $|\widebar{\mathcal{I}}_0|/m$ is about a constant.


Although $a_{i,1}^2$ obeys the \emph{Chi-square} distribution with $k=1$ degrees of freedom, its cdf is rather complicated and does not admit a nice closed-form expression. A small trick is hence taken in the sequel. Assume without loss of generality that both $m$ and $|\widebar{\mathcal{I}}_0|$ are even. Grouping two consecutive $a_{[i],1}^2$'s together, introduce a new variable $\vartheta{[i]}:=a_{[2k-1],1}^2+a_{[2k],1}^2$, $\forall k\in[{m}/{2}]$, 
hence yielding a sequence of ordered numbers, i.e., $\vartheta_{[1]}\ge \vartheta_{[2]}\ge \cdots \ge \vartheta_{[m/2]}>0
$. Then, one can equivalently write the wanted sum as 
\begin{equation}
\sum_{i=1}^{|\widebar{\mathcal{I}}_0|}a_{[i],1}^2=\sum_{i=1}^{|\widebar{\mathcal{I}}_0|/2}\vartheta_{[i]}. 
\label{eq:rsum}	
\end{equation}

On the other hand, for i.i.d. standard normal random variables $\left\{a_{i,1}\right\}_{i=1}^m$, let us consider grouping randomly two of them and denote the corresponding sum of their squares by $\chi_{k}:=a_{k_i,1}^2+a_{k_j,1}^2$, where $k_i\ne k_j\in[m]$, and $k\in[m/2]$. It is self-evident that the $\chi_k$'s are identically distributed obeying the \emph{Chi-square} distribution with $k=2$ degrees of freedom, having the pdf 
\begin{equation}\label{eq:chipdf}
p\left(r\right)=\frac{1}{2}{\rm e}^{-\frac{r}{2}},\quad r\ge 0,
\end{equation}
and the following complementary cdf (ccdf)
\begin{equation}\label{eq:chiccdf}
\mathbb{P}\!\left(\chi_k\ge \xi\right):=\int_{\xi}^{\infty}\frac{1}{2}{\rm e}^{-\frac{r}{2}}{\rm d}r={\rm e}^{-\frac{\xi}{2}},\quad\forall \xi\ge 0.
\end{equation}
Ordering all $\chi_k$'s, summing the $|\widebar{\mathcal{I}}_0|/2$ largest ones, and comparing the resultant sum with the one in~\eqref{eq:rsum} confirms that 
\begin{equation}\label{eq:relation}
\sum_{i=1}^{|\widebar{\mathcal{I}}_0|/2}\chi_{[i]}\le \sum_{i=1}^{|\widebar{\mathcal{I}}_0|/2}\vartheta_{[i]}=\sum_{i=1}^{|\widebar{\mathcal{I}}_0|}a_{[i],1}^2,\quad \forall |\widebar{\mathcal{I}}_0|\in [m].
\end{equation} 

Upon setting $\mathbb{P}\!\left(\chi_k\ge \xi\right)={|\widebar{\mathcal{I}}_0|}/{m}$, one obtains an estimate of $\chi_{|\widebar{\mathcal{I}}_0|/2}$, the $(|\widebar{\mathcal{I}}_0|/2)$-th largest value in $\left\{\chi_k\right\}_{k=1}^{m/2}$ as follows
\begin{equation}
\hat{\chi}_{|\widebar{\mathcal{I}}_0|/2}:=2\log\big(m\big/|\widebar{\mathcal{I}}_0|\big).
\end{equation}
Furthermore, applying the Hoeffding-type inequality~\cite[Proposition 5.10]{chap2010vershynin} and leveraging the convexity of the ccdf in \eqref{eq:chiccdf}, one readily establishes that
\begin{equation}\label{eq:hoeffding1}
\mathbb{P}\!\left(\hat{\chi}_{|\widebar{\mathcal{I}}_0|/2}-\chi_{|\widebar{\mathcal{I}}_0|/2}>\xi \right)\le {\rm e}^{-\frac{1}{4}m\xi^2{\rm e}^{-\xi}(|\widebar{\mathcal{I}}_0|/m)^2},\quad\forall \xi>0.
\end{equation}
Taking without loss of generality $\xi:=0.05 \hat{\chi}_{|\widebar{\mathcal{I}}_0|/2}=0.1\log\big(m\big/|\widebar{\mathcal{I}}_0|\big)$ gives
\begin{equation}\label{eq:hoeffding}
\mathbb{P}\!\left(\chi_{|\widebar{\mathcal{I}}_0|/2}<0.95\hat{\chi}_{|\widebar{\mathcal{I}}_0|/2} \right)\le {\rm e}^{-c_{0}m
}
\end{equation}
for some universal constants $c_{0},\, c_{\chi}>0$, and sufficiently large $n$ such that ${|\widebar{\mathcal{I}}_0|}/{m}\gtrsim c_{\chi}>0$. 
The remaining part in this section assumes that this event 
 occurs.
 
 
Choosing $\xi:=4\log n$ and substituting this into the ccdf in \eqref{eq:chiccdf} leads to
\begin{equation}\label{eq:chibound}
	\mathbb{P}\left(\chi\le 4\log n \right)= 
	1-1/n^2.
	\end{equation}
Notice that each summand in $\sum_{i=1}^{|\widebar{\mathcal{I}}_0|/2}\chi_{[i]}\ge \sum_{i=1}^{m/2}\chi_i\mathbb{1}_{\tilde{\mathcal{E}}_i}$ is Chi-square distributed, and hence could be unbounded, so we choose to work with the truncation $\sum_{i=1}^{m/2}\chi_i\mathbb{1}_{\tilde{\mathcal{E}}_i}$, where the $\mathbb{1}_{\tilde{\mathcal{E}}_i}$'s are independent copies of $\mathbb{1}_{\tilde{\mathcal{E}}}$, and $\mathbb{1}_{\tilde{\mathcal{E}}}$
denotes the indicator function for the ensuing events
\begin{equation}\label{eq:event}
	\tilde{\mathcal{E}}:=\left\{\chi\ge \hat{\chi}_{|\widebar{\mathcal{I}}_0|/2}\right\}\cap\left\{\chi\le 4\log n \right\}.
\end{equation}
Apparently, it holds that $\sum_{i=1}^{|\widebar{\mathcal{I}}_0|/2}\chi_{[i]}\ge \sum_{i=1}^{m/2}\chi_i\mathbb{1}_{\tilde{\mathcal{E}}_i}$.
One further establishes that 
\begin{align}\label{eq:tmean}
		\mathbb{E}\left[\chi_i\mathbb{1}_{\tilde{\mathcal{E}}_i}\right]:\!&=\int_{\hat{\chi}_{|\widebar{\mathcal{I}}_0|/2}}^{4\log n} \frac{1}{2}r{\rm e}^{-r/2}{\rm d}r\nonumber\\
		&=\left( \hat{\chi}_{|\widebar{\mathcal{I}}_0|/2}\!+2\right){\rm e}^{-{\hat{\chi}_{|\widebar{\mathcal{I}}_0|/2}}/{2}} \!-\left(4\log n+2\right){\rm e}^{-2\log n}\nonumber\\
		&=\frac{2|\widebar{\mathcal{I}}_0|}{m}\left[1+\log \big(m\big/|\widebar{\mathcal{I}}_0\big)\right]-\frac{\left(4\log n + 2\right)}{n^2}.
\end{align}

The task of bounding $\sum_{i=1}^{|\widebar{\mathcal{I}}_0|}a_{[i],1}^2$ in~\eqref{eq:relation} 
now boils down to bounding $\sum_{i=1}^{m/2}\chi_i\mathbb{1}_{\tilde{\mathcal{E}}_i}$ from its expectation in \eqref{eq:tmean}.    
A convenient way to accomplish this is using the Bernstein inequality~\cite[Proposition 5.16]{chap2010vershynin}, that deals with bounded random variables. That also justifies introducing the upper-bound truncation on $\chi$ in \eqref{eq:event}. Specifically, define
\begin{equation}
	\label{eq:zeta}
	\vartheta_i:=\chi_i\mathbb{1}_{\tilde{\mathcal{E}}_i}-\mathbb{E}\left[\chi_i\mathbb{1}_{\tilde{\mathcal{E}}_i}\right],\quad 1\le i\le m/2.
\end{equation}
Thus, $\left\{\vartheta_i\right\}_{i=1}^{m/2}$ are i.i.d. centered and bounded random variables following from the mean-subtraction and the upper-bound truncation. Further, according to the ccdf~\eqref{eq:chiccdf} and the definition of sub-exponential random variables~\cite[Definition 5.13]{chap2010vershynin}, the terms $\left\{\vartheta_i\right\}_{i=1}^{m/2}$ are sub-exponential. 
Then, the following 
\begin{equation}
	\label{eq:meanbound}
	\Big|\sum_{i=1}^{m/2}\vartheta_i\Big|\ge \tau
\end{equation}
holds with probability at least $1-2{\rm e}^{-c_s\min\left({\tau}/{K_s},{\tau^2}/{K_s^2}\right)}$, in which $c_s>0$ is a universal constant, and $K_s:=\max_{i\in [m/2]}\|\vartheta_i\|_{\psi_1}$ represents the maximum subexponential norm of the $\vartheta_i$'s. 

Indeed, $K_s$ can be found as follows~\cite[Definition 5.13]{chap2010vershynin}:
\begin{align}\label{eq:subnorm}
	K_s:\!&=\sup_{p\ge 1}p^{-1}\left(\mathbb{E}\left[|\vartheta_i|^p\right]\right)^{1/p}\nonumber\\
	&\le \left(4\log n-2\log\big(m\big/|\widebar{\mathcal{I}}_0|\big)\right) \left[|\widebar{\mathcal{I}}_0|\big/m-1/n^2\right]\nonumber\\
	&\le \frac{2|\widebar{\mathcal{I}}_0|}{m}\log\left(n^2|\widebar{\mathcal{I}}_0|\big/m\right)\nonumber\\
	&\le \frac{4|\widebar{\mathcal{I}}_0|}{m}\log n.
\end{align}
Choosing $\tau:=8|\widebar{\mathcal{I}}_0|/(c_sm)\cdot\log^2 n$ in \eqref{eq:meanbound} yields
\begin{align}
	\label{eq:zetabound}
	\sum_{i=1}^{m/2}\chi_i\mathbb{1}_{\tilde{\mathcal{E}}_i}&\ge |\widebar{\mathcal{I}}_0|\left[1+\log \big(m\big/|\widebar{\mathcal{I}}_0|\big)\right]-8|\widebar{\mathcal{I}}_0|/(c_sm)\cdot\log^2 n
	-{m\left(2\log n + 1\right)}/{n^2}\nonumber\\
	&\ge (1-\epsilon_s)|\widebar{\mathcal{I}}_0|\left[1+\log \big(m\big/|\widebar{\mathcal{I}}_0|\big)\right]
\end{align}
for some small constant $\epsilon_s >0$, which 
holds with probability at least $1-m{\rm e}^{-n/2}-{\rm e}^{-c_{0}m}-1/n^2$ as long as $m/n$ exceeds some numerical constant and $n$ is sufficiently large.
Therefore, combining~\eqref{eq:down},~\eqref{eq:relation}, and~\eqref{eq:zetabound}, 
one concludes that the following holds with high probability
\begin{equation}\label{eq:low}
\big\|\widebar{\bm{S}}_0\bm{x}\big\|^2=\sum_{i=1}^{|\widebar{\mathcal{I}}_0|}\frac{a_{i,1}^2}{\left\|\bm{a}_{i}\right\|^2}\ge (1-\epsilon_s)\frac{|\widebar{\mathcal{I}}_0|}{2.3n}\left[1+\log \big(m\big/|\widebar{\mathcal{I}}_0|\big)\right].
\end{equation}
Taking $\epsilon_s:=0.01$ without loss of generality concludes the proof of Lemma~\ref{lem:low}.  


\subsection{Proof of Lemma~\ref{le:1stterm}}\label{proof:1stterm}
%

Let us first prove the argument for a fixed pair $\bm{h}$ and $\bm{x}$, such that $\bm{h}$ and $\bm{z}$ are independent of $\{\bm{a}_i\}_{i=1}^m$, and then apply a covering argument. 
To start, introduce a Lipschitz-continuous counterpart for the discontinuous indicator function~\cite[A.2]{twf}
    \begin{equation}
    	\chi_E(\theta):=\left\{\begin{array}
    		{lll}
    		1,&|\theta|\ge \frac{\sqrt{1.01}}{1+\gamma},\\
    		{100(1+\gamma)^2\theta^2-100},&\frac{1}{1+\gamma}\le 
    		|\theta|<\frac{\sqrt{1.01}}{1+\gamma},\\
    		0,&|\theta|<\frac{1}{1+\gamma}
    	\end{array}\right.
    \end{equation}
    with Lipschitz constant $\mathcal{O}(1)$. 
Recall that $\mathcal{E}_i=\left\{\left|\frac{\bm{a}_i^\ccalT\bm{z}}{\bm{a}_i^\ccalT\bm{x}}\right|\ge \frac{1}{1+\gamma}\right\}$, so it holds that $0\le \chi_E\left(\left|\frac{\bm{a}_i^\ccalT\bm{z}}{\bm{a}_i^\ccalT\bm{x}}\right|\right)\le \mathbb{1}_{\mathcal{E}_i}$ for any $\bm{x}\in\mathbb{R}^n$ and $\bm{h}\in\mathbb{R}^n$, thus yielding
\begin{align}
	\frac{1}{m}\sum_{i=1}^m\left(\bm{a}_i^\ccalT\bm{h}\right)^2\mathbb{1}_{\mathcal{E}_i}&\ge \frac{1}{m}\sum_{i=1}^m\left(\bm{a}_i^\ccalT\bm{h}\right)^2\chi_E\left(\left|\frac{\bm{a}_i^\ccalT\bm{z}}{\bm{a}_i^\ccalT\bm{x}}\right|\right)
	=\frac{1}{m}\sum_{i=1}^m\left(\bm{a}_i^\ccalT\bm{h}\right)^2\chi_E\left(\left|1+\frac{\bm{a}_i^\ccalT\bm{h}}{\bm{a}_i^\ccalT\bm{x}}\right|\right)
	.\label{eq:1term2}
\end{align}

By homogeneity and rotational invariance of normal distributions, it suffices to prove the case where $\bm{x}=\bm{e}_1$ and $\|\bm{h}\|/\|\bm{x}\|=\|\bm{h}\|\le \rho$. 
According to~\eqref{eq:1term2}, lower bounding the first term in~\eqref{eq:target} can be achieved by lower bounding $\sum_{i=1}^m(\bm{a}_i^\ccalT\bm{h})^2\chi_E\left( \left|1+\frac{\bm{a}_i^\ccalT\bm{h}}{\bm{a}_i^\ccalT\bm{x}}\right|\right)$ instead. To that end, let us find the mean of $\left(\bm{a}_i^\ccalT\bm{h}\right)^2\chi_E\left( \left|1+\frac{\bm{a}_i^\ccalT\bm{h}}{\bm{a}_i^\ccalT\bm{x}}\right|\right)$. Note that $\left(\bm{a}_i^\ccalT\bm{h}\right)^2$ and $\chi_E\left( \left|1+\frac{\bm{a}_i^\ccalT\bm{h}}{\bm{a}_i^\ccalT\bm{x}}\right|\right)$ are dependent. Introduce an orthonormal matrix $\bm{U}_{\bm{h}}$ that contains $\bm{h}^\ccalT/\|\bm{h}\|$ as its first row, i.e.,
\begin{equation}
	\bm{U}_{\bm{h}}:=\left[\begin{array}
		{c}
		\bm{h}^\ccalT/\|\bm{h}\|\\
		\widetilde{\bm{U}}_{\bm{h}}
	\end{array}
	\right]
\end{equation} 
for some orthogonal matrix $\widetilde{\bm{U}}_{\bm{h}}\in\mathbb{R}^{(n-1)\times n}$ such that $\bm{U}_{\bm{h}}$ is orthonormal. Moreover, define $\tilde{\bm{h}}:=\bm{U}_{\bm{h}}\bm{h}$, and $\tilde{\bm{a}}_i:=\bm{U}_{\bm{h}}\bm{a}_i$; and let $\tilde{a}_{i,1}$ and $\tilde{\bm{a}}_{i,\backslash 1}$ denote the first entry and the remaining entries in the vector $\tilde{\bm{a}}_i$; likewise for $\tilde{\bm{h}}$. Then, for any $\bm{h}$ such that $\|\bm{h}\|\le \rho$, we have
\begin{align}
	&\mathbb{E}\left[(\bm{a}_i^\ccalT\bm{h})^2\chi_E\left( \left|1+\frac{\bm{a}_i^\ccalT\bm{h}}{\bm{a}_i^\ccalT\bm{x}}\right|\right)\right]\nonumber\\
&=\mathbb{E}\!\left[(\tilde{a}_{i,1}\tilde{h}_1)^2\chi_E\!\left(\left|1\!+\!
\frac{\bm{a}_{i}^\ccalT\bm{h}}{\bm{a}_i^\ccalT\bm{x}}\right|\right)\right]\!+\!\mathbb{E}\!\left[(\tilde{\bm{a}}_{i,\backslash 1}^\ccalT\tilde{\bm{h}}_{\backslash 1})^2\chi_E\!\left(\left|1\!+\!\frac{\bm{a}_i^\ccalT\bm{h}}{\bm{a}_i^\ccalT\bm{x}}\right|\right)\right]\nonumber\\
	&=\tilde{h}_1^2\,\mathbb{E}\left[\tilde{a}_{i,1}^2\,\chi_E\left( \left|1+\frac{\bm{a}_i^\ccalT\bm{h}}{a_{i,1}}\right|\right)\right]
	+\mathbb{E}\left[(\tilde{\bm{a}}_{i,\backslash 1}^\ccalT\tilde{\bm{h}}_{\backslash 1})^2\right]\mathbb{E}\!\left[\chi_E\left( \left|1+
	\frac{\bm{a}_i^\ccalT\bm{h}}{a_{i,1}}
	\right|\right)\right]
	\nonumber\\
	&=\tilde{h}_1^2\mathbb{E}\!\left[\tilde{a}_{i,1}^2\chi_E\!\left( \left|1+\frac{\bm{a}_i^\ccalT\bm{h}}{a_{i,1}}\right|\right)\right]+\big\|\tilde{\bm{h}}_{\backslash 1}\big\|^2
	\mathbb{E}\!\left[\chi_E\!\left( \left|1\!+\!\frac{\bm{a}_i^\ccalT\bm{h}}{a_{i,1}}\right|\right)\right]\nonumber\\
	&\ge \left(\tilde{h}_1^2\!+\!\|\tilde{\bm{h}}_{\backslash 1}\|^2\right)\min\left\{\mathbb{E}\left[{a}_{i,1}^2\chi_E\left( \left|1+h_1+
	\frac{\bm{a}_{i,\backslash 1}^\ccalT\bm{h}_{\backslash 1}}{a_{i,1}}\right|\right)\right],\mathbb{E}\left[\chi_E\left(\left|1+h_1+\frac{\bm{a}_{i,\backslash 1}^\ccalT\bm{h}_{\backslash 1}}{a_{i,1}}\right|\right)\right]
	\right\}\nonumber\\
	&\ge \|\bm{h}\|^2\min\bigg\{\mathbb{E}\left[a_{i,1}^2\chi_E\left(\left|1-\rho+\frac{a_{i,2}}{a_{i,1}}\rho\right|
	\right)
	\right],\mathbb{E}\left[\chi_E\left(1-\rho+\frac{a_{i,2}}{a_{i,1}}\rho
	\bigg)
	\right]
	\right\}
	\nonumber\\
	&= (1-\zeta_1)\|\bm{h}\|^2\label{eq:1bound}
\end{align}
where the second equality follows from the independence between $\tilde{\bm{a}}_{i,\backslash 1}^\ccalT\tilde{\bm{h}}_{\backslash 1}$ and $\bm{a}_i^\ccalT\bm{h}$, the second inequality holds for $\rho\le 1/10$ and $\gamma\ge 1/2$, 
 and the last equality comes from the definition of $\zeta_1 $ in~\eqref{eq:zeta}. Notice that $\varrho:=(\bm{a}_i^\ccalT\bm{h})^2\chi_E\left( \left|1+\frac{\bm{a}_i^\ccalT\bm{h}}{\bm{a_i}^\ccalT\bm{x}}\right|\right)\le (\bm{a}_i^\ccalT\bm{h})^2\eqdef \|\bm{h}\|^2a_{i,1}^2$ is a subexponential variable, and thus its subexponential norm $\|\varrho\|_{\psi_1}:=\sup_{p\ge 1}\left[\mathbb{E}(|\varrho|^p)\right]^{1/p}$ is finite. 
 
 Direct application of the Berstein-type inequality~\cite[Proposition 5.16]{chap2010vershynin} confirms that for any $\epsilon>0$, the following
 \begin{align}
 \frac{1}{m}\sum_{i=1}^m\left(\bm{a}_i^\ccalT\bm{h}\right)^2\chi_E\left(\left|1+\frac{\bm{a}_i^\ccalT\bm{h}}{\bm{a}_i^\ccalT\bm{x}}\right|\right)
 &\ge \mathbb{E}\left[\left(\bm{a}_i^\ccalT\bm{h}\right)^2\chi_E\left(\left|1+\frac{\bm{a}_i^\ccalT\bm{h}}{\bm{a}_i^\ccalT\bm{x}}\right|\right)\right]-\epsilon\|\bm{h}\|^2
 \nonumber\\
&\ge   \left(1-\zeta_1-\epsilon\right)\|\bm{h}\|^2 	\label{eq:1sttermbound}
  \end{align} 
holds with probability at least $1-{\rm e}^{-c_5m\epsilon^2}$ for some numerical constant $c_5>0$ provided that $\epsilon\le \|\varrho\|_{\psi_1}$ by assumption. 

To obtain uniform control over all vectors $\bm{z}$ and $\bm{x}$ such that $\|\bm{z}-\bm{x}\|\le \rho$, 
the net covering argument is applied~\cite[Definition 5.1]{chap2010vershynin}. 
Let $\mathcal{S}_\epsilon$ be an $\epsilon$-net of the unit sphere, $\mathcal{L}_\epsilon$ be an $\epsilon$-net of $[0,\,\rho]$, and define
\begin{equation}
	\mathcal{N}_{\epsilon}:=\left\{\left(\bm{z},\,\bm{h},\,t\right):\left(\bm{z}_0,\,\bm{h}_0,\,t_0\right)\in\mathcal{S}_\epsilon\times\mathcal{S}_\epsilon\times\mathcal{L}_\epsilon
	\right\}.
\end{equation} 
Since the cardinality $\left|\mathcal{S}_\epsilon\right|\le \left(1+2/\epsilon\right)^{n}$~\cite[Lemma 5.2]{chap2010vershynin}, then
\begin{equation}
	\left|\mathcal{N}_\epsilon\right|\le \left(
	1+2/\epsilon\right)^{2n}\rho/\epsilon\le \left(1+2/\epsilon\right)^{2n+1}
\end{equation}
due to the fact that $\rho/\epsilon< 2/\epsilon< 1+2/\epsilon$ for $0< \rho< 1$.

Consider now any $\left(\bm{z},\,\bm{h},\,t\right)$ obeying $\|\bm{h}\|=t\le \rho$. There exists a pair $\left(\bm{z}_0,\,\bm{h}_0,\,t_0\right)\in\mathcal{N}_\epsilon$ such that $\left\|\bm{z}-\bm{z}_0
\right\|$, $\|\bm{h}-\bm{h}_0\|$, and $|t-t_0|$ are each at most $\epsilon$. Taking the union bound yields
\begin{align}
	 \frac{1}{m}\sum_{i=1}^m\left(\bm{a}_i^\ccalT\bm{h}_0\right)^2\chi_E\left(\left|1+\frac{\bm{a}_i^\ccalT\bm{h}_0}{\bm{a}_i^\ccalT\bm{x}}\right|\right)
	 &\ge  \frac{1}{m}\sum_{i=1}^m\left(\bm{a}_i^\ccalT\bm{h}_0\right)^2\chi_E\left(\left|1-t_0+\frac{{a}_{i,2}}{{a}_{i,1}}t_0\right|\right)\nonumber\\
& \ge   \left(1-\zeta_1-\epsilon \right)\|\bm{h}_0\|^2,\quad \forall \left(\bm{z}_0,\,\bm{h}_0,\,t_0\right)\in\mathcal{N}_\epsilon
\end{align} 
with probability at least $1-\left(1+2/\epsilon\right)^{2n+1}{\rm e}^{-c_5 \epsilon^2m}\ge 1-{\rm e}^{-c_0m}$, which follows by choosing $m$ such that $m\ge \left(c_6\cdot\epsilon^{-2}\log\epsilon^{-1}\right) n$ for some constant $c_6>0$.

Recall that $\chi_E\left(\tau\right)$ is Lipschitz continuous, thus 
\begin{align}
&\bigg|
\frac{1}{m}\sum_{i=1}^m
\left(\bm{a}_i^\ccalT\bm{h}\right)^2\chi_E\left(\left|1+\frac{\bm{a}_i^\ccalT\bm{h}}{\bm{a}_i^\ccalT\bm{x}}\right|\right)-
\left(\bm{a}_i^\ccalT\bm{h}_0\right)^2\chi_E\left(\left|1+\frac{\bm{a}^\ccalT\bm{h}_0}{\bm{a}_i^\ccalT\bm{x}}\right|\right)
\bigg|
\nonumber\\
&\lesssim \frac{1}{m}\sum_{i=1}^m\left|\left(\bm{a}_i^\ccalT\bm{h}\right)^2-\left(\bm{a}_i^\ccalT\bm{h}_0\right)^2
\right|
\nonumber\\
&= \frac{1}{m}\sum_{i=1}^m\left|\bm{a}_i^\ccalT\left(\bm{h}\bm{h}^\ccalT-\bm{h}_0\bm{h}_0^\ccalT\right)\bm{a}_i
\right|\nonumber\\
&\lesssim c_7\sum_{i=1}^m\left|\bm{h}\bm{h}^\ccalT-\bm{h}_0\bm{h}_0^\ccalT
\right|\nonumber\\
&\le 2.5 c_{7}\left\|\bm{h}-\bm{h}_0\right\|\left\|\bm{h}\right\|\nonumber\\
&\le 2.5c_7\rho \epsilon
\end{align}
for some numerical constant $c_7$ and provided that $\epsilon<1/2$ and  $m\ge \left(c_6\cdot\epsilon^{-2}\log\epsilon^{-1}\right) n$,  
where the first inequality arises from the Lipschitz property of $\chi_E(\tau)$, the second uses the results in Lemma 1 in \cite{twf}, and the third from Lemma 2 in \cite{twf}.

Putting all results together confirms that with probability exceeding $1-2{\rm e}^{-c_0m}$, we have
\begin{align}
	&\frac{1}{m}\sum_{i=1}^m\left(\bm{a}_i^\ccalT\bm{h}\right)^2\chi_E\left(\left|1+\frac{\bm{a}_i^\ccalT\bm{h}}{\bm{a}_i^\ccalT\bm{x}}\right|\right)
\ge   \left[1-\zeta_1-\left(1+2.5c_7\rho\right)\epsilon\right]\left\|\bm{h}\right\|^2
\end{align} 
for all vectors $\left\|\bm{h}\right\|/\left\|\bm{x}\right\|\le\rho $, concluding the proof.

\subsection{Proof of Lemma~\ref{le:2ndterm}}\label{sec:rare}
Similar to the proof in Section~\ref{proof:1stterm}, it is convenient to work with the following auxiliary function instead of the discontinuous indicator function
    \begin{equation}
    	\chi_D(\theta)\!:=\!\left\{\!\!\begin{array}
    		{ll}
    		1,&|\theta|\ge\frac{2+\gamma}{1+\gamma}\\
    		{-100\left(\frac{1+\gamma}{2+\gamma}\right)^2\theta^2+100},&\sqrt{0.99}\cdot\frac{2+\gamma}{1+\gamma}\le\! 
    		|\theta|<\!\frac{2+\gamma}{1+\gamma}\\
    		0,&|\theta|<\sqrt{0.99}\cdot\frac{2+\gamma}{1+\gamma}
    	\end{array}\right.
    \end{equation}
    which is Lipschitz continuous in $\theta$ with Lipschitz constant $\mathcal{O}(1)$.  
For $\mathcal{D}_i=\left\{\left|\frac{\bm{a}_i^\ccalT\bm{h}}{\bm{a}_i^\ccalT\bm{x}}\right|\ge \frac{2+\gamma}{1+\gamma}\right\}$, it holds that $0\le \mathbb{1}_{\mathcal{D}_i}\le \chi_D\left(\left|\frac{\bm{a}_i^\ccalT\bm{h}}{\bm{a}_i^\ccalT\bm{x}}\right|\right)$ for any $\bm{x}\in\mathbb{R}^n$ and $\bm{h}\in\mathbb{R}^n$.
Assume without loss of generality that $\bm{x}=\bm{e}_1$.  Then for $\gamma> 0$ and $\rho\le 1/10$, it holds that
\begin{align}
	\frac{1}{m}\sum_{i=1}^m\mathbb{1}_{\Big\{\frac{|\bm{a}_i^\ccalT\bm{h}|}{
	|\bm{a}_i^\ccalT\bm{x}|}\ge \frac{2+\gamma}{1+\gamma}\Big\}}
	\le
	 \frac{1}{m}\sum_{i=1}^m\chi_D\left(\left|\frac{\bm{a}_i^\ccalT\bm{h}}{\bm{a}_i^\ccalT\bm{x}}\right|\right)	
&	
	=\frac{1}{m}\sum_{i=1}^m\chi_D\left(\left|\frac{\bm{a}_i^\ccalT\bm{h}}{a_{i,1}}\right|\right)\nonumber\\ 
&=\frac{1}{m}\sum_{i=1}^m\chi_D\left(\left|h_1+\frac{\bm{a}_{i,\backslash 1}^\ccalT\bm{h}_{\backslash 1}}{a_{i,1}}\right|\right)\nonumber\\
	&=\frac{1}{m}\sum_{i=1}^m\chi_D\left(\left|h_1+\frac{a_{i,2}}{a_{i,1}}\left\|\bm{h}_{\backslash 1}\right\|\right|\right)\nonumber\\
	&\overset{\rm (i)}{\le}\frac{1}{m}\sum_{i=1}^m\mathbb{1}_{\left\{
	\left|h_1+\frac{a_{i,2}}{a_{i,1}}\|\bm{h}_{\backslash 1}\|\right|\ge \sqrt{0.99}\cdot\frac{2+\gamma}{1+\gamma}
	\right\}}
	\end{align}
	where the last inequality arises from the definition of $\chi_D$. 
Note that $a_{i,2}/a_{i,1}$ obeys the standard Cauchy distribution, i.e., $a_{i,2}/a_{i,1}\sim{\rm Cauchy(0,1)}$ \cite{1962cauchy}. Transformation properties of Cauchy distributions assert that $h_1+\frac{a_{i,2}}{a_{i,1}}\|\bm{h}_{\backslash 1}\|\sim{\rm Cauchy}(h_1,\|\bm{h}_{\backslash 1}\|)$~\cite{2014cauchy}. Recall that the cdf of a Cauchy distributed random variable $w\sim{\rm Cauchy}\left(\mu_0,\alpha\right)$ is given by~\cite{1962cauchy} 
	\begin{equation}
		F(w;\mu_0,\alpha)=\frac{1}{\pi}\arctan\left(\frac{w-\mu_0}{\alpha}\right)+\frac{1}{2}.
	\end{equation}
It is easy to check that when $\|\bm{h}_{\backslash 1}\|=0$, the indicator function $\mathbb{1}_{\mathcal{D}_i}=0$ due to $|h_1|\le \rho< \sqrt{0.99}(2+\gamma)/(1+\gamma)$. Consider only $\|\bm{h}_{\backslash 1}\|\ne 0$ next. 
	Define for notational brevity $w:=a_{i,2}/a_{i,1}$, $\alpha:=\|\bm{h}_{\backslash 1}\|$, as well as $\mu_0:=h_1/\alpha$ and $w_0:=\sqrt{0.99}\frac{2+\gamma}{\alpha(1+\gamma)}$. Then, 
	\begin{align}
		\mathbb{E}&[\mathbb{1}_{
	\{
	|\mu_0+w|\ge w_0
	\}}]
	=1-\big[F
(w_0;\mu_0,1)-F
(-w_0;\mu_0,1)\big]\nonumber\\
&=1-\frac{1}{\pi}\big[\arctan({w_0-\mu_0})-\arctan({-w_0-\mu_0})\big]
\nonumber\\
&\overset{\rm (i)}{=}\frac{1}{\pi}\arctan\!\left(\frac{2w_0}{w_0^2-\mu_0^2-1}\right)
\nonumber\\
&\overset{\rm (ii)}{\le}
\frac{1}{\pi}\cdot\frac{2w_0}{w_0^2-\mu_0^2-1}
\nonumber\\
&\overset{\rm (iii)}{\le}\frac{1}{\pi}\cdot\frac{2\sqrt{0.99}\rho(2+\gamma)/(1+\gamma)}{0.99(2+\gamma)^2/(1+\gamma)^2-\rho^2}\nonumber\\
&\le 0.0646
\label{eq:prob11}
	\end{align}
for all $\gamma> 0$ and $\rho\le 1/10$.
In deriving $\rm (i)$, the fact that $\arctan(u)+\arctan(v)=\arctan\!\left(\frac{u+v}{1-uv}\right) ~({\rm mod}~\pi)$ for any $uv\ne 1$ was used. Concerning $\rm (ii)$, the inequality $\arctan(x)\le x$ for $x\ge 0$ is employed. Plugging given parameter values and using $\|\bm{h}_{\backslash 1}\|\le \|\bm{h}\|\le \rho$  confirms ${\rm (iii)}$. 
Next, $\mathbb{1}_{
	\left\{
	\left|\mu_0+w\right|\ge w_0
	\right\}}$ is bounded; and it is known that all bounded random variables are subexponential. Thus, upon applying the Bernstein-type inequality~\cite[Corollary 5.17]{chap2010vershynin},   
the next holds with probability at least $1-{\rm e}^{-c_5m\epsilon^2}$ for some numerical constant $c_5>0$ and any sufficiently small $\epsilon> 0$:
\begin{align}
	\frac{1}{m}\sum_{i=1}^m\mathbb{1}_{\left\{\frac{\left|\bm{a}_i^\ccalT\bm{h}\right|}{
	\left|\bm{a}_i^\ccalT\bm{x}\right|}\ge \frac{2+\gamma}{1+\gamma}\right\}}
	&\le \frac{1}{m}\sum_{i=1}^m\mathbb{1}_{
	\left\{
	\left|h_1+\frac{a_{i,2}}{a_{i,1}}\|\bm{h}_{\backslash 1}\|\right|\ge \sqrt{0.99}\frac{2+\gamma}{1+\gamma}
	\right\}
}\nonumber\\
&
\le (1+\epsilon)\mathbb{E}\Big[\mathbb{1}_{
\left\{
	\left|h_1+\frac{a_{i,2}}{a_{i,1}}\|\bm{h}_{\backslash 1}\|\right|\ge \sqrt{0.99}\frac{2+\gamma}{1+\gamma}
	\right\}
}\Big]\nonumber\\
	&\le\frac{1+\epsilon}{\pi}\cdot\frac{2\sqrt{0.99}\rho(2+\gamma)/(1+\gamma)}{0.99(2+\gamma)^2/(1+\gamma)^2-\rho^2}.\label{eq:inter1}
\end{align}

On the other hand, it is easy to establish that the following holds true for any fixed 
$ \bm{h}\in\mathbb{R}^n$:
\begin{align}
	\mathbb{E}\left[(\bm{a}_i^\ccalT\bm{h})^4\right]=\mathbb{E}\left[a_{i,1}^4\right]\left\|\bm{h}\right\|^4=3\left\|\bm{h}\right\|^4
\end{align}
which has also been used in Lemma 1~\cite{twf} and Lemma 6.1~\cite{sun2016}. 
Furthermore, recalling our working assumption $\|\bm{a}_i\|\le \sqrt{2.3n}$ and $\|\bm{h}\|\le \rho \|\bm{x}\|$, the random variables $(\bm{a}_i^\ccalT\bm{h})^4$ are bounded, and thus they are subexponential~\cite{chap2010vershynin}.  Appealing again to the Bernstein-type inequality for subexponential random variables \cite[Proposition 5.16]{chap2010vershynin} and provided that $m/n>c_6\cdot \epsilon^{-2}\log\epsilon^{-1}$ for some numerical constant $c_6>0$, we have 
\begin{align}
		\frac{1}{m}\sum_{i=1}^m\left(\bm{a}_i^\ccalT\bm{h}\right)^4\le 3(1+\epsilon)\left\|\bm{h}\right\|^4\label{eq:inter2}
\end{align}
which holds with probability exceeding $1-{\rm e}^{-c_5m\epsilon^2}$ for some universal constant $c_5>0$ and  any sufficiently small $\epsilon> 0$. 

Combining results \eqref{eq:inter1}, \eqref{eq:inter2}, leveraging the Cauchy-Schwartz inequality, and considering $\mathcal{D}_i\cap\mathcal{K}_i$ only consisting of a spherical cap, the following holds for any $ \rho\le 1/10$ and $  \gamma> 0$:
\begin{align}
	\frac{1}{m}\sum_{i=1}^m
	\left(\bm{a}_i^\ccalT\bm{h}\right)^2\mathbb{1}_{\mathcal{D}_i\cap \mathcal{K}_i}
	&\le \sqrt{\frac{1}{m}\sum_{i=1}^m\left(\bm{a}_i^\ccalT\bm{h}\right)^4}
	\sqrt{\frac{1}{2}\cdot\frac{1}{m}\sum_{i=1}^m\mathbb{1}_{\left\{\frac{\left|\bm{a}_i^\ccalT\bm{h}\right|}{
	\left|\bm{a}_i^\ccalT\bm{x}\right|}\ge \frac{2+\gamma}{1+\gamma}\right\}}}\nonumber\\
	&\le \sqrt{3(1+\epsilon)\left\|\bm{h}\right\|^4}\sqrt{\frac{1+\epsilon}{\pi}\cdot\frac{\sqrt{0.99}\rho(2+\gamma)/(1+\gamma)}{0.99(2+\gamma)^2/(1+\gamma)^2-\rho^2}
}\nonumber\\
	& \overset{\Delta}{=} (\zeta_2'+\epsilon')\left\|\bm{h}\right\|^2
\end{align}
where $\zeta_2':=0.9748\sqrt{\rho\tau/(0.99\tau^2-\rho^2)}$ with $\tau:=(2+\gamma)/(1+\gamma)$, which
 holds with probability at least $1-2{\rm e}^{-c_0m}$. The latter arises upon choosing $c_0\le c_5\epsilon^2$ in 
$1-2{\rm e}^{-c_5m\epsilon^2}$, which can be accomplished by taking $m/n$ sufficiently large.

\subsection*{Acknowledgments}

The authors would like to thank Prof. John Duchi for pointing out an error in an initial draft of this paper. We also thank Mahdi Soltanolkotabi, Yuxin Chen, Kejun Huang, and Ju Sun for helpful discussions.

\bibliographystyle{IEEEtran}
\bibliography{apower}

\begin{thebibliography}{10}
\providecommand{\url}[1]{#1}
\csname url@samestyle\endcsname
\providecommand{\newblock}{\relax}
\providecommand{\bibinfo}[2]{#2}
\providecommand{\BIBentrySTDinterwordspacing}{\spaceskip=0pt\relax}
\providecommand{\BIBentryALTinterwordstretchfactor}{4}
\providecommand{\BIBentryALTinterwordspacing}{\spaceskip=\fontdimen2\font plus
\BIBentryALTinterwordstretchfactor\fontdimen3\font minus
  \fontdimen4\font\relax}
\providecommand{\BIBforeignlanguage}[2]{{%
\expandafter\ifx\csname l@#1\endcsname\relax
\typeout{** WARNING: IEEEtran.bst: No hyphenation pattern has been}%
\typeout{** loaded for the language `#1'. Using the pattern for}%
\typeout{** the default language instead.}%
\else
\language=\csname l@#1\endcsname
\fi
#2}}
\providecommand{\BIBdecl}{\relax}
\BIBdecl

\bibitem{2006balan}
R.~Balan, P.~Casazza, and D.~Edidin, ``On signal reconstruction without
  phase,'' \emph{Appl. Comput. Harmon. Anal.}, vol.~20, no.~3, pp. 345--356,
  May 2006.

\bibitem{4m-4}
A.~Conca, D.~Edidin, M.~Hering, and C.~Vinzant, ``An algebraic characterization
  of injectivity in phase retrieval,'' \emph{Appl. Comput. Harmon. Anal.},
  vol.~38, no.~2, pp. 346--356, Mar. 2015.

\bibitem{savephase}
A.~S. Bandeira, J.~Cahill, D.~G. Mixon, and A.~A. Nelson, ``Saving phase:
  {I}njectivity and stability for phase retrieval,'' \emph{Appl. Comput.
  Harmon. Anal.}, vol.~37, no.~1, pp. 106--125, 2014.

\bibitem{nphard}
P.~M. Pardalos and S.~A. Vavasis, ``Quadratic programming with one negative
  eigenvalue is {NP}-hard,'' \emph{J. Global Optim.}, vol.~1, no.~1, pp.
  15--22, 1991.

\bibitem{book2001nemirovski}
A.~Ben-Tal and A.~Nemirovski, \emph{Lectures on {M}odern {C}onvex
  {O}ptimization: {A}nalysis, {A}lgorithms, and {E}ngineering
  {A}pplications}.\hskip 1em plus 0.5em minus 0.4em\relax SIAM, 2001, vol.~2.

\bibitem{twf}
Y.~Chen and E.~J. Cand{\`e}s, ``Solving random quadratic systems of equations
  is nearly as easy as solving linear systems,'' \emph{Comm. Pure Appl. Math.},
  vol.~70, no.~5, pp. 822--883, Dec. 2017.

\bibitem{1978fienup}
J.~R. Fienup, ``Reconstruction of an object from the modulus of its {F}ourier
  transform,'' \emph{Opt. Letters}, vol.~3, no.~1, pp. 27--29, July 1978.

\bibitem{siam2015candes}
E.~J. Cand{\`e}s, Y.~C. Eldar, T.~Strohmer, and V.~Voroninski, ``Phase
  retrieval via matrix completion,'' \emph{SIAM Rev.}, vol.~57, no.~2, pp.
  225--251, May 2015.

\bibitem{spm2016eldar}
K.~Jaganathan, Y.~C. Eldar, and B.~Hassibi, ``Phase retrieval: {A}n overview of
  recent developments,'' \emph{Opt. Compressive Sens; also in
  arXiv:1510.07713}, 2015.

\bibitem{nature1999miao}
J.~Miao, P.~Charalambous, J.~Kirz, and D.~Sayre, ``Extending the methodology of
  {X}-ray crystallography to allow imaging of micrometre-sized non-crystalline
  specimens,'' \emph{Nature}, vol. 400, no. 6742, pp. 342--344, July 1999.

\bibitem{optics}
R.~P. Millane, ``Phase retrieval in crystallography and optics,'' \emph{J. Opt.
  Soc. Am. A}, vol.~7, no.~3, pp. 394--411, 1990.

\bibitem{oe2015bian}
L.~Bian, J.~Suo, G.~Zheng, K.~Guo, F.~Chen, and Q.~Dai, ``Fourier ptychographic
  reconstruction using {W}irtinger flow optimization,'' \emph{Opt. Express},
  vol.~23, no.~4, pp. 4856--4866, 2015.

\bibitem{array}
A.~Chai, M.~Moscoso, and G.~Papanicolaou, ``Array imaging using intensity-only
  measurements,'' \emph{Inverse Probl.}, vol.~27, no.~1, p. 015005, Dec. 2011.

\bibitem{altmin2015}
S.~Marchesini, Y.-C. Tu, and H.-T. Wu, ``Alternating projection, ptychographic
  imaging and phase synchronization,'' \emph{Appl. Comput. Harmon. Anal.}, June
  2015, to appear.

\bibitem{astronomy}
C.~Fienup and J.~Dainty, ``Phase retrieval and image reconstruction for
  astronomy,'' \emph{Image Recovery: {T}heory and Application}, pp. 231--275,
  1987.

\bibitem{micro}
J.~Miao, I.~Ishikawa, Q.~Shen, and T.~Earnest, ``Extending {X}-ray
  crystallography to allow the imaging of noncrystalline materials, cells, and
  single protein complexes,'' \emph{Annu. Rev. Phys. Chem.}, vol.~59, pp.
  387--410, May 2008.

\bibitem{nphard1}
H.~Sahinoglou and S.~D. Cabrera, ``On phase retrieval of finite-length
  sequences using the initial time sample,'' \emph{IEEE Trans. Circuits and
  Syst.}, vol.~38, no.~8, pp. 954--958, Aug. 1991.

\bibitem{2013nphard}
C.~J. Hillar and L.-H. Lim, ``Most tensor problems are {NP}-hard,'' \emph{J.
  ACM}, vol.~60, no.~6, p.~45, 2013.

\bibitem{wf}
E.~J. Cand{\`e}s, X.~Li, and M.~Soltanolkotabi, ``Phase retrieval via
  {W}irtinger flow: {T}heory and algorithms,'' \emph{IEEE Trans. Inf. Theory},
  vol.~61, no.~4, pp. 1985--2007, Apr. 2015.

\bibitem{nips2016wg}
G.~Wang and G.~B. Giannakis, ``Solving random systems of quadratic equations
  via truncated generalized gradient flow,'' in \emph{Adv. Neural Inf. Process.
  Syst.}, Barcelona, Spain, 2016, pp. 568--576.

\bibitem{hardproblems}
K.~G. Murty and S.~N. Kabadi, ``Some {NP}-complete problems in quadratic and
  nonlinear programming,'' \emph{Math. Program.}, vol.~39, no.~2, pp. 117--129,
  1987.

\bibitem{spm2015eldar}
Y.~Shechtman, Y.~C. Eldar, O.~Cohen, H.~N. Chapman, J.~Miao, and M.~Segev,
  ``Phase retrieval with application to optical imaging: {A} contemporary
  overview,'' \emph{IEEE Signal Proc. Mag.}, vol.~32, no.~3, pp. 87--109, May
  2015.

\bibitem{1duniqueness}
E.~Hofstetter, ``Construction of time-limited functions with specified
  autocorrelation functions,'' \emph{IEEE Trans. Inf. Theory}, vol.~10, no.~2,
  pp. 119--126, Apr. 1964.

\bibitem{gespar}
Y.~Shechtman, A.~Beck, and Y.~C. Eldar, ``{GESPAR}: {E}fficient phase retrieval
  of sparse signals,'' \emph{IEEE Trans. Signal Process.}, vol.~62, no.~4, pp.
  928--938, Feb. 2014.

\bibitem{2013spr}
K.~Jaganathan, S.~Oymak, and B.~Hassibi, ``Sparse phase retrieval: {U}niqueness
  guarantees and recovery algorithms,'' \emph{arXiv:1311.2745}, 2013.

\bibitem{acha2014yesm}
Y.~C. Eldar and S.~Mendelson, ``Phase retrieval: Stability and recovery
  guarantees,'' \emph{Appl. Comput. Harmon. Anal.}, vol.~36, no.~3, pp.
  473--494, May 2014.

\bibitem{2015stft}
Y.~C. Eldar, P.~Sidorenko, D.~G. Mixon, S.~Barel, and O.~Cohen, ``Sparse phase
  retrieval from short-time {F}ourier measurements,'' \emph{IEEE Signal
  Process. Lett.}, vol.~22, no.~5, pp. 638--642, May 2015.

\bibitem{coded}
E.~J. Cand{\`e}s, X.~Li, and M.~Soltanolkotabi, ``Phase retrieval from coded
  diffraction patterns,'' \emph{Appl. Comput. Harmon. Anal.}, vol.~39, no.~2,
  pp. 277--299, Sept. 2015.

\bibitem{altmin}
P.~Netrapalli, P.~Jain, and S.~Sanghavi, ``Phase retrieval using alternating
  minimization,'' \emph{IEEE Trans. Signal Process.}, vol.~63, no.~18, pp.
  4814--4826, Sept. 2015.

\bibitem{phaselift}
E.~J. Cand{\`e}s, T.~Strohmer, and V.~Voroninski, ``{P}hase{L}ift: {E}xact and
  stable signal recovery from magnitude measurements via convex programming,''
  \emph{Appl. Comput. Harmon. Anal.}, vol.~66, no.~8, pp. 1241--1274, Nov.
  2013.

\bibitem{mtwf}
H.~Zhang, Y.~Chi, and Y.~Liang, ``Provable non-convex phase retrieval with
  outliers: {M}edian truncated {W}irtinger flow,'' \emph{arXiv:1603.03805},
  2016.

\bibitem{gerchberg}
R.~W. Gerchberg and W.~O. Saxton, ``A practical algorithm for the determination
  of phase from image and diffraction,'' \emph{Optik}, vol.~35, pp. 237--246,
  Nov. 1972.

\bibitem{thesis2014}
M.~Soltanolkotabi, ``Algorithms and theory for clustering and nonconvex
  quadratic programming,'' Ph.D. dissertation, Stanford University, 2014.

\bibitem{ip2015wei}
K.~Wei, ``Solving systems of phaseless equations via {K}aczmarz methods: {A}
  proof of concept study,'' \emph{Inverse Probl.}, vol.~31, no.~12, p. 125008,
  2015.

\bibitem{sun2016}
J.~Sun, Q.~Qu, and J.~Wright, ``A geometric analysis of phase retrieval,''
  \emph{arXiv:1602.06664}, 2016.

\bibitem{procrustes}
S.~Tu, R.~Boczar, M.~Soltanolkotabi, and B.~Recht, ``Low-rank solutions of
  linear matrix equations via {P}rocrustes flow,'' \emph{arXiv:1507.03566},
  2015.

\bibitem{localcvx}
S.~Sanghavi, R.~Ward, and C.~D. White, ``The local convexity of solving systems
  of quadratic equations,'' \emph{Results Math.}, pp. 1--40, June 2016.

\bibitem{2016li}
X.~Li, S.~Ling, T.~Strohmer, and K.~Wei, ``Rapid, robust, and reliable blind
  deconvolution via nonconvex optimization,'' \emph{arXiv:1606.04933}, 2016.

\bibitem{focs2015sun}
R.~Sun and Z.-Q. Luo, ``Guaranteed matrix completion via nonconvex
  factorization,'' in \emph{IEEE 56th Annual Symposium on Foundations of
  Computer Science}, 2015, pp. 270--289.

\bibitem{phasecut}
I.~Waldspurger, A.~d'Aspremont, and S.~Mallat, ``Phase recovery, maxcut and
  complex semidefinite programming,'' \emph{Math. Program.}, vol. 149, no.~1,
  pp. 47--81, 2015.

\bibitem{2016huang}
K.~Huang, Y.~C. Eldar, and N.~D. Sidiropoulos, ``Phase retrieval from {1D}
  {F}ourier measurements: {C}onvexity, uniqueness, and algorithms,'' \emph{IEEE
  Trans. Signal Process.}, vol.~64, no.~23, pp. 6105--6117, Dec. 2016.

\bibitem{tsp2016qian}
C.~Qian, N.~D. Sidiropoulos, K.~Huang, L.~Huang, and H.~C. So, ``Phase
  retrieval using feasible point pursuit: {A}lgorithms and {C}ramer-{R}ao
  bound,'' \emph{IEEE Trans. Signal Process.}, vol.~64, no.~20, pp. 5282--5296,
  Oct. 2016.

\bibitem{qianfu}
C.~Qian, X.~Fu, N.~D. Sidiropoulos, L.~Huang, and J.~Xie, ``Inexact alternating
  optimization for phase retrieval in the presence of outliers,'' \emph{IEEE
  Trans. Signal Processing}, 2016 (to appear).

\bibitem{staf}
G.~Wang, G.~B. Giannakis, and J.~Chen, ``Scalable solvers of random quadratic
  equations via stochastic truncated amplitude flow,'' \emph{IEEE Trans. Signal
  Process.}, vol.~65, no.~8, pp. 1961--1974, Apr. 2017.

\bibitem{sparta}
G.~Wang, L.~Zhang, G.~B. Giannakis, J.~Chen, and M.~Ak{\c{c}}akaya, ``{S}parse
  phase retrieval via truncated amplitude flow,'' \emph{arXiv:1611.07641},
  2016.

\bibitem{raf}
G.~Wang, G.~B. Giannakis, Y.~Saad, and J.~Chen, ``Solving almost all systems of
  random quadratic equations,'' \emph{arXiv:1705.10407}, 2017.

\bibitem{fcm2014candes}
E.~J. Cand{\`e}s and X.~Li, ``Solving quadratic equations via {P}hase{L}ift
  when there are about as many equations as unknowns,'' \emph{Found. Comput.
  Math.}, vol.~14, no.~5, pp. 1017--1026, 2014.

\bibitem{tit2010spectral}
R.~H. Keshavan, A.~Montanari, and S.~Oh, ``Matrix completion from a few
  entries,'' \emph{IEEE Trans. Inf. Theory}, vol.~56, no.~6, pp. 2980--2998,
  Jun. 2010.

\bibitem{experimental2015}
L.-H. Yeh, J.~Dong, J.~Zhong, L.~Tian, M.~Chen, G.~Tang, M.~Soltanolkotabi, and
  L.~Waller, ``Experimental robustness of {F}ourier ptychography phase
  retrieval algorithms,'' \emph{Opt. Express}, vol.~23, no.~26, pp.
  33\,214--33\,240, Dec. 2015.

\bibitem{jmlr2013cai}
T.~Cai, J.~Fan, and T.~Jiang, ``Distributions of angles in random packing on
  spheres,'' \emph{J. Mach. Learn. Res.}, vol.~14, no.~1, pp. 1837--1864, Jan.
  2013.

\bibitem{reshaped}
H.~Zhang, Y.~Zhou, Y.~Liang, and Y.~Chi, ``Reshaped {W}irtinger flow and
  incremental algorithm for solving quadratic system of equations,''
  \emph{arXiv:1605.07719}, 2016.

\bibitem{shor1972class}
N.~Z. Shor, ``A class of almost-differentiable functions and a minimization
  method for functions of this class,'' \emph{Cybern. Syst. Anal.}, vol.~8,
  no.~4, pp. 599--606, July 1972.

\bibitem{book1998rockafellar}
R.~Rockafellar and R.~J.-B. Wets, \emph{Variational {A}nalysis}.\hskip 1em plus
  0.5em minus 0.4em\relax Berlin-Heidelberg: Springer Verlag, 1998.

\bibitem{book1985shor}
N.~Z. Shor, K.~C. Kiwiel, and A.~Ruszcay{\`n}ski, \emph{Minimization {M}ethods
  for {N}on-differentiable {F}unctions}.\hskip 1em plus 0.5em minus 0.4em\relax
  Springer-Verlag New York, Inc., 1985.

\bibitem{book1990clarke}
F.~H. Clarke, \emph{Optimization and {N}onsmooth {A}nalysis}.\hskip 1em plus
  0.5em minus 0.4em\relax SIAM, 1990, vol.~5.

\bibitem{clarke1975gg}
------, ``Generalized gradients and applications,'' \emph{T. Am. Math. Soc.},
  vol. 205, pp. 247--262, 1975.

\bibitem{2015chen1}
P.~Chen, A.~Fannjiang, and G.-R. Liu, ``Phase retrieval with one or two
  diffraction patterns by alternating projection with null initialization,''
  \emph{arXiv:1510.07379}, 2015.

\bibitem{2015chen2}
P.~Chen and F.~A., ``Fourier phase retrieval with a single mask by
  {D}ouglas-{R}achford algorithm,'' \emph{arXiv:1509.00888}, 2015.

\bibitem{chap2010vershynin}
R.~Vershynin, ``Introduction to the non-asymptotic analysis of random
  matrices,'' \emph{arXiv:1011.3027}, 2010.

\bibitem{saad1}
Y.~Saad, \emph{Iterative {M}ethods for {S}parse {L}inear {S}ystems}.\hskip 1em
  plus 0.5em minus 0.4em\relax SIAM, 2003.

\bibitem{duchi2017}
J.~C. Duchi and F.~Ruan, ``Solving (most) of a set of quadratic equalities:
  {C}omposite optimization for robust phase retrieval,''
  \emph{arXiv:1705.02356}, 2017.

\bibitem{1981ecd}
S.~Cambanis, S.~Huang, and G.~Simons, ``On the theory of elliptically contoured
  distributions,'' \emph{J. Multivar. Anal.}, vol.~11, no.~3, pp. 368--385,
  Sep. 1981.

\bibitem{chisquaretail}
B.~Laurent and P.~Massart, ``Adaptive estimation of a quadratic functional by
  model selection,'' \emph{Ann. Stat.}, vol.~28, no.~5, pp. 1302--1338, 2000.

\bibitem{erf2011}
S.-H. Chang, P.~C. Cosman, and L.~B. Milstein, ``Chernoff-type bounds for the
  {G}aussian error function,'' \emph{IEEE Trans. Commun.}, vol.~59, no.~11, pp.
  2939--2944, July 2011.

\bibitem{1962cauchy}
T.~S. Ferguson, ``A representation of the symmetric bivariate {C}auchy
  distribution,'' \emph{Ann. Math. Stat.}, vol.~33, no.~4, pp. 1256--1266,
  1962.

\bibitem{2014cauchy}
H.~Y. Lee, G.~J. Parka, and H.~M. Kim, ``A clarification of the {C}auchy
  distribution,'' \emph{Commun. Stat. Appl. Methods}, vol.~21, no.~2, pp.
  183--191, Mar. 2014.

\end{thebibliography}

\end{document}